\documentclass[a4paper,10pt]{article}
\usepackage[margin=1in]{geometry}
\usepackage[utf8]{inputenc} 
\usepackage[T1]{fontenc}    
\usepackage{comment}
\usepackage{hyperref}       
\usepackage{url}            
\usepackage{nicefrac}       
\usepackage{microtype}     
\usepackage{fullpage}

\usepackage{xcolor}
\usepackage{framed}
\colorlet{shadecolor}{pink} 
\usepackage{authblk}
\usepackage{adjustbox}
\usepackage{comment}
\usepackage{multirow}
\usepackage{natbib}
	
\usepackage{color, colortbl}
\usepackage{graphicx}
\usepackage{soul}
\usepackage{subfigure}
\usepackage{booktabs}
\usepackage{tablefootnote}

\usepackage{amsmath,amsthm,amssymb,amsfonts}
\usepackage{algorithm}
\usepackage{algorithmic}
\usepackage{enumerate}
\usepackage{cleveref}
\usepackage{bm}
\usepackage{pifont}

\theoremstyle{plain} 
\newcommand{\vertiii}[1]{{\left\vert\kern-0.25ex\left\vert\kern-0.25ex\left\vert #1 
\right\vert\kern-0.25ex\right\vert\kern-0.25ex\right\vert}}

\usepackage{algorithm}
\usepackage[ruled,algo2e]{algorithm2e}
\usepackage{hyperref}

\hypersetup{colorlinks=true,linkcolor=blue,filecolor=magenta,urlcolor=cyan,}

\definecolor{antiquewhite}{rgb}{0.98, 0.92, 0.84} 
\definecolor{blizzardblue}{rgb}{0.67, 0.9, 0.93}

\usepackage{epsf}
\usepackage{graphics}
\usepackage{booktabs}
\usepackage{subfigure}
\usepackage{wrapfig}
\usepackage{psfrag}

\usepackage{color}

\usepackage{mathtools}
\usepackage{amsfonts}
\usepackage{amsthm}
\usepackage{amsmath}
\usepackage{amssymb}

\newcommand{\inprod}[2]{\ensuremath{\left\langle #1 , \, #2 \right\rangle}}
\newcommand{\pare}[1]{\left( #1 \right)}
\newcommand{\cbr}[1]{\left\{ #1 \right\}}
\newcommand{\sbr}[1]{\left[ #1 \right]}

\newtheorem{theorem}{Theorem}
\newtheorem{proposition}{Proposition}
\newtheorem{lemma}{Lemma}
\newtheorem{corollary}{Corollary}
\newtheorem{definition}{Definition}
\newtheorem{remark}{Remark}

\newtheorem{assumption}{Assumption}

\long\def\comment#1{}

\newcommand{\norm}[1]{\left\| #1 \right\|}

\newcommand{\E}{\ensuremath{{\mathbb{E}}}}

\newcommand{\R}{\mathbb{R}}

\newcommand{\cA}{\mathcal{A}}
\newcommand{\cP}{\mathcal{P}}

\newcommand{\cD}{\mathcal{D}}  

\newcommand{\cS}{\mathcal{S}}
   
\newcommand{\cW}{\mathcal{W}}   
  
\newcommand{\cL}{\mathcal{L}} 
   
\newcommand{\balpha}{\boldsymbol{\alpha}}

\newcommand{\ba}{\mathbf{a}}
\newcommand{\bb}{\mathbf{b}}

\newcommand{\bx}{\mathbf{x}}

\newcommand{\bv}{\mathbf{v}}

\newcommand{\bw}{\mathbf{w}}
\newcommand{\mask}{\mathbf{m}}

\newcommand{\br}{\mathbf{r}}
\newcommand{\by}{\mathbf{y}}
\newcommand{\bg}{\mathbf{g}}
\newcommand{\bp}{\mathbf{p}}
\newcommand{\bh}{\mathbf{h}}

\newcommand{\mA}{\mathbf{A}}

\newcommand{\bxi}{\boldsymbol{\xi}}

\newcommand{\mH}{{\bf H}}
\newcommand{\mI}{{\bf I}}

\newcommand{\mM}{{\bf M}}

\usepackage[colorinlistoftodos,bordercolor=orange,backgroundcolor=orange!20,linecolor=orange,textsize=scriptsize]{todonotes}

\title{\bf On the Convergence and Stability of Distributed Sub-model Training}
\date{}

\author{Yuyang Deng$^\dagger$ \qquad Fuli Qiao$^*$ \qquad Mehrdad Mahdavi$^*$ \vspace*{.2em} \\ 
 \quad $^\dagger$Columbia University \quad  $^*$The Pennsylvania State University \vspace*{.2em} \\ \texttt{yd282@columbia.edu \ \qquad  \{fvq5015,mzm616\}@psu.edu}
}
\begin{document}

\maketitle

\begin{abstract}
As learning models continue to grow in size, enabling on-device local training of these models has emerged as a critical challenge in federated learning. A popular solution is sub-model training, where the server only distributes randomly sampled sub-models to the edge clients, and clients only update these small models. However, those random sampling of sub-models may not give satisfying convergence performance. In this paper, observing the success of SGD with shuffling,
we propose a distributed shuffled sub-model training, where the full model is partitioned into several sub-models in advance, and the server shuffles those sub-models, sends each of them to clients at each round, and by the end of local updating period, clients send back the updated sub-models, and server averages them. We establish the convergence rate of this algorithm. We also study the generalization of distributed sub-model training via stability analysis, and find that the sub-model training can improve the generalization via amplifying the stability of training process. The extensive experiments also validate our theoretical findings.
\end{abstract}

\section{Introduction}
\label{intro}

We consider optimizing the following   objective
\begin{align}\label{eqn:main:obj}
    \min_{\bw\in\cW\subset \R^d} F(\bw):=\frac{1}{N} \sum\nolimits_{i=1}^N \cbr{f_i(\bw) :=  \E_{\xi\sim\cD_i} [\ell(\bw;\xi)] },
\end{align}
collaboratively in a distributed setting with $N$ clients, where $f_i(\bw)$ is the local loss function realized by $i$th  client's data, and $\cW \subset \R^d$ is some bounded convex set. To solve the problem with communication efficiency and data privacy, a widely employed method is FedAvg~\cite{mcmahan2017communication} and its variants, where multiple devices (clients) collaborate to train a shared machine learning model without exchanging their data. Specifically, each client trains a local version of the model using its own data with stochastic gradient descent (SGD) for $K$ steps, then sends its updated model parameters to a central server. The server aggregates these updates by averaging the parameters from all clients, and the resulting global model is sent back to the clients for further local training. This process is repeated until the model converges.

A key limitation of full-model training approaches is that each client must maintain a local model with the same complexity  as the global model. In modern machine learning (ML), where model complexity can reach millions or even billions of parameters, many clients may lack the memory and computational resources needed to store and optimize the full model. Moreover, devices  participating in collaborative learning greatly vary in  computational and storage capacity and  can only confine to  ML models that meet their resources for training. A common approach to address this issue is \textit{partial training} where the server selects sub-models proportional to the computational resources available on each device, either \textit{randomly} or based on predefined rules (e.g., \textit{rolling} or \textit{static partitioning}), and then distribute them to the clients. The clients only update these sub-models and, after a few rounds of local updates, send the updated sub-models back to the server for aggregation.  The representative works include PruneFL~\citep{jiang2022model}, IST~\citep{yuan2019distributed}, HeteroFL~\citep{diao2020heterofl} and  FedRolex~\citep{alam2022fedrolex}.

Despite the empirical success of this training paradigm, the convergence of distributed sub-model training has not been well understood. A key research question we seek to rigorously address is: \textit{how does sub-model training, in comparison to full-model training, affect both the convergence and generalization of the learned model.}

Investigating this could provide insights into the trade-offs between computational efficiency and model performance in resource-constrained federated settings. A few recent studies have begun to explore this question, primarily from an optimization perspective, highlighting both the benefits and limitations of sub-model training in federated learning.~\cite{shulgin2023towards} studied distributed fully synchronized sub-model training algorithm on quadratic objective, and showed that the convergence result will suffer from a residual error, unless the objective and masking scheme admit some benign properties. \cite{demidovich2023mast} studied similar algorithm, and proved convergence to the optimal point of the masked objective for general strongly convex losses. For nonconvex loss, \cite{mohtashami2022masked} studied single machine setting, where $N=1$, and proved that sub-model training converges to first order stationary point of the masked objective. Their convergence bound include factors that quantifies the alignment between gradient on the masked and non-masked models; however, it remains unclear whether these quantities can be effectively controlled.~\cite{zhou2024every} and \cite{wu2024fiarse} investigate the convergence of sub-model training with local updates on general nonconvex loss functions, making it the most relevant prior work to this study. However, their bound depends on the sum of the norms of the intermediate solutions; if these norms are too large, the convergence bound becomes vacuous.

A desired bound should {\em explicitly} show how sub-model selection strategy affects the convergence and generalization due to \textit{model drift} caused by partial updating. In this paper, we present the first concrete convergence analysis of distributed sub-model training with different sampling schemes. We first consider FedAvg with Bernoulli random sub-model sampling. Then we consider rolling sub-model training method, where the server partitions the full model into $R$  pieces and at the beginning of each epoch, it shuffles these sub-models and sequentially assigns them to the clients to be updated locally.   We establish the convergence rate of both sampling schemes in convex and non-convex settings, highlighting the impact of partial training as captured through sampling. In nonconvex setting, we show that sub-model training will converge to the stationary point of an alternative objective induced by masking. To study the generalization of sub-model training, we further provide the generalization analysis of distributed sub-model training with random and rolling masking, and find that masking can enhance generalization by stabilizing the training process, as long as the residual optimization error from partial training remains controlled.

\noindent\textbf{Contribution.}~The main results of this work include:
\begin{itemize}
    \item We establish the rigorous convergence rate of distributed sub-model training with random sub-model selection (Section~\ref{sec:conv:random}) and shuffled rolling sub-model selection (Section~\ref{sec:conv:shuffling}). We show that under strongly convex and smooth setting, they both enjoy an $O(\frac{1}{T})$ rate plus a residual error due to model masking. To better illustrate the relationship between this alternative objective and original objective, we show that the stationary point of alternative objective can be translated to that of original objective, with some residual error depending on masking ratio. 
   \item We further establish the convergence results of the two algorithms under nonconvex setting. We show that, under the algorithm dynamic, the model will finally converge to the stationary point of an alternative objective function induced by model masking. To the best of our knowledge, this is the first rigorous analysis of the convergence of permutation-based methods in sub-model training, thoroughly examining the interaction between model drift from partial training on sub-models and the impact of permutation-based sub-model assignments on convergence.
    \item  We analyze the generalization ability of distributed sub-model training with random and rolling masking (Section~\ref{sec:stability}), and find that masking can improve the generalization via stabilizing the training process, as long as the residual optimization error from partial training is controlled.
    \item We conduct thorough experiments that corroborate our theoretical results (Section~\ref{sec:exp}).
\end{itemize}
Additional related works, empirical results, and omitted proofs can be found in the appendix.
\section{Distributed Sub-model Training via Random  Sampling}\label{sec:conv:random}
We consider the FedAvg (a.k.a. Local SGD) to optimize the objective in Eq.~\ref{eqn:main:obj}. The algorithm proceeds for $R$ communication rounds, where at round $r$, each client, upon receiving the global model from the server, independently performs $K$   local updates using its local data to compute gradients to update its local model parameters accordingly. After completing the $K$ local steps, the clients send their updated parameters to the central server, which averages the models across all clients to produce a global model for the next round. This periodic averaging allows for reduced communication frequency and efficient use of bandwidth while enhancing convergence stability.

Adapting FedAvg for system heterogeneity is key to enabling collaboration among clients with varying computational power. A simple solution is distributed sub-model training using random sub-model sampling, which is first proposed in~\cite{alam2022fedrolex}. At the beginning of $r$th communication round, the server  generates a $d$-dimensional Bernoulli random masking $\mask_i^r\sim Ber(p_i)$ (each coordinate is $1$ with probability $p_i$, otherwise $0$) for $i$th client and distributes the masked global model $\bw_r^i = \mask_i^r \odot \bw_r$ to the client to perform the following  update locally for $K$ steps : 
\begin{equation*}
    \bw_{r,k+1}^i = \bw_{r,k}^i - \eta \mask^r_i \odot \nabla f_i(\mask^r_i\odot\bw_{r,k}^i;\xi_{r,k}^i), 
\end{equation*}
where $\xi_{r,k}^i$ is the data point uniformly randomly sampled from $i$th client's dataset. Each client updates only a subset of the model parameters, with the number of updated parameters determined by the masking probability $p_i$. For clients with limited capacity, the server can assign a smaller 
$p_i$  to reduce their computational burden. After $K$ local steps, $i$th client sends model $\mask_i^r \odot \bw_{r,K}$ to server, and server averages the received models:
\begin{align*}
  \bw_{r+1} = \cP_{\cW}\pare{  \frac{1}{N}\sum_{i=1}^N (\bw_{r,K}^i + (\mathbf{1}- \mask_i^r) \odot \bw_r)}.
\end{align*}
where $\cP_{\cW}(\cdot)$ is the projection operator onto convex set $\cW$.  In words, server will fill the parameters not selected by $\mask_i^r$ with old parameters of the last round model, i.e., $\bw_{r}$, and then average all clients models. After that, server distributes $\mask_i^{r+1} \odot \bw_{r+1}$ to $i$th client. The pseudo-code of the algorithm is depicted in Algorithm~\ref{algorithm: Masked FedAvg}.

\begin{algorithm2e}[t] 
	\DontPrintSemicolon
    \caption{\sffamily{Randomly Masked FedAvg }}
    	\label{algorithm: Masked FedAvg}
    	\textbf{Input:} Initial model $ {\bw}_{0} = \mathbf{0} $, masking probabilities $p_1,\ldots,p_N$ and  stepsizes $\eta$.    
	\\ 
 \For{$i = 1,\ldots,N$}{
 \For{$r = 0,\ldots, R-1$}{
 
  Server generates Bernoulli random masks $\mask_1^r,\ldots,\mask_N^r$ for each client, according to probability $p_1,\ldots,p_N$.\\
  Server distributes $ \bw_r^i = \mask_i^r \odot \bw_r$ to $i$th user.
	
 \For{$k = 0,\ldots,K-1$}
	 {   
            $ \bw_{r,k+1}^i = \bw_{r,k }^i - \eta \mask_i^r \odot \nabla f_i(\mask_i^r\odot\bw_{r,k }^i;\xi_{r,k }^i) $  
  	 
   } 
   The $i$th client sends $ \bw_{r,K}^i$ back to server.\\ 
  Server averages models $\bw_{r+1} = \cP_{\cW}\pare{  \frac{1}{N}\sum_{i=1}^N (\bw_{r,K}^i + (\mathbf{1}- \mask_i^r) \odot \bw_r)}$\\
   
}
}
        \textbf{Output:}   $\hat{\bw} = \cP_{\cW}\pare{ \bw_R - \frac{1}{L} \frac{1}{N}\sum_{i=1}^N \mask_i \odot \nabla f_i(\mask_i\odot\bw_R ) }$ where $\mask_i \sim Ber(p_i)$. 
\end{algorithm2e}
\subsection{Convergence in Convex Setting}
To study the convergence of this simple algorithm, we make the following standard assumptions.
 \begin{assumption}[{\sffamily{Smoothness}}]\label{assump: smooth} We assume $\forall i \in [N]$, $f_i(\bx)$ is L-smooth, i.e.,
$$
\forall \bx, \by:\left\|\nabla f_i(\bx)-\nabla f_i(\by)\right\| \leq L\|\bx-\by\| .
$$ 
\end{assumption}
\begin{assumption}[{\sffamily{Strong convexity}}]\label{assump: sc} We assume $\forall i \in [N]$, $f_i(\bx)$ is L-smooth and $\mu$-strongly convex
$$
\forall \bx, \by: f_i(\by) \geq f_i(\bx)+\langle\nabla f_i(\bx), \by-\bx\rangle+\frac{\mu}{2} \|\by-\bx\|^2 .
$$
\end{assumption}
We denote the condition number by $\kappa={L}/{\mu}$. 
\begin{assumption}[{\sffamily{Bounded variance}}]\label{assumption: bounded var}
The variance of stochastic gradients computed at each local function is bounded, i.e.,  $\forall i \in [N], \forall \bw \in \cW, \mathbb{E}_{\mask_i \sim Ber(p_i),\xi}[\|\mask_i\odot\nabla f_i(\mask_i\odot\bw;\xi) - \mask_i\odot\nabla f_i(\mask_i\odot\bw)\|^2]  \leq \delta^2$.
\end{assumption}

\begin{assumption}[{\sffamily{Bounded domain}}]\label{assumption: bounded domain}
The domain $\cW \subset \R^d$ is a bounded convex set, with diameter $W$ under $\ell_2$ metric, i.e., $\forall \bw \in \cW, \norm{\bw } \leq W$.

\end{assumption}

\begin{assumption}[{\sffamily{Bounded gradient}}]\label{assumption: bounded grad}
The gradients computed at each local function are bounded, i.e.,  $\forall i \in [N], \sup_{\bw\in\cW}\|\nabla f_i(\bw)\|\leq G$.
\end{assumption}
\begin{definition}\label{def:drift at optimum}
We define the masked heterogeneity at optimum as follows:
    \begin{align*}
   \frac{1}{N}\sum\nolimits_{i=1}^N\E_{\mask_i \sim Ber(p_i)}\norm{\mask_i\odot\nabla  f_i(\mask_i\odot \bw^*) }^2 \leq \sigma^2_{*}.
    \end{align*}
where $\bw^* = \arg\min_{\bw \in \cW}   \frac{1}{N}\sum_{i=1}^N \E_{\mask_i \sim Ber(p_i)} [f_i(\mask_i\odot\bw)]$.
\end{definition}

Assumptions~\ref{assump: smooth},~\ref{assump: sc},~\ref{assumption: bounded domain} and~\ref{assumption: bounded grad} are standard in convex optimization. Assumption~\ref{assumption: bounded var} is the bounded variance of masked gradients which becomes small if the masking probability $p_i$ is high. Definition~\ref{def:drift at optimum} is also standard in Local SGD analysis~\citep{khaled2019tighter}, but here it is adapted to the masked gradients.

 \begin{theorem}\label{thm:randomly masked fedavg}
Let Assumptions~\ref{assump: smooth}-~\ref{assumption: bounded grad} hold. Then Algorithm~\ref{algorithm: Masked FedAvg} with $\eta = \frac{\log (KR)^2}{\tilde \mu KR} $  and $R \geq \frac{L}{\tilde \mu} \log(K^2R^2)$ will output the solution $\hat{\bw} $, such that  the following statement holds:
    \begin{align*}
         F(\hat\bw) - F(\bw^*) &\leq   L    \tilde O\pare{    \frac{\E \norm{\bw_0  - \bw^*}^2 }{K^2 R^2}   +     \frac{ \tilde \kappa \sigma^2_{*} +    \tilde \kappa  \delta^2 }{\tilde \mu^2   R^2}    +  \frac{  \delta^2 }{\tilde \mu^2 NKR}  
   } \\
   &\quad + \underbrace{\pare{  \frac{ 5L}{2\bar\mu  }    + \frac{4 }{ L} }\frac{2G^2 + 2W^2L^2}{N } \sum_{i=1}^N d(1-p_i)  }_{\mathrm{\text{Residual error due to masked updates}}},
    \end{align*}
where $\bar\mu := \frac{1}{N}\sum_{i=1}^N p_i \mu$, $\tilde \mu: = \min_{i\in[N]} p_i \mu$, $\tilde L: = \max_{i\in[N]} p_i L$ and $\tilde \kappa:= \tilde L/\tilde \mu$.
\end{theorem}
Here we achieve an $O(\frac{1}{R^2} + \frac{1}{NKR})$  rate plus residual error due to masked updates.  If each client chooses masking probability to be $1$, i.e., enabling full model training, the residual error vanishes and we recover the convergence of heterogeneous Local SGD~\citep{woodworth2020minibatch,khaled2019tighter}. 

\noindent\textbf{\ding{110}~Comparison to existing works.}  \cite{shulgin2023towards} studied the special scenario where $f(\bw)$ is quadratic, and also proved that the convergence rate will suffer from a residual error. \cite{demidovich2023mast} studied single machine and distributed fully synchronized versions of Algorithm~\ref{algorithm: Masked FedAvg}, and proved convergence to the optimal point of the masked objective, i.e., the minimizer of $  F_{\cD}(\bw) :=   \E_{\mask  \sim \cD} [f (\mask \odot\bw)] $ where $\cD$ represents a distribution on masking vector.

\subsection{Convergence in Nonconvex Setting}

In this section, we will present convergence result of Algorithm~\ref{algorithm: Masked FedAvg} in nonconvex setting. We will need the following heterogeneity measure.

\begin{definition}
We define the masked gradient dissimilarity as follows: 
 {\small \begin{align*}
   \max_{\bw\in\R^d }\frac{1}{N}\sum\nolimits_{i=1}^N\E_{\mask_i}\norm{\mask_i\odot\nabla   f_i(\mask_i\odot\bw ) - \nabla  F_{\bp}(\bw) }^2 \leq \zeta^2_{\bp},
    \end{align*} }
where $\mask_i\sim Ber(p_i), i=1, \ldots, N$ and   $F_{\bp}(\bw) :=\frac{1}{N}\sum_{i=1}^N \E_{\mask_i\sim Ber(p_i)} [f_i(\mask_i\odot\bw)]  $.
\end{definition}
Similar definition can be found in the classical Local SGD analysis~\cite{woodworth2020minibatch}, but here dissimilarity is defined over the masked objective. A more aggressive masking scheme will result in a smaller  $\zeta^2_{\bp}$.\vspace{-2mm}
\begin{theorem}\label{thm:random mask nonconvex}
Let Assumptions~\ref{assump: smooth} and \ref{assumption: bounded var} hold. Then Algorithm~\ref{algorithm: Masked FedAvg} with $\eta = \Theta\pare{\frac{1}{L\sqrt{KR }}}$  guarantees that  the following statement holds for $F_{\bp}$ as defined in Eq.~(\ref{eq: obj random}):
\begin{align*}
 \frac{1}{R}\sum_{r=1}^R\E\norm{\nabla  F_{\bp}(\bw_r)}^2  \leq O\pare{\frac{\tilde L\E [  F_{\bp}(\bw_0) ]   }{\sqrt{R K} }  +     \frac{K\zeta^2_{\bp}}{R}    +  \frac{K\delta^2}{R}          +  \frac{ \delta^2}{ N\sqrt{RK}}}.
\end{align*}\vspace{-2mm}
\end{theorem}
We can see that sub-model training will converge to the stationary point of an alternative objective induced by masking, not the raw objective $F(\bw)$.
We achieve $O(\frac{1}{\sqrt{RK}} + \frac{K\zeta^2_\bp}{R})$ rate, analogous to the rate of FedAvg with full model training~\cite{haddadpour2019convergence}.

\noindent\textbf{{\ding{110}}~Comparison to existing works.}~\cite{mohtashami2022masked} studied single machine sub-model training setting, i.e., $N=1$. Given a fixed sequence of masks at each iteration, i.e., $\mask_1,\ldots,\mask_T$,
they perform update $\bw_{t+1} = \bw_t - \eta \mask_t \odot \nabla F(\mask_t\odot\bw_t;\xi_t)$.
They proved $\frac{1}{T}\sum_{t=1}^T\alpha_t\E \norm{\mask_t\odot F(\mask_t\odot\bw_t)}^2 \leq O(\frac{1}{\sqrt{T}})$, where $\alpha_t := \frac{\inprod{\mask_t\odot \nabla F(\bw_t)}{ \mask_t\odot \nabla F(\mask_t\odot\bw_t)}}{\norm{\mask_t\odot \nabla F(\mask_t\odot\bw_t)}^2}$. However, it is not shown how small the $\alpha_t$ can be. If $\alpha_t$s approach zero, the convergence measure loses significance.~\cite{zhou2024every}studies a similar algorithm and examines convergence via the gradient norm; however, their analysis yields a bound that depends on the norm of history models, i.e., $\norm{\bw_t}$. With the model's norm uncontrolled in the bound, it is possible that the convergence bound becomes vacuous.

\noindent\textbf{{\ding{110}}~On the stationary points of $F$ and $F_{\bp}$.}~In~\citep{zhou2024every}, the authors proved the convergence of the model to the stationary point of $F$, with residual error depending on norm of iterates, i.e., $\norm{\bw_r}$. Indeed, our result can also be translated to the stationarity of $F$, by using norm of iterates. To see this,
assume $\tilde \bw_\epsilon$ is $\epsilon$ stationary point of $F_{\bp}$. If we evaluate $F$'s gradient at $\tilde \bw_\epsilon$ we have:
 \begin{align*}
   \norm{\nabla F(\tilde\bw_\epsilon)}^2     &\leq 2\norm{\nabla  F_{\bp}(\tilde\bw_\epsilon)}^2 + 2\norm{\nabla  F_{\bp}(\tilde\bw_\epsilon) - \nabla F(\tilde\bw_\epsilon)}^2 \\
   & \leq 2\epsilon^2 + 2\norm{  \frac{1}{N}\sum_{i=1}^N  \E \sbr{ \mask_i \odot \nabla f_i(\mask_i \odot\tilde\bw_\epsilon) }-  \nabla F(\tilde\bw_\epsilon)}^2\\
   & \leq 2\epsilon^2 +\frac{1}{N}\sum_{i=1}^N d(1-p_i)(G^2 + L^2\norm{\tilde \bw_\epsilon}^2 ).
\end{align*} 
Thus, any $\epsilon$-stationary point of $F_\bp$ is also $\sqrt{2\epsilon^2 +\frac{1}{N}\sum_{i=1}^N d(1-p_i)(G^2 + L^2\norm{\tilde \bw_\epsilon}^2 )}$-stationary point of $F$.
\section{Distributed Sub-model Training via Rolling Masking}\label{sec:conv:shuffling}
Another popular algorithm for distributed sum-model training is via rolling masking~\cite{alam2022fedrolex}, where each of the whole model parameters is divided into several pre-defined sub-models, 
\begin{algorithm2e}[t] 
	\DontPrintSemicolon
    \caption{\sffamily{Rolling Masked FedAvg }}
    	\label{algorithm: Rolling Masked FedAvg}
    	\textbf{Input:} Initial model $ {\bw}_{0} = \mathbf{0} $, pre-defined masking vectors $ \cbr{\cbr{\mask_i^j}_{j=1}^R}_{i=1}^N  $,  stepsizes $\eta$.    
	\\ 
\For{$e = 0,\ldots,T-1$}{
   Server generates random permutation $\sigma_e: [R]\mapsto[R]$.\\
    $\bw_e = \bw_{e-1,r}$\\
 \For{$r = 1,\ldots, R$}{
  Server distributes $ \bw_{e,r}^i = \mask_i^{\sigma_e(r)} \odot \bw_{e,r}$ to $i$th user.\\
  \For{$i = 1,...,N$}{

 \For{$k = 0,\ldots,K-1$}
	 {   
            $ \bw_{e,r,k+1}^i = \bw_{e,r,k}^i - \eta \mask^{\sigma_e(r)}_i \odot \nabla f_i(\mask^{\sigma_e(r)}_i\odot\bw_{e,r,k}^i; \xi_{e,r,k}^i)  $  
  	 
   } 
   $i$th Client sends $ \bw_{e, r,K}^i$ back to server.\\ 

}
  Server averages models $\bw_{e,r+1} = \cP_{\cW}\pare{  \frac{1}{N}\sum_{i=1}^N (\bw_{e, r,K}^i + (\mathbf{1}- \mask_i^{\sigma_e(r)}) \odot \bw_{e, r } )}$\\
 
}
}	   
       \textbf{Output:}   $\hat{\bw} = \cP_{\cW}\pare{ \bw_{T } - \frac{1}{L} \frac{1}{N}\sum_{i=1}^N \frac{1}{R}\sum_{r=1}^R\mask_i^r \odot \nabla f_i(\mask_i^r\odot\bw_{T } ) }$ 
\end{algorithm2e}
\begin{align*}
   \mask_i^1\odot \bw,\ldots, \mask_i^R\odot \bw
\end{align*}
where $\mask_i^j$ is such that $j$ to $(j+s_i)\mod d$ coordinate is $1$ and rest are zero, $s_i$ is the sub-model size for $i$th client. At the beginning of $e$th epoch, server generates a permutation $\sigma_e: [R]\mapsto[R]$ and shuffles the clients' sub-model according to $\sigma_e$:
\begin{align*}
   \mask_i^{\sigma_e(1)}\odot \bw,\ldots, \mask_i^{\sigma_e(R)}\odot \bw.
\end{align*}
Then, $i$th client will optimize on those sub-models sequentially at each round. At the beginning of $r$th round, it performs the following local updates:
\begin{align*}
 \bg_{e,r,k}^i&= \mask^{\sigma_e(r)}_i \odot \nabla f_i(\mask^{\sigma_e(r)}_i\odot\bw_{e,r,k}^i; \xi^i_{e,r,k})\\
    \bw_{e,r,k+1}^i &= \bw_{e,r,k}^i - \eta \bg_{e,r,k}^i. 
\end{align*}

After $K$ local steps, $i$th client sends model $\mask_i^{\sigma_e(r)} \odot \bw_{e,r,K}$ to server, and server averages the local models:
{\small\begin{align*}
    \bw_{e,r+1} = \cP_{\cW}\pare{  \frac{1}{N}\sum\nolimits_{i=1}^N (\bw_{e, r,K}^i + (\mathbf{1}- \mask_i^{\sigma_e(r)}) \odot \bw_{e, r } )}.
\end{align*}}
In words, server will fill the parameters not selected by $\mask_i^{\sigma_e(r)}$ with old parameters of the model from last round model $\bw_{e,r}$, and then average all clients models.
Next, server distributes $\mask_i^{\sigma_e(r+1)} \odot \bw_{e,r+1}$ to $i$th client to proceed another round of local updates. The pseudo-code  is depicted in Algorithm~\ref{algorithm: Rolling Masked FedAvg}.

\subsection{Convergence in Convex Setting}
\begin{assumption}[{\sffamily{Bounded variance}}]\label{assumption: bounded var rolling}
The variance of stochastic gradients computed at each local function is bounded, i.e.,  $\forall i \in [N],\forall r \in [R], \forall \bw \in \cW, \mathbb{E}_{ \xi}[\|\mask_i^r\odot\nabla f_i(\mask_i^r\odot\bw;\xi) - \mask_i^r\odot\nabla f_i(\mask_i^r\odot\bw)\|^2]  \leq \delta^2$.
\end{assumption}

\begin{definition}\label{def:rolling mask grad dis}
Given a  masking configuration $\mask= \sbr{\sbr{\mask_i^1,\ldots,\mask_i^R}}_{i=1}^N \in \cbr{0,1}^{dNR}$, we define the masked gradient dissimilarity as follows:
   {\small \begin{align*}
     \max_{\bw\in\cW} \frac{1}{NR}\sum_{i=1}^N   \sum_{j=1}^{R}  \norm{  \mask_i^{j} \odot\nabla f_i ( \mask_i^{j} \odot \bw     ) - \nabla   F_{\mask}(  \bw    ) }^2 \leq \zeta^2_{\mask}.
    \end{align*}}
\end{definition}

Assumption~\ref{assumption: bounded var rolling} and Definition~\ref{def:rolling mask grad dis} are analogous to Assumption~\ref{assumption: bounded var} and Definition~\ref{def:random mask grad dis}, but here the masking vectors are deterministic.

\begin{theorem}\label{thm:rolling mask}
Let Assumptions~\ref{assump: smooth},~\ref{assump: sc},~\ref{assumption: bounded domain},~\ref{assumption: bounded grad} and~\ref{assumption: bounded var rolling} hold. Then Algorithm~\ref{algorithm: Rolling Masked FedAvg} with $\eta = \Theta\pare{\frac{\log(T^2)}{\mu KRT}}$ and $T \geq 512 \kappa^2 \log T$ will output the solution $\hat{\bw} $, such that with probability at least $1-\nu$, the following statement holds :
    \begin{align*}
       F(\hat\bw) - F(\bw^*)& \leq     O\pare{ \frac{ L\E\norm{ \bw_0 - \bw^*}^2 }{T^2}}   +  L O\pare{   \frac{\kappa^2  \log( RK/\nu)  }{\mu T^2 R  }    +    \frac{\kappa   \zeta^2_\mask \log^2(T )}{\mu^2 T^2   R^2}        +        \frac{  \delta^2 \log(T)}{\mu^2 T N}} \\
        & \quad + \underbrace{ \pare{  \frac{ L}{\mu  }    + \frac{1 }{ L}  } \frac{ G^2 +   W^2L^2}{NR}        \sum_{i=1}^N \sum_{j=1}^R \norm{\mask_i^j - \mathbf{1}}^2
  }_{\text{Residual error due to masked updates}}.
\end{align*}
\end{theorem}
The first part of the rate is contributed from shuffled Local SGD, and the second part is due to masking updates. If each mask $\mask_i^j = \mathbf{1}$, i.e., full model training, we can get rid of this residual error. Note that our algorithm is different from Cyclic FedAvg~\cite{cho2023convergence}, where they do the shuffling on the client level, and for each communication round, a subset of clients are picked to do local updates.  

The proof of theorem is deferred to Appendix~\ref{app:proof_fedrolex}. We note that in our convergence analysis, we account for the partitioning of the full model into $R$ sub-models. At each epoch, the server shuffles and sequentially assigns these sub-models to clients, introducing analytical challenges due to model drift from partial training and the effects of permutation-based assignments. Our technical contribution lies in jointly addressing these challenges to establish convergence, an aspect that is interesting in its own right.
\begin{remark}
  \cite{deng2024distributed} proposed and studied similar algorithm, but in their work, the clients directly optimize the full models, while in ours, clients optimize local models and server performs averaging periodically. Another relevant work is~\citep{cho2023convergence}, where they studied Local SGD with cyclic client participation. The main difference is that they do client-level shuffling and at each round server only picks a subset of clients to participate training.  
\end{remark}

\subsection{Convergence in Nonconvex Setting}
In this section, we will present convergence result of Algorithm~\ref{algorithm: Rolling Masked FedAvg} in nonconvex setting.

 \begin{theorem}\label{thm:rolling mask nonconvex}
Let Assumptions~\ref{assump: smooth}, \ref{assumption: bounded var} and \ref{assumption: bounded grad} hold. Then Algorithm~\ref{algorithm: Rolling Masked FedAvg} with $\eta = \Theta\pare{\frac{1}{L\sqrt{KRT}}}$ guarantees that with probability at least $1-\nu$, for $F_\mask$  defined as in Eq. (\ref{eq: obj rolling}) the following  holds true:
    \begin{align*}
 \frac{1}{T}\sum\nolimits_{e=1}^T  \E \norm{ \nabla F_{\mask}(\bw_e)}^2 \leq O\pare{\frac{L\E [  F_{\mask}(\bw_0)]}{\sqrt{RK T}}       +    \frac{K^2  \zeta^2_\mask}{  TL^2}    +      \frac{\delta^2}{N\sqrt{ RKT} L}     }.
    \end{align*}
\end{theorem}
The convergence rate here matches that of the random masking case, measured by the gradient norm of an alternative objective $F_{\mask}$, induced by sub-model selection scheme.

\noindent\textbf{\ding{110}~On the stationary points of $F$ and $F_{\mask}$.}~We can also translate stationarity of $F_{\mask}$ to that of $F$ as follows. Similar to random masking setting, we can also translate the stationarity between $F$ and $F_{\mask}$.
Assume $\tilde \bw_\epsilon$ is $\epsilon$ stationary point of $  F_{\mask}$, if we evaluate $F$'s gradient at $\tilde \bw_\epsilon$ we have:
{\small  \begin{align*}
   &\norm{\nabla F(\tilde\bw_\epsilon)}^2  \leq 2\norm{\nabla   F_{\mask}(\tilde\bw_\epsilon)}^2  + 2\norm{\nabla  F_{\mask}(\tilde\bw_\epsilon) - \nabla F(\tilde\bw_\epsilon)}^2 \\
   & \leq 2\epsilon^2 + 2\norm{  \frac{1}{N}\sum\nolimits_{i=1}^N (\frac{1}{R}\sum\nolimits_{j=1}^R \mask_i^j \odot \nabla f_i(\mask_i^j \odot\tilde\bw_\epsilon) -     \nabla f_i(\tilde\bw_\epsilon))}^2\\
   & \leq 2\epsilon^2 +\frac{1}{N}\sum\nolimits_{i=1}^N \frac{1}{R}\sum\nolimits_{j=1}^R \norm{ \mathbf{1} -  \mask_i^j  }^2(G^2 + L^2\norm{\tilde \bw_\epsilon}^2 ).
\end{align*}}
\section{On the Stability of Masked Training}\label{sec:stability}
In this section, we will study the generalization ability of distributed sub-model training with Bernoulli and rolling based  masking.  Formally, for a learning algorithm $\cA$ and training dataset $\cS$ drawn from distribution $\cD$, the generalization error is defined as $
  \epsilon_{gen}:=  \E_{\cA,\cS}\left| \cL_{\cD}( \cA(\cS) ) - \cL_{\cS}( \cA(\cS) ) \right|
$. A classical way to study the above error is algorithmic stability. We adopt the following definition of stable federated learning algorithm from~\citep{sun2024understanding}.
\begin{definition}\label{def:avg stab}
A federated learning algorithm $\cA$ is said
to have $\epsilon$-on-average stability if given any two neighboring datasets $\cS$ and $\cS^{(i)}$
, then $\forall i \in [N], j \in [n]$
\begin{align*}
    \E_{\cA,\cS, z'_{i,j}} \left| \ell( \cA(\cS);z'_{i,j} ) - \ell( \cA(\cS^{(i)});z'_{i,j} ) \right| \leq \epsilon.
\end{align*}
where $\cS$ and $\cS^{(i)}$ are the two dataset only differing at $j$th point of $i$th client's dataset, i.e.,  $\cS = \cbr{\ldots, z_{i,j}, \ldots}$ and $\cS^{(i)}= \cbr{\ldots, z'_{i,j}, \ldots}$.
\end{definition}

An immediate implication of $\epsilon$-on-average stability is the following lemma on generalization error.
\begin{lemma}\label{lem:generalization}\citep{sun2024understanding}
If $\cA$ is an $\epsilon$-on-average stable algorithm, then
    \begin{align*}
      \epsilon_{gen}\leq   \E_{\cA,\cS}\left| \cL_{\cD}( \cA(\cS) ) - \cL_{\cS}( \cA(\cS) ) \right| \leq \epsilon.
    \end{align*}
\end{lemma}
Lemma~\ref{lem:generalization} indicates that any $\epsilon$-on-average stable algorithm will admit the expected generalization error no larger than $\epsilon$ as well. Hence, to study the generalization of Algorithm~\ref{algorithm: Masked FedAvg}, it suffices to bound the difference between the models trained on raw training set and one-point-perturbed dataset as 
\begin{align*}
  \E_{\cA,\cS, z'_{i,j}} \left| \ell( \cA(\cS);z'_{i,j} ) - \ell( \cA(\cS^{(i)});z'_{i,j} ) \right| \leq  \E_{\cA,\cS, z'_{i,j}} G\norm{  \cA(\cS)  -   \cA(\cS^{(i)}) }.  
\end{align*}

Hence, we are aimed at studying $\norm{  \cA(\cS)  -   \cA(\cS^{(i)}) }$.
We will make the following assumption.
\begin{assumption}[{\sffamily{Convexity}}]\label{assump: cvx} We assume $f_i(\bx)$'s are convex, i.e., $
\forall \bx, \by: f_i(\by) \geq f_i(\bx)+\langle\nabla f_i(\bx), \by-\bx\rangle$.
\end{assumption}

\begin{definition}\label{def:random mask grad dis} 
We define the maximum pair-wise discrepancy as $D_{\max} = \sup_{i,j \in[N] }  d_{\mathrm{TV}} (\cD_i,\cD_j)$.
\end{definition}

\begin{theorem}~\label{thm:stab random cvx}[Stability of Random Masking]
Let Assumptions~\ref{assump: smooth},~\ref{assumption: bounded var},~\ref{assumption: bounded domain} and~\ref{assump: cvx} hold. We assume each client has $n$ training data drawn from its distribution. Let $\hat\bw$ and $\hat\bw'$ be the output model of Algorithm~\ref{algorithm: Masked FedAvg} on dataset $\cS$ and $\cS^{(i)}$. Then, if we choose $\eta=\frac{\sqrt{Nn}}{RK}$ and $R\geq \tilde L\sqrt{Nn}$, it holds that 
    \begin{align*}
     \E\norm{\hat\bw  - \hat\bw' } \leq     O\pare{  \frac{   \delta+ G D_{\max}}{\sqrt{Nn}} +  \frac{1}{\sqrt{Nn}}
      \sqrt{ \frac{ \tilde L }{\sqrt{Nn}}    +          \sigma_*^2  +    {\delta^2}    }  },
    \end{align*}
    where  $\tilde L: = \max_{i\in[N]} p_i L$  . 
\end{theorem}

\begin{remark}
    Even though we assume each client has the same amount of data, the analysis can be easily extended to the case where $i$th client has $n_i$ data samples, but we have to assume our objective is also weighted accordingly, i.e., $F(\bw) = \sum_{i=1}^N \frac{n_i}{n} f_i(\bw)$.
\end{remark}

\begin{corollary}
Let Assumptions~\ref{assump: smooth},~\ref{assumption: bounded var},~\ref{assumption: bounded domain} and~\ref{assump: cvx} hold. Then, if we choose $\eta=\frac{\sqrt{Nn}}{RK}$ and $R\geq \tilde L\sqrt{Nn}$, Algorithm~\ref{algorithm: Masked FedAvg} admits the generalization error:
    \begin{align*}
   \epsilon_{gen} \leq  O\pare{  \frac{ G(\delta+ G D_{\max}) }{\sqrt{Nn}} +  \frac{G}{\sqrt{Nn}}
      \sqrt{ \frac{ \tilde L }{\sqrt{Nn}}    +          \sigma_*^2  +    {\delta^2}    }  }.
    \end{align*}
\end{corollary}

 \begin{remark}
    We can see that all terms depend on masking probability. A smaller masking probability $\cbr{p_i}_{i=1}^N$ will result in a smaller $ \delta, \tilde L, \sigma_*, \zeta_{\max} $. Hence masking improves the generalization (empirical risk vs population risk) by stabilizing the training process, as validated by our  experiments. However, a lower masking probability $p_i$ will also increase the residual error of the convergence or empirical risk, as stated in Theorem~\ref{thm:randomly masked fedavg}. Consequently, the above theorem implies that masking can enhance generalization, as long as the residual optimization error from partial training remains controlled.
\end{remark}

\begin{remark}
    A related work to ours is~\cite{fu2023effectiveness}, which  also studies the generalization error of sparse training, and has similar conclusion that sparsity can improve the algorithmic stability. The main difference to ours is that they consider regularized ERM algorithm, while our analysis is built for an iterative distributed optimization algorithm.
\end{remark}

\begin{theorem}~\label{thm:stab rolling cvx}[Stability of Rolling Masking]
Let Assumptions~\ref{assump: smooth},~\ref{assumption: bounded var},~\ref{assumption: bounded domain} and~\ref{assump: cvx} hold. We assume each client has $n$ training data drawn from its distribution. Let $\hat\bw$ and $\hat\bw'$ be the output models of Algorithm~\ref{algorithm: Rolling Masked FedAvg} on dataset $\cS$ and $\cS^{(i)}$, respectively. Then, if we choose $\eta=\frac{\sqrt{Nn}}{RK}$ and $T\geq  \sqrt{\frac{n}{N}}$, we have:
    \begin{align*}
 \E\norm{\hat{\bw}   - {\hat \bw'} }     
      \leq    O\pare{  \frac{   \delta+GD_{\max} }{\sqrt{Nn}} +  \frac{1}{\sqrt{Nn}}
      \pare{ \sqrt{ \frac{1}{\sqrt{Nn}}    +    G^2 D_{\max}^2    +      \delta^2  }   }  }.   
\end{align*}
\end{theorem}
Here we achieve similar $\frac{1}{\sqrt{Nn}}$ rate as random masking, indicating that rolling masking can also enjoy a more stable training dynamic, leading to a better generalization rate.

 \section{Experiments}\label{sec:exp}
In this section, we present a comprehensive evaluation of different masking algorithms through a series of experiments designed to assess their performance across various scenarios. Additional results are reported in Appendix~\ref{app: exp}.

\subsection{Experiment Setup}
\noindent\textbf{Datasets and Models.}~We evaluate the performance of different masking in the following scenarios. We train pre-activated ResNet18 models on CIFAR-10 and CIFAR-100. We modify the ResNet18 architecture by replacing batch normalization with static batch normalization and incorporating a scalar module after each convolutional layer.\\

\noindent\textbf{Data Heterogeneity.}~To create non-IID data distributions for CIFAR-10 and CIFAR-100 into 100 clients respectively, we follow FedRolex~\cite{alam2022fedrolex}, restricting each client to have access to only $L$ labels. We evaluate two levels of data heterogeneity. For CIFAR-10, we set $L=2$ as high data heterogeneity and $L=5$ as low data heterogeneity, which corresponds to the Dirichlet distribution with $\alpha = 0.1$ and $\alpha = 0.5$, respectively. For CIFAR-100, we set $L=20$ as high data heterogeneity and $L=50$ as low data heterogeneity, which also corresponds to the Dirichlet distribution with $\alpha = 0.1$ and $\alpha = 0.5$, respectively. Results for high data heterogeneity and low data heterogeneity are both presented in the following subsections.\\

\noindent\textbf{Model Heterogeneity.}~For ResNet18, client model capacities $\beta = \{1,1/2,1/4,1/8,1/16\}$ are used for evaluation, i.e., $1/16$ means the client model capacity is $1/16$ of the largest client model capacity (full model). We vary the number of kernels in the convolutional layers while maintaining the same number of nodes in the output layers.

\subsection{Convergence of  Masked Training}

\noindent\textbf{Model-Heterogeneous Setting.}~We compare the performance of rolling and random masking in model-heterogeneous scenarios, following prior work where client capacities are uniformly distributed and the global server model is the same as the largest client model. Fig.~\ref{fig:modelhetero_highdatahetero_globalloss_com_rollex_random} and Fig.~\ref{fig:modelhetero_highdatahetero_globalloss_com_rollex_random_cifar100} show the global model testing loss under high data heterogeneity, indicating that rolling masking outperforms random in both datasets. Similarly, Fig.~\ref{fig:modelhetero_highdatahetero_globalacc_com_rollex_random} and Fig.~\ref{fig:modelhetero_highdatahetero_globalacc_com_rollex_random_cifar100} illustrate the corresponding global model testing accuracy. Consistent with the loss results, rolling masking outperforms random in both datasets. Under low data heterogeneity, Fig.~\ref{fig:modelhetero_lowdatahetero_globalloss_com_rollex_random} and Fig.~\ref{fig:modelhetero_lowdatahetero_globalloss_com_rollex_random_cifar100} show the global model testing loss, while  Fig.~\ref{fig:modelhetero_lowdatahetero_globalacc_com_rollex_random} and Fig.~\ref{fig:modelhetero_lowdatahetero_globalacc_com_rollex_random_cifar100} show the global model testing accuracy, respectively. They demonstrate that the results under low data heterogeneity follow the high data heterogeneity. \\

\begin{figure}[t]
    \centering
    \subfigure[{CIFAR-10}]{
        \includegraphics[width=0.22\textwidth]{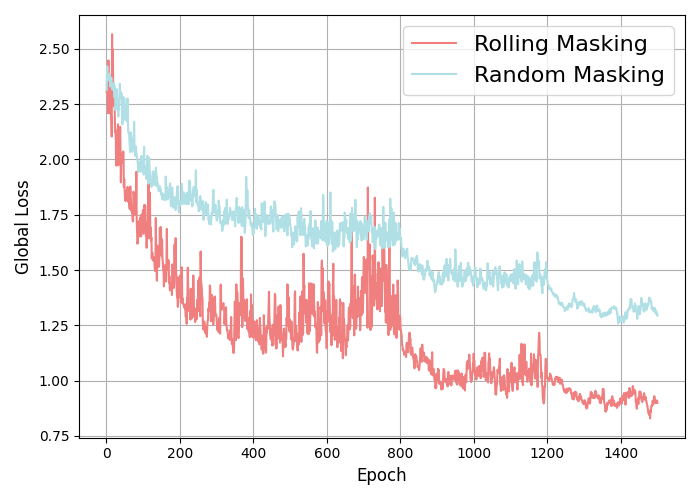}
        \label{fig:modelhetero_highdatahetero_globalloss_com_rollex_random}
    }
    \subfigure[{CIFAR-100}]{
        \includegraphics[width=0.22\textwidth]{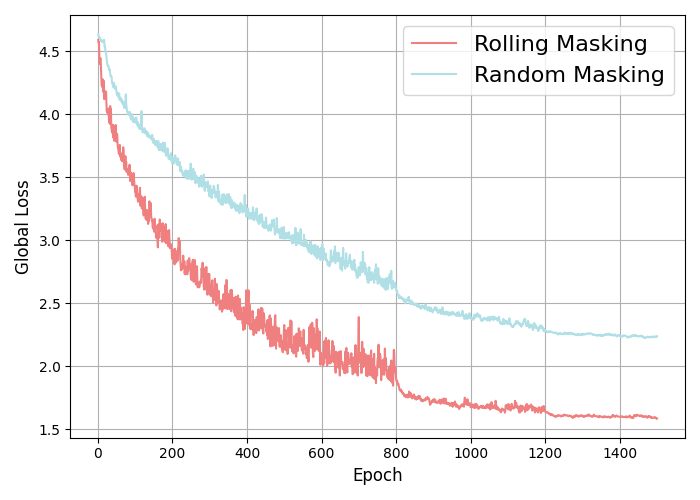}
        \label{fig:modelhetero_highdatahetero_globalloss_com_rollex_random_cifar100}
    }
    \subfigure[{CIFAR-10}]{
        \includegraphics[width=0.22\textwidth]{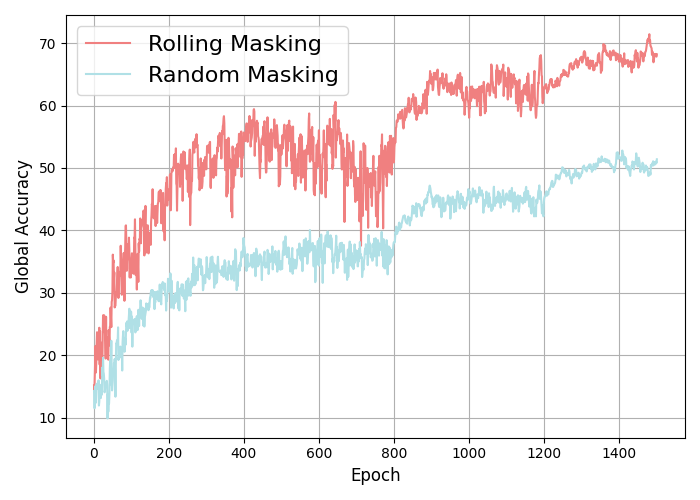}
        \label{fig:modelhetero_highdatahetero_globalacc_com_rollex_random}
    }
    \subfigure[{CIFAR-100}]{
        \includegraphics[width=0.22\textwidth]{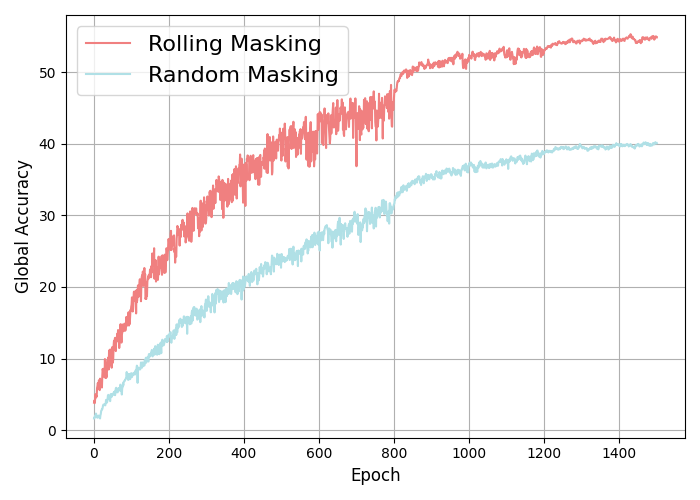}
        \label{fig:modelhetero_highdatahetero_globalacc_com_rollex_random_cifar100}
    }
    \caption{Global testing loss/accuracy of rolling and random masking under high data heterogeneity.}
    \label{fig:global_loss_high_data_Heterogeneity_rollex_random}
    \vskip -0.12in
\end{figure}
\begin{figure}[t]
\vskip -0.1in
    \centering
    \subfigure[{CIFAR-10}]{
        \includegraphics[width=0.22\textwidth]{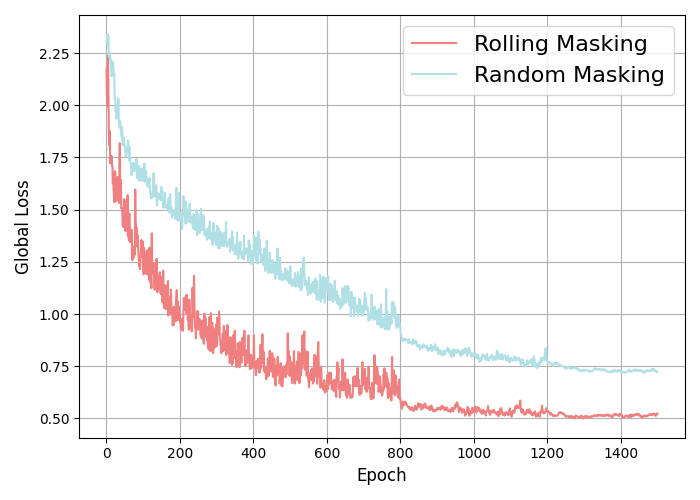}
        \label{fig:modelhetero_lowdatahetero_globalloss_com_rollex_random}
    }
    \subfigure[{CIFAR-100}]{
        \includegraphics[width=0.22\textwidth]{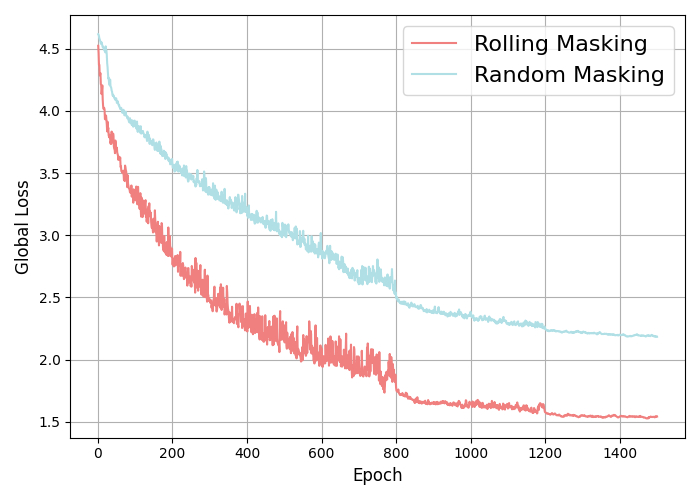}
        \label{fig:modelhetero_lowdatahetero_globalloss_com_rollex_random_cifar100}
    }
    \subfigure[{CIFAR-10}]{
        \includegraphics[width=0.22\textwidth]{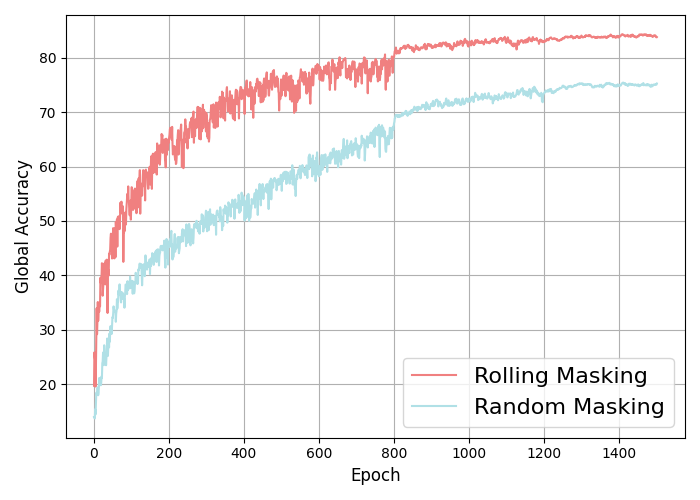}
        \label{fig:modelhetero_lowdatahetero_globalacc_com_rollex_random}
    }
    \subfigure[{CIFAR-100}]{
        \includegraphics[width=0.22\textwidth]{Exp_Figs/modelhetero_lowdatahetero_globalacc_com_rollex_random.png}
        \label{fig:modelhetero_lowdatahetero_globalacc_com_rollex_random_cifar100}
    }
    \caption{Global testing loss/accuracy of rolling and random masking with low data heterogeneity.}
    \label{fig:global_acc_low_data_Heterogeneity_rollex_random}
    \vskip -0.1in
\end{figure}

\noindent\textbf{Model-Homogeneous Setting.}~We compare global model testing loss/accuracy for CIFAR-10 in two model-homogeneous cases: all clients have the largest capacity model $(\beta$=$1)$ and all clients have the smallest capacity model $(\beta$=$1/16)$, representing the upper and lower performance bounds. For the global testing loss, Fig.~\ref{fig:modelhomo_largest_highdatahetero_globalloss_com_rollex_random} and Fig.~\ref{fig:modelhomo_smallest_highdatahetero_globalloss_com_rollex_random} show that in both cases, rolling masking outperforms random. For global testing accuracy, Fig.~\ref{fig:modelhomo_largest_highdatahetero_globalacc_com_rollex_random} and Fig.~\ref{fig:modelhomo_smallest_highdatahetero_globalacc_com_rollex_random} are also consistent with the conclusion. Additionally, the largest model consistently achieves better performance than the smallest model. Similarly, the global testing loss shown in Fig.~\ref{fig:modelhomo_largest_lowdatahetero_globalloss_com_rollex_random} and Fig.~\ref{fig:modelhomo_smallest_lowdatahetero_globalloss_com_rollex_random}, and the global testing accuracy shown in Fig.~\ref{fig:modelhomo_largest_lowdatahetero_globalacc_com_rollex_random} and Fig.~\ref{fig:modelhomo_smallest_lowdatahetero_globalacc_com_rollex_random} show that low data heterogeneity has the same conclusion as the high data heterogeneity in both scenarios.

\begin{figure}[t]
\vskip -0.1in
    \centering
    \subfigure[{Largest capacity}]{
        \includegraphics[width=0.22\textwidth]{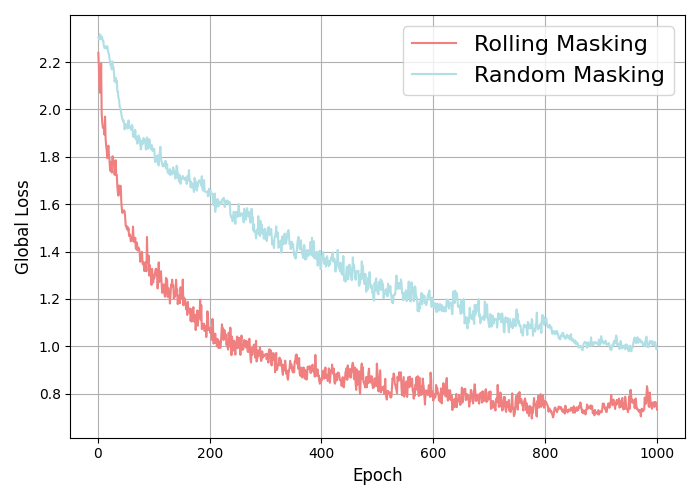}
        \label{fig:modelhomo_largest_highdatahetero_globalloss_com_rollex_random}
    }
    \subfigure[{Smallest capacity}]{
        \includegraphics[width=0.22\textwidth]{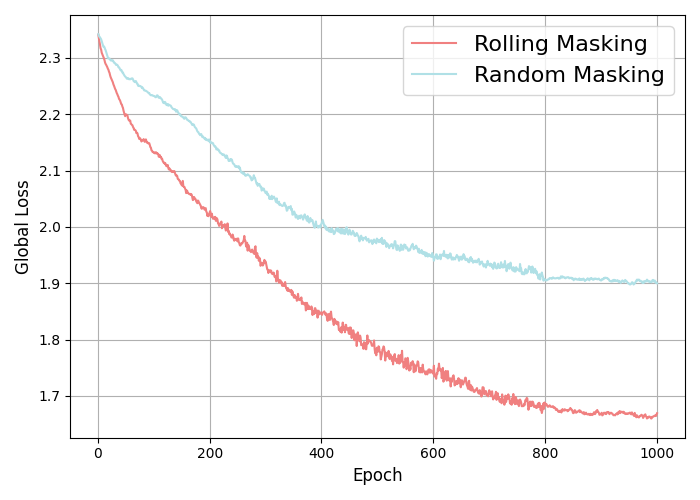}
        \label{fig:modelhomo_smallest_highdatahetero_globalloss_com_rollex_random}
    }
    \subfigure[{Largest capacity}]{
        \includegraphics[width=0.22\textwidth]{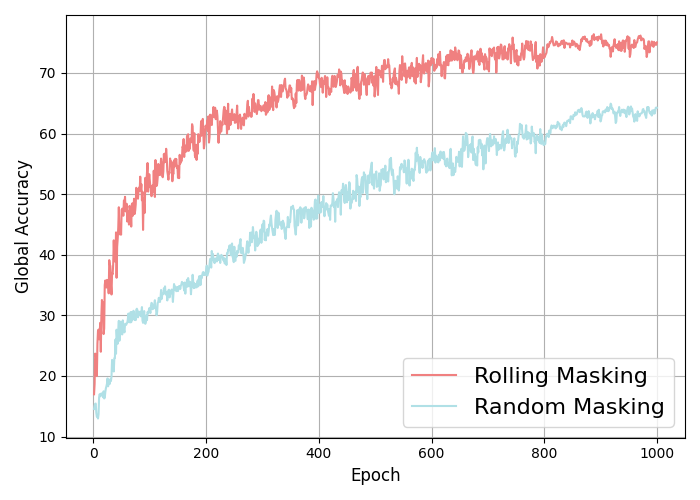}
        \label{fig:modelhomo_largest_highdatahetero_globalacc_com_rollex_random}
    }
   \subfigure[{Smallest capacity}]{
        \includegraphics[width=0.22\textwidth]{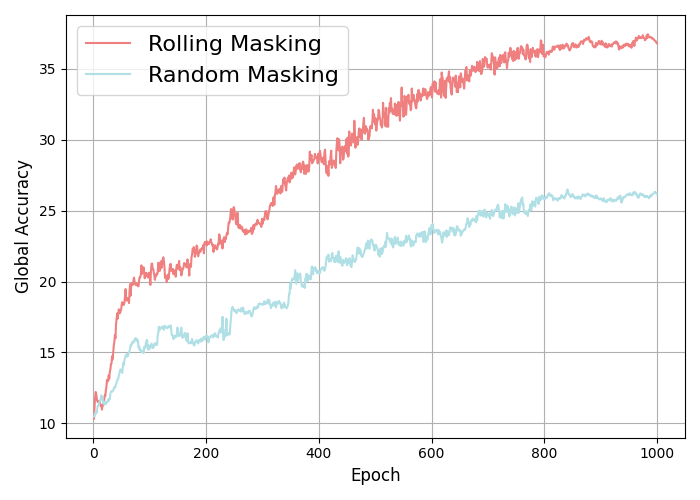}
        \label{fig:modelhomo_smallest_highdatahetero_globalacc_com_rollex_random}
    }
    \caption{Global testing loss/accuracy of rolling and random masking under the largest and smallest client model capacity under high data heterogeneity.}
    \label{fig:global_loss_model_homo_high_data_Heterogeneity_rollex_random}
\end{figure}
\begin{figure}[t]
\vskip -0.1in
    \centering
    \subfigure[{Largest Model Capacity}]{
        \includegraphics[width=0.22\textwidth]{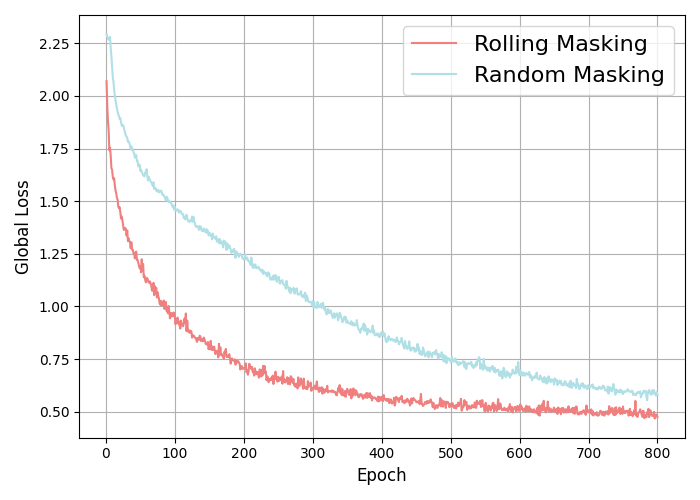}
        \label{fig:modelhomo_largest_lowdatahetero_globalloss_com_rollex_random}
    }
    \subfigure[{Smallest Model Capacity}]{
        \includegraphics[width=0.22\textwidth]{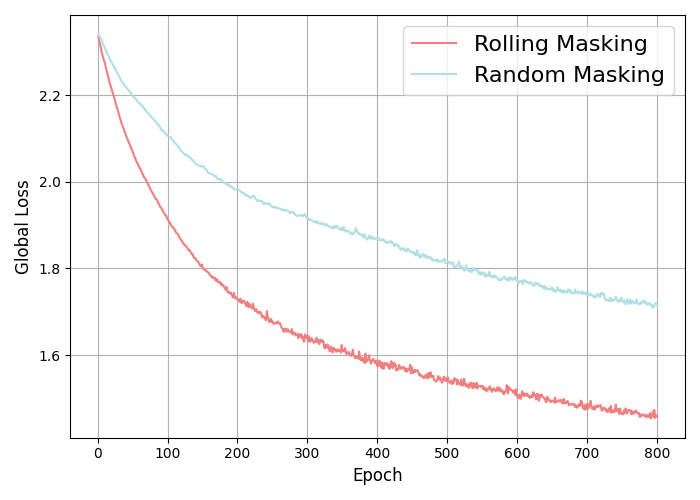}
        \label{fig:modelhomo_smallest_lowdatahetero_globalloss_com_rollex_random}
    }
    \subfigure[{Largest Model Capacity}]{
        \includegraphics[width=0.22\textwidth]{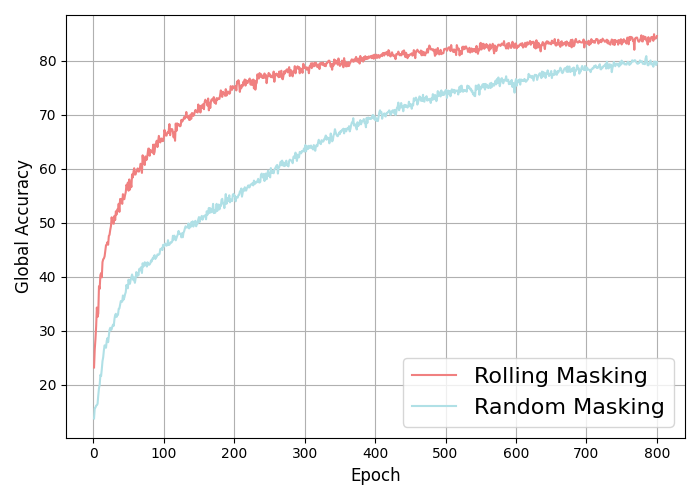}
        \label{fig:modelhomo_largest_lowdatahetero_globalacc_com_rollex_random}
    }
   \subfigure[{Smallest Model Capacity}]{
        \includegraphics[width=0.22\textwidth]{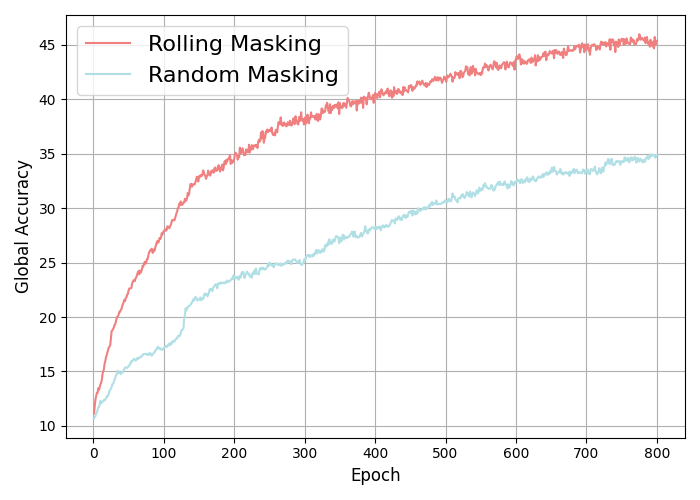}
        \label{fig:modelhomo_smallest_lowdatahetero_globalacc_com_rollex_random}
    }
    \caption{Global testing loss/accuracy of rolling and random masking under the largest and smallest client model capacity under low data heterogeneity.}
    \label{fig:global_acc_model_homo_low_data_Heterogeneity_rollex_random}
    \vskip -0.1in
\end{figure}

\begin{table}[t]
    \centering
    \caption{Generalization of random masking and full model training under high data heterogeneity.}
     \setlength{\tabcolsep}{3pt}
     \begin{tabular}{l l}
    \toprule
                    & \textbf{Global Loss Diff.} \\
    \midrule
      Random masking & $|0.9744 - 0.9618| = 0.0126$ \\
      Full model     & $|0.7592 - 0.7402| = 0.0190$ \\
      \midrule
                      & \textbf{Global Accuracy Diff.} \\
    \midrule
      Random masking & $|66.124 - 65.74| = 0.384$ \\
      Full model     & $|74.322 - 75.81| = 1.488$ \\
    \bottomrule
    \end{tabular}
    \label{tab:modelhomo_largest_highdatahetero_global_generalization_com_random_full}
\end{table}
\begin{table}[t]
    \centering
    \caption{Generalization of random masking and full model training under low data heterogeneity.}
    \setlength{\tabcolsep}{3pt}
     \begin{tabular}{l l}
    \toprule
                    & \textbf{Global Loss Diff.} \\
    \midrule
      Random masking & $|0.501 - 0.5164| = 0.0154$ \\
      Full model     & $|0.4379 - 0.471| = 0.0330$ \\
      \midrule
                      & \textbf{Global Accuracy Diff.} \\
    \midrule
      Random masking & $|83.014 - 82.92| = 0.094$ \\
      Full model     & $|85.46 - 84.62| = 0.840$ \\
    \bottomrule
    \end{tabular}
    
    \label{tab:modelhomo_largest_lowdatahetero_global_generalization_com_random_full}
\end{table}

\subsection{Generalization of Masked Training}
We evaluate the generalization of random masking and full model training (FedAvg) under both high and low data heterogeneity. To quantify generalization, we measure the gap between the global model's training and testing performance, considering both loss and accuracy. Specifically, we compute the difference between the global model's training loss (evaluated on local training data) and its test loss, as well as the difference between training accuracies (evaluated on local training data) and test accuracies for the global model. A smaller difference indicates better generalization and less overfitting to the local training distribution. Tab.~\ref{tab:modelhomo_largest_highdatahetero_global_generalization_com_random_full} and Tab.~\ref{tab:modelhomo_largest_lowdatahetero_global_generalization_com_random_full} show that random masking achieves smaller global loss and accuracy differences compared to full model training, under both cases. It suggests that random masking not only mitigates the adverse impact of non-i.i.d. data distributions but also enhances the stability and robustness of the global model compared to FedAvg training.

\section{Conclusions \& Directions for Future Work}
 This paper provides a comprehensive analysis of sub-model training in federated learning, addressing a significant gap in the rigorous convergence analysis of this approach. We have established convergence bounds for both randomly selected sub-model training and the partitioned model variant, demonstrating the impact of masking probability on residual error. Additionally, our stability analysis reveals that sub-model training can enhance generalization by stabilizing the training process. The empirical results corroborate our theoretical findings, highlighting the effectiveness of sub-model training in improving on-device local training for large learning models.

This work lays the foundation for further exploration in optimization and generalization of sub-model training. Nonetheless, several aspects warrant deeper examination in future studies to enhance the robustness and applicability of our findings. First, our theoretical analysis assumes strong convexity, smoothness, and bounded gradients, which are standard in the analysis of convergence for optimization algorithms. We note that these assumptions do not fully capture the behavior of deep non-convex neural networks. Proving convergence for such models typically requires fundamentally different techniques that leverage their specific structure (e.g., over-parameterization, layer-wise training dynamics, or implicit regularization). Developing such analyses is a valuable direction for future research and is currently beyond the scope of this work. Moreover, while privacy is not the main focus of our current work, we believe that sharing masked sub-models inherently offers some degree of privacy protection compared to full-model sharing, as in FedAvg. Since only partial model updates are communicated, the exposure of client-specific information is reduced. That said, we acknowledge that the masking mechanism could also introduce new security concerns. For example, clients with larger model capacity or more frequent participation might disproportionately influence the global model. This imbalance could potentially be exploited for targeted model manipulation or data inference attacks. Designing masking schemes that are both privacy-preserving and robust to such adversarial behavior is a compelling direction for future research.

\subsection*{Acknowledgment}
This work was partially supported by NSF CAREER Award \#2239374.

\bibliography{example_paper}

@article{yuan2019distributed,
  title={Distributed learning of deep neural networks using independent subnet training},
  author={Yuan, Binhang and Wolfe, Cameron R and Dun, Chen and Tang, Yuxin and Kyrillidis, Anastasios and Jermaine, Christopher M},
  journal={arXiv preprint arXiv:1910.02120},
  year={2019}
}

@article{jiang2022model,
  title={Model pruning enables efficient federated learning on edge devices},
  author={Jiang, Yuang and Wang, Shiqiang and Valls, Victor and Ko, Bong Jun and Lee, Wei-Han and Leung, Kin K and Tassiulas, Leandros},
  journal={IEEE Transactions on Neural Networks and Learning Systems},
  year={2022},
  publisher={IEEE}
}

@article{wu2024fiarse,
  title={FIARSE: Model-Heterogeneous Federated Learning via Importance-Aware Submodel Extraction},
  author={Wu, Feijie and Wang, Xingchen and Wang, Yaqing and Liu, Tianci and Su, Lu and Gao, Jing},
  journal={arXiv preprint arXiv:2407.19389},
  year={2024}
}

@InProceedings{pmlr-v51-schneider16,
  title = 	 {Probability Inequalities for Kernel Embeddings in Sampling without Replacement},
  author = 	 {Schneider, Markus},
  booktitle = 	 {Proceedings of the 19th International Conference on Artificial Intelligence and Statistics},
  pages = 	 {66--74},
  year = 	 {2016},
  editor = 	 {Gretton, Arthur and Robert, Christian C.},
  volume = 	 {51},
  series = 	 {Proceedings of Machine Learning Research},
  address = 	 {Cadiz, Spain},
  month = 	 {09--11 May},
  publisher =    {PMLR},
  url = 	 {https://proceedings.mlr.press/v51/schneider16.html}
}

@article{ghadimi2016mini,
  title={Mini-batch stochastic approximation methods for nonconvex stochastic composite optimization},
  author={Ghadimi, Saeed and Lan, Guanghui and Zhang, Hongchao},
  journal={Mathematical Programming},
  volume={155},
  number={1},
  pages={267--305},
  year={2016},
  publisher={Springer}
}

@article{deng2024distributed,
  title={Distributed personalized empirical risk minimization},
  author={Deng, Yuyang and Kamani, Mohammad Mahdi and Mahdavinia, Pouria and Mahdavi, Mehrdad},
  journal={Advances in Neural Information Processing Systems},
  volume={36},
  year={2024}
}

@inproceedings{fu2023effectiveness,
  title={On the effectiveness of parameter-efficient fine-tuning},
  author={Fu, Zihao and Yang, Haoran and So, Anthony Man-Cho and Lam, Wai and Bing, Lidong and Collier, Nigel},
  booktitle={Proceedings of the AAAI conference on artificial intelligence},
  volume={37},
  number={11},
  pages={12799--12807},
  year={2023}
}

@inproceedings{sun2024understanding,
  title={Understanding generalization of federated learning via stability: Heterogeneity matters},
  author={Sun, Zhenyu and Niu, Xiaochun and Wei, Ermin},
  booktitle={International Conference on Artificial Intelligence and Statistics},
  pages={676--684},
  year={2024},
  organization={PMLR}
}

@INPROCEEDINGS{10622512,
  author={Fang, Wenzhi and Han, Dong-Jun and Brinton, Christopher G.},
  booktitle={ICC 2024 - IEEE International Conference on Communications}, 
  title={Submodel Partitioning in Hierarchical Federated Learning: Algorithm Design and Convergence Analysis}, 
  year={2024},
  volume={},
  number={},
  pages={268-273},
  doi={10.1109/ICC51166.2024.10622512}}

@article{hyeon2021fedpara,
  title={Fedpara: Low-rank hadamard product for communication-efficient federated learning},
  author={Hyeon-Woo, Nam and Ye-Bin, Moon and Oh, Tae-Hyun},
  journal={arXiv preprint arXiv:2108.06098},
  year={2021}
}

@article{qiao2021communication,
  title={Communication-efficient federated learning with dual-side low-rank compression},
  author={Qiao, Zhefeng and Yu, Xianghao and Zhang, Jun and Letaief, Khaled B},
  journal={arXiv preprint arXiv:2104.12416},
  year={2021}
}

@inproceedings{mohtashami2022masked,
  title={Masked training of neural networks with partial gradients},
  author={Mohtashami, Amirkeivan and Jaggi, Martin and Stich, Sebastian},
  booktitle={International Conference on Artificial Intelligence and Statistics},
  pages={5876--5890},
  year={2022},
  organization={PMLR}
}

@article{haddadpour2019convergence,
  title={On the convergence of local descent methods in federated learning},
  author={Haddadpour, Farzin and Mahdavi, Mehrdad},
  journal={arXiv preprint arXiv:1910.14425},
  year={2019}
}

@article{demidovich2023mast,
  title={MAST: Model-Agnostic Sparsified Training},
  author={Demidovich, Yury and Malinovsky, Grigory and Shulgin, Egor and Richt{\'a}rik, Peter},
  journal={arXiv preprint arXiv:2311.16086},
  year={2023}
}

@inproceedings{cho2023convergence,
  title={On the convergence of federated averaging with cyclic client participation},
  author={Cho, Yae Jee and Sharma, Pranay and Joshi, Gauri and Xu, Zheng and Kale, Satyen and Zhang, Tong},
  booktitle={International Conference on Machine Learning},
  pages={5677--5721},
  year={2023},
  organization={PMLR}
}

@article{diao2020heterofl,
  title={Heterofl: Computation and communication efficient federated learning for heterogeneous clients},
  author={Diao, Enmao and Ding, Jie and Tarokh, Vahid},
  journal={arXiv preprint arXiv:2010.01264},
  year={2020}
}

@article{shulgin2023towards,
  title={Towards a better theoretical understanding of independent subnetwork training},
  author={Shulgin, Egor and Richt{\'a}rik, Peter},
  journal={arXiv preprint arXiv:2306.16484},
  year={2023}
}

@article{zhou2024every,
  title={Every parameter matters: Ensuring the convergence of federated learning with dynamic heterogeneous models reduction},
  author={Zhou, Hanhan and Lan, Tian and Venkataramani, Guru Prasadh and Ding, Wenbo},
  journal={Advances in Neural Information Processing Systems},
  volume={36},
  year={2024}
}

@inproceedings{zhu2021data,
  title={Data-free knowledge distillation for heterogeneous federated learning},
  author={Zhu, Zhuangdi and Hong, Junyuan and Zhou, Jiayu},
  booktitle={International conference on machine learning},
  pages={12878--12889},
  year={2021},
  organization={PMLR}
}

@article{lin2020ensemble,
  title={Ensemble distillation for robust model fusion in federated learning},
  author={Lin, Tao and Kong, Lingjing and Stich, Sebastian U and Jaggi, Martin},
  journal={Advances in Neural Information Processing Systems},
  volume={33},
  pages={2351--2363},
  year={2020}
}

@article{sun2020federated,
  title={Federated model distillation with noise-free differential privacy},
  author={Sun, Lichao and Lyu, Lingjuan},
  journal={arXiv preprint arXiv:2009.05537},
  year={2020}
}

@article{guha2019one,
  title={One-shot federated learning},
  author={Guha, Neel and Talwalkar, Ameet and Smith, Virginia},
  journal={arXiv preprint arXiv:1902.11175},
  year={2019}
}

@article{chen2020fedbe,
  title={Fedbe: Making bayesian model ensemble applicable to federated learning},
  author={Chen, Hong-You and Chao, Wei-Lun},
  journal={arXiv preprint arXiv:2009.01974},
  year={2020}
}

@article{li2019fedmd,
  title={Fedmd: Heterogenous federated learning via model distillation},
  author={Li, Daliang and Wang, Junpu},
  journal={arXiv preprint arXiv:1910.03581},
  year={2019}
}

@article{sattler2021fedaux,
  title={Fedaux: Leveraging unlabeled auxiliary data in federated learning},
  author={Sattler, Felix and Korjakow, Tim and Rischke, Roman and Samek, Wojciech},
  journal={IEEE Transactions on Neural Networks and Learning Systems},
  volume={34},
  number={9},
  pages={5531--5543},
  year={2021},
  publisher={IEEE}
}

@article{yao2021fedhm,
  title={Fedhm: Efficient federated learning for heterogeneous models via low-rank factorization},
  author={Yao, Dezhong and Pan, Wanning and O'Neill, Michael J and Dai, Yutong and Wan, Yao and Jin, Hai and Sun, Lichao},
  journal={arXiv preprint arXiv:2111.14655},
  year={2021}
}

@article{alam2022fedrolex,
  title={FedRolex: Model-Heterogeneous Federated Learning with Rolling Sub-Model Extraction},
  author={Alam, Samiul and Liu, Luyang and Yan, Ming and Zhang, Mi},
  journal={arXiv preprint arXiv:2212.01548},
  year={2022}
}

@article{cho2024heterogeneous,
  title={Heterogeneous Low-Rank Approximation for Federated Fine-tuning of On-Device Foundation Models},
  author={Cho, Yae Jee and Liu, Luyang and Xu, Zheng and Fahrezi, Aldi and Joshi, Gauri},
  journal={arXiv preprint arXiv:2401.06432},
  year={2024}
}

@article{woodworth2020minibatch,
  title={Minibatch vs local sgd for heterogeneous distributed learning},
  author={Woodworth, Blake E and Patel, Kumar Kshitij and Srebro, Nati},
  journal={Advances in Neural Information Processing Systems},
  volume={33},
  pages={6281--6292},
  year={2020}
}

@article{stich2018local,
  title={Local SGD converges fast and communicates little},
  author={Stich, Sebastian U},
  journal={arXiv preprint arXiv:1805.09767},
  year={2018}
}

@inproceedings{haddadpour2019local,
  title={Local SGD with periodic averaging: Tighter analysis and adaptive synchronization},
  author={Haddadpour, Farzin and Kamani, Mohammad Mahdi and Mahdavi, Mehrdad and Cadambe, Viveck},
  booktitle={Advances in Neural Information Processing Systems},
  pages={11080--11092},
  year={2019}
}

@article{li2019convergence,
  title={On the Convergence of FedAvg on Non-IID Data},
  author={Li, Xiang and Huang, Kaixuan and Yang, Wenhao and Wang, Shusen and Zhang, Zhihua},
  journal={arXiv preprint arXiv:1907.02189},
  year={2019}
}

@inproceedings{mcmahan2017communication,
  title={Communication-Efficient Learning of Deep Networks from Decentralized Data},
  author={McMahan, Brendan and Moore, Eider and Ramage, Daniel and Hampson, Seth and y Arcas, Blaise Aguera},
  booktitle={Artificial Intelligence and Statistics},
  pages={1273--1282},
  year={2017}
}

@InProceedings{khaled2019tighter,
    title={Tighter Theory for Local SGD on Identical and Heterogeneous Data},
    author={Ahmed Khaled and Konstantin Mishchenko and Peter Richtárik},
    year={2020},
    booktitle={AISTAT}
}

@article{yuan2020federated,
  title={Federated Accelerated Stochastic Gradient Descent},
  author={Yuan, Honglin and Ma, Tengyu},
  journal={arXiv preprint arXiv:2006.08950},
  year={2020}
}

@article{woodworth2020local,
  title={Is Local SGD Better than Minibatch SGD?},
  author={Woodworth, Blake and Patel, Kumar Kshitij and Stich, Sebastian U and Dai, Zhen and Bullins, Brian and McMahan, H Brendan and Shamir, Ohad and Srebro, Nathan},
  journal={arXiv preprint arXiv:2002.07839},
  year={2020}
}
\bibliographystyle{unsrtnat}

\clearpage
\appendix
\section*{Appendix}

\paragraph{Organization} The appendix is organized as follows:
\begin{itemize}
    \item In Appendix~\ref{app:related} we discuss additional related works.   
    \item In Appendix~\ref{app: exp} we provide additional experimental results and setup details.  
    \item In Appendix~\ref{app:proof_random_mask_fedavg} we provide the proof of randomly masked FedAvg. 
    \item In Appendix~\ref{app:proof_fedrolex} we provide the proof of masking with rolling (FedRolex).  
    \item In Appendix~\ref{app:stability_masked_training} we provide the proof of the stability of the masked training method. 
\end{itemize}

\section{Additional Related Works}\label{app:related}
\noindent\textbf{Convergence analysis of FedAvg.}
  FedAvg (or Local SGD)~\citep{mcmahan2017communication} was used as a solution to reduce communication cost and protect user data privacy in distributed learning. FedAvg is firstly proposed by~\cite{mcmahan2017communication} to alleviate communication bottleneck in the distributed machine learning. \cite{stich2018local} was the first to prove that local SGD achieves $O\left(1/T\right)$ convergence rate with only $O(\sqrt{T})$ communication rounds on  IID data for smooth strongly-convex loss functions. \cite{haddadpour2019local} analyzed the convergence of local SGD on nonconvex (PL condition) function, and proposed an adaptive synchronization scheme. \cite{khaled2019tighter} gave the tighter bound of local SGD, which directly reduces the $O(\sqrt{T})$ communication rounds in \citep{stich2018local} to $O(N)$, under smooth strongly-convex setting. \cite{yuan2020federated} proposed the first accelerated local SGD, which further reduced the communication rounds to  $O(N^{1/3})$.~\cite{haddadpour2019convergence} gave the analysis of local GD and SGD on smooth nonconvex functions in non-IID setting. Li et al~\cite{li2019convergence}  analyzed the convergence of FedAvg under non-IID data for strongly convex functions. \cite{woodworth2020local,woodworth2020minibatch} investigated the difference between local SGD and mini-batch SGD, in both homogeneous and heterogeneous data settings.

\noindent\textbf{Distributed/federated Sub-model training.}~Distributed/federated sub-model training is proposed to solve clients' insufficient computation and memory issue in federated learning, especially in this large language model era. \cite{diao2020heterofl} proposed the first federated learning algorithm with heterogeneous client model capacity. \cite{alam2022fedrolex} propose FedRolex, which allows the clients to pick the suitable sub-model to optimize according to their capacity.

From the theoretical perspective, \cite{mohtashami2022masked} is the first to study the sub-model training where they considered the single machine setting, where $N=1$, and proved that SGD with masked model parameter will converge to first order stationary point of the masked objective. \cite{shulgin2023towards} studied single machine sub-model training algorithm and considered the special scenario where $f_i(\bw)$ is quadratic, and proved that the convergence rate will suffer from a residual error unless the objective and sub-model sampling schema have some benign property. \cite{demidovich2023mast} studied similar algorithm, and proved convergence to the optimal point of the masked objective for general strongly convex losses. For nonconvex regime, \cite{zhou2024every},~\cite{10622512} and~\cite{wu2024fiarse} studied convergence of the distributed sub-model training algorithm with local updates on general nonconvex loss function, but their bound depends on the norm of history iterates, i.e., $\norm{\bw_t}$. When the model's norm is large, the convergence bound becomes vacuous.

\noindent\textbf{Low-rank Federated Learning.}~Another line of works that reduce client computation/communication burden are low-rank federated learning~\citep{qiao2021communication,hyeon2021fedpara,yao2021fedhm,cho2024heterogeneous}. \cite{qiao2021communication} propose FedDLR, where the clients only send low-rank model to server at communication stage. Server then performs averaging, decomposes the averaged model into low-rank version again,  and sends them back to clients. They also propose an adaptive rank selection which boosts the performance. \cite{hyeon2021fedpara} proposed FedPara, where the clients directly optimize on low-rank models instead of full model, to meet the local computation and memory constraints.

\noindent\textbf{Other Model Heterogeneous Federated Learning.}~Besides sub-model and low-rank federated learning, there is also a body of works for model heterogeneous federated learning, via knowledge distillation~\citep{zhu2021data,lin2020ensemble,sun2020federated,guha2019one,chen2020fedbe,li2019fedmd,sattler2021fedaux}. \cite{li2019fedmd} proposed the first knowledge distillation based method for model heterogeneous federated learning, where they leveraged a public dataset to compute the 'consensus', and force each client's model to behave close to this consensus, to share knowledge among clients. \cite{zhu2021data} proposed the first data-free distillation method for federated learning, where they try to learn a data generator to generate synthetic clients' data for model distillation. 

\section{Experiments Details}\label{app: exp}
We provide the data settings in Tab.~\ref{tab:dataset_setting} and experimental results under low data heterogeneity in this section.
\begin{table}[h]
    \centering
    \caption{Dataset Description}
    \begin{tabular}{c c c c}
    \hline
        \textbf{Dataset} & \textbf{\# of Training Clients} &\textbf{\# of Training Examples} &  \textbf{\# of Testing Examples}\\
        CIFAR-10 &  100 & 50,000 & 10,000\\
        CiFAR-100 & 100 & 50,000 & 10,000\\
    \hline
    \end{tabular}
    \label{tab:dataset_setting}
\end{table}

The experiments are conducted on 2 NVIDIA 6000 GPUs. To compare ours with HeteroFL~\cite{diao2020heterofl}, which is referred to as static masking in our experiments. 
We conducted experiments using CIFAR-100 with high data heterogeneity and the ResNet18 model, and client model capacities are 1/4 and 1/8 of the full model size. Client selection is 10\% of clients are randomly selected from a pool of 100 clients per global epoch. The results are shown in Tab.~\ref{tab:heterofl_com}, indicating that ours outperforms HeteroFL.

\begin{table}[h]
    \centering
    \caption{Comparisons of global model performance with HeteroFL using CIFAR-100 with ResNet18.}
    \begin{tabular}{c c c}
    \toprule
        \textbf{Method} & \textbf{Global Accuracy} &\textbf{Global Loss}\\
        \midrule
        Static Masking (HeteroFL)&  34.37\% & 2.5026\\
        Roll Masking (Ours)& 35.30\%& 2.4681\\
    \bottomrule
    \end{tabular}
    \label{tab:heterofl_com}
\end{table}

\section{Proof of Convergence of Randomly Masked FedAvg}\label{app:proof_random_mask_fedavg}

In this section, we provide detailed proofs for the results and theorems on sub-model training with random masking omitted from the main body. We begin by outlining several general results that serve as helper for the main proofs, and then provide the proofs of Theorem~\ref{thm:randomly masked fedavg} (strongly convex setting) and Theorem~\ref{thm:random mask nonconvex} (nonconvex setting) in   Subsection~\ref{app:supp:proof:thm1} and Subsection~\ref{app:supp:proof:thm2}, respectively.

\subsection{Proof of Convex Setting}\label{app:supp:proof:thm1}

In this subsection, we are going to prove Theorem~\ref{thm:randomly masked fedavg}.  We begin with a high-level sketch of the proof to briefly illustrate our  strategy  before presenting the detailed argument. Consider an alternative objective induced by masking:
\begin{align}
    F_{\bp}(\bw) = \frac{1}{N}\sum\nolimits_{i=1}^N \E_{\mask_i \sim Ber(p_i)} [f_i(\mask_i\odot\bw)] \label{eq: obj random},
\end{align}
where $\bp = [p_1,\ldots,p_N]$
and define $\bw^*(\bp):= \arg\min_{\bw\in\cW}F_\bp(\bw)$ as the optimal model given $\bp$. Apparently, when $\bp = \mathbf{1}$, $F_{\bp}(\bw)$ becomes original objective $F(\bw)$.
Our proof relies on a key Lipschitz property of $\bw^*(\bp)$. That is, if each $f_i(\bw)$ is strongly convex and with bounded gradient, then
\begin{align*}
    \norm{\bw^*(\bp) -\bw^*(\mathbf{1})} \leq c\cdot \norm{\bp - \mathbf{1}},
\end{align*}
for some constant $c$ depending on $G$, $\mu$ and $\bp$. As a result, we can decompose the objective value into (1) convergence error of the sub-model training to ${\bw}^*(\bp)$ and (2) residual error due to masking:
  \begin{align*}
      F(  \hat{\bw}) - & F(  \bw^* ) \leq   \frac{5}{2}L\|\tilde{\bw}  - {\bw}^*(\bp)\|^2  + \pare{  \frac{ 5L}{2\bar\mu  }    + \frac{4 }{ L} }\frac{2G^2 + 2W^2L^2}{N } \sum_{i=1}^N d(1-p_i)  .
    \end{align*}
It remains to prove the convergence of Algorithm~\ref{algorithm: Masked FedAvg} to ${\bw}^*(\bp)$, which can be achieved by standard Local SGD analysis.

We now proceed to the formal proof, making each step of the argument precise. We start by showing the strong convexity of the  alternative objective induced in Eq.~\ref{eq: obj random}.

\begin{proposition}\label{prop:property of random masked obj}
    $F_\bp(\bw)$ is $\mu_{\bp} := \frac{1}{N}\sum_{i=1}^N p_i \mu$ strongly convex, and $L_\bp :=  \frac{1}{N}\sum_{i=1}^N p_i L $ smooth.
    \begin{proof} 
    The proof mainly follows~\cite[Lemmas 9 and 11]{demidovich2023mast}.
     We first examine the smoothness and convexity parameter of $\tilde F$. For smoothness:
   \begin{align*}
    F_\bp (\bw+\Delta \bw) &= \frac{1}{N}\sum_{i=1}^N \E_{\mask_i \sim Ber(p_i)} f_i(\mask_i\odot (\bw+\Delta\bw) )   \\
   & \leq \frac{1}{N}\sum_{i=1}^N \E_{\mask_i \sim Ber(p_i)} \pare{ f_i(\mask_i \odot \bw  ) + \inprod{\mask_i\odot \Delta \bw}{ \nabla f_i(\mask_i\odot\bw ) } + \frac{L}{2}\norm{\mask_i\odot\Delta\bw}^2 } \\ 
   & = F_\bp(\bw ) + \inprod{ \Delta \bw}{\frac{1}{N}\sum_{i=1}^N \E_{\mask_i \sim Ber(p_i)} \mask_i\odot\nabla f_i(\mask_i\odot\bw ) }    \\
   &\quad + \frac{L}{2} \frac{1}{N}\sum_{i=1}^N \E_{\mask_i \sim Ber(p_i)}\inprod{(\mM_i )^2 \Delta \bw}{\Delta \bw} \\ 
   & = F_\bp (\bw ) + \inprod{ \Delta \bw}{\nabla F_\bp(\bw ) }    +\frac{1}{2}\frac{1}{N}\sum_{i=1}^N p_i L \norm{\Delta \bw}^2,
\end{align*}
where $\mM_i$ is the diagonal matrix with $\mask_i$ is its diagonal entries. 

For convexity we have:

 \begin{align*}
    F_\bp (\bw+\Delta \bw) &= \frac{1}{N}\sum_{i=1}^N \E_{\mask_i \sim Ber(p_i)} f_i(\mask_i\odot (\bw+\Delta\bw) )   \\
   & \geq \frac{1}{N}\sum_{i=1}^N \E_{\mask_i \sim Ber(p_i)} \pare{ f_i(\mask_i\odot \bw  ) + \inprod{\mask_i\odot \Delta \bw}{ \nabla f_i(\mask_i\odot\bw ) } + \frac{\mu}{2}\norm{\mask_i\odot\Delta\bw}^2 } \\ 
   & = F_\bp (\bw ) + \inprod{ \Delta \bw}{\frac{1}{N}\sum_{i=1}^N \E_{\mask_i \sim Ber(p_i)} \mask_i\odot\nabla f_i(\mask_i\odot\bw ) }    \\
   &\quad + \frac{\mu}{2} \frac{1}{N}\sum_{i=1}^N \E_{\mask_i \sim Ber(p_i)} \inprod{(\mM_i)^2 \Delta \bw}{\Delta \bw} \\ 
   & = F_\bp (\bw ) + \inprod{ \Delta \bw}{\nabla F_\bp(\bw ) }    + \frac{\mu}{2}\frac{1}{N}\sum_{i=1}^N p_i \norm{\Delta \bw}^2.
\end{align*}

So $F_\bp$ is $\mu_{\bp} := \frac{1}{N}\sum_{i=1}^N p_i \mu$ strongly convex and $L_\bp :=  \frac{1}{N}\sum_{i=1}^N p_i L $ smooth.
\end{proof}

\end{proposition}

\begin{lemma}\label{lem:random lipschitz}
Given a $N$-dimensional vector $\bp\in[0,1]^N$, we define $\bw^*(\bp):= \arg\min_{\bw \in \cW} \cbr{\Phi(\bp,\bw) := \frac{1}{N}\sum_{i=1}^N \E_{\mask_i\sim Ber(p_i)} f_i(\mask_i \odot \bw)} $. Then the following statement holds:
    \begin{align*}
     \|\bw^*(\mathbf{1}) - \bw^*(\bp)\|\leq  \sqrt{\frac{2G^2 + 2W^2L^2}{\mu_\bp^2 N }} \sqrt{\sum_{i=1}^N d(1- p_i)}.
    \end{align*}
   \begin{proof}
   
         First, according to optimality conditions we have:
        \begin{align*}
            \langle \bw - \bw^*(\bp), \nabla_{2} \Phi(\bp, \bw^*(\bp)) \rangle \geq 0,\\
            \langle \bw - \bw^*(\mathbf{1}), \nabla_{2} \Phi(\mathbf{1}, \bw^*(\mathbf{1})) \rangle \geq 0
        \end{align*}
        Substituting $\bw$ with $\bw^*(\mathbf{1})$ and $\bw^*(\bp)$ in the above first and second inequalities respectively yields:
        \begin{align*}
             \langle\bw^*(\mathbf{1}) - \bw^*(\bp), \nabla_{2} \Phi(\bp, \bw^*(\bp)) \rangle \geq 0,\\
            \langle \bw^*(\bp)- \bw^*(\mathbf{1}), \nabla_{2} \Phi(\mathbf{1}, \bw^*(\mathbf{1})) \rangle \geq 0.
        \end{align*}
        Adding up the above two inequalities yields:
        \begin{align}
            \langle \bw^*(\mathbf{1}) - \bw^*(\bp), \nabla_{2} \Phi(\bp, \bw^*(\bp))-\nabla_{2} \Phi(\mathbf{1}, \bw^*(\mathbf{1}))   \rangle \geq 0,\label{eq:random lipschitz 1} 
        \end{align}
        Since $\Phi(\bp,\cdot)$ is $\mu_\bp$ strongly convex, as shown in Proposition~\ref{prop:property of random masked obj}, we have:
        \begin{align}
            \langle \bw^*(\mathbf{1}) - \bw^*(\bp), \nabla_2 \Phi(\bp, \bw^*(\mathbf{1})) - \nabla_2 \Phi(\bp, \bw^*(\bp)) \geq \mu_\bp \|\bw^*(\mathbf{1}) - \bw^*(\bp)\|^2. \label{eq:random lipschitz 2} 
        \end{align}
        Adding up (\ref{eq:random lipschitz 1}) and (\ref{eq:random lipschitz 2}) yields:
        \begin{align*}
            \langle \bw^*(\mathbf{1}) - \bw^*(\bp), \nabla_2 \Phi(\bp, \bw^*(\mathbf{1})) - \nabla_2 \Phi(\mathbf{1}, \bw^*(\mathbf{1})) \geq \mu_\bp \|\bw^*(\mathbf{1}) - \bw^*(\bp)\|^2.
        \end{align*}
        Now we examine the smoothness of $\nabla_2 \Phi(\bp,\bw)$ in terms of the first variable. 
        \begin{align}
            \norm{ \nabla_2 \Phi(\bp,\bw) - \nabla_2 \Phi(\mathbf{1},\bw)   }^2 &= \norm{ \frac{1}{N}\sum_{i=1}^N  \E_{\mask_i \sim Ber(p_i)} \sbr{\mask_i  \odot \nabla f_i(\mask_i \odot \bw)} -  \frac{1}{N}\sum_{i=1}^N   \mathbf{1} \odot \nabla f_i(\mathbf{1}\odot \bw)   }^2  \nonumber\\
           & \leq \frac{1}{N}\sum_{i=1}^N   \norm{  \E_{\mask_i \sim Ber(p_i)} \sbr{\mask_i  \odot \nabla f_i(\mask_i \odot \bw)} -   \mathbf{1} \odot \nabla f_i(\mathbf{1}\odot \bw) }^2 \nonumber\\
           & \leq \frac{1}{N}\sum_{i=1}^N  \pare{ 2G^2 \E_{\mask_i \sim Ber(p_i)}\norm{  \mask_i    -    \mathbf{1}   }^2 
 + 2W^2L^2 \E_{\mask_i \sim Ber(p_i)}\norm{\mask_i    -   \mathbf{1}}^2} \nonumber\\
 & = \frac{2G^2 + 2W^2L^2}{N } \sum_{i=1}^N d(1- p_i)  \label{eq:random lipschitz 1}
        \end{align}
        Finally, putting pieces together will conclude the proof:
        \begin{align*}
           \sqrt{\frac{2G^2 + 2W^2L^2}{N }}\| \bw^*(\bp) - \bw^*(\mathbf{1})\| \sqrt{\sum_{i=1}^N d(1- p_i)} &\geq \mu_\bp \|\bw^*(\bp) - \bw^*(\mathbf{1})\|^2\\
            \Longleftrightarrow   \sqrt{\frac{2G^2 + 2W^2L^2}{\mu_\bp^2 N }} \sqrt{\sum_{i=1}^N d(1- p_i)} &\geq  \|\bw^*(\bp) - \bw^*(\mathbf{1})\|. 
        \end{align*}
    \end{proof}
\end{lemma}

\begin{lemma} [Optimality Gap]\label{lem:opt gap}
Let $\Phi(\bp,\bw) := \frac{1}{N}\sum_{i=1}^N \E_{\mask_i\sim Ber(p_i)} f_i(\mask_i \odot \bw)$ . Let $\hat{\bw} = \cP_{\cW}(\tilde{\bw} - \frac{1}{L} \nabla_{\bw} \Phi(\bp,\tilde{\bw}))$. If we assume each $f_i$ is $L$-smooth, $\mu$-strongly convex and with gradient bounded by $G$, then the following statement holds true:
    \begin{align*}
    \E[ F(  \hat{\bw}) -  F(  \bw^* ) ]\leq   \frac{5}{2}L\E\|\tilde{\bw}  - {\bw}^*(\bp)\|^2+ \pare{  \frac{ 5L}{2\mu_{\bp}  }    + \frac{4 }{ L} }\frac{2G^2 + 2W^2L^2}{N } \sum_{i=1}^N d(1-p_i) ,
    \end{align*}
    where    $\bw^* = \arg\min_{\bw \in \cW} F(   {\bw} )$.

    \begin{proof}
Define $\hat\nabla_{\bw} \Phi(\bp, {\bw}):=\frac{1}{N}\sum_{i=1}^N \mask_i \odot \nabla f_i(\mask_i\odot\bw  )   $     According to property of projection, we have:
    \begin{align*}
        0 &\leq \inprod{\bw - \hat{\bw}}{L(\hat{\bw} - \tilde{\bw})+\hat\nabla_{\bw} \Phi(\bp,\tilde{\bw}) }\\
        &= \underbrace{\inprod{\bw - \hat{\bw}}{L(\hat{\bw} - \tilde{\bw})+\nabla_{\bw} \Phi( \mathbf{1},\tilde{\bw}) } }_{T_1}+ \underbrace{\inprod{\bw - \hat{\bw}}{ \hat\nabla_{\bw} \Phi(\bp,\tilde{\bw}) -\nabla_{\bw} \Phi( \mathbf{1},\tilde{\bw})}}_{T_2}.
    \end{align*}

For $T_1$, we notice:
\begin{align*}
   \inprod{\bw - \hat{\bw}}{L(\hat{\bw} - \tilde{\bw})+\nabla_{\bw} \Phi( \mathbf{1} ,\tilde{\bw}) }  & =L \inprod{\bw - \tilde{\bw}} {\hat{\bw} - \tilde{\bw}} + L \inprod{ \tilde{\bw}-\hat{\bw}}{\hat{\bw} - \tilde{\bw}} +  \inprod{\bw - \hat{\bw}}{ \nabla_{\bw} \Phi( \mathbf{1},\tilde{\bw}) }\\
   & =L \inprod{\bw - \tilde{\bw}} {\hat{\bw} - \tilde{\bw}} - L \norm{ \tilde{\bw}-\hat{\bw}}^2 +   \inprod{\bw - \hat{\bw}_i}{ \nabla_{\bw} \Phi( \mathbf{1},\tilde{\bw}) } \\
   & \leq L ( \norm{\bw - \tilde{\bw}}^2 + \frac{1}{4}\norm{\hat{\bw} - \tilde{\bw}}^2 )- L \norm{ \tilde{\bw}-\hat{\bw}}^2 +  \underbrace{\inprod{\bw - \hat{\bw}}{ \nabla_{\bw} \Phi(\mathbf{1},\tilde{\bw}) }}_{\spadesuit } 
\end{align*}
where at last step we used Young's inequality.
To bound $\spadesuit$,  we apply the $L$ smoothness and $\mu$ strongly convexity of $\Phi(\mathbf{1},\cdot)$:
\begin{align*}
    \inprod{\bw - \hat{\bw} }{ \nabla_{\bw} \Phi(\mathbf{1},\tilde{\bw} ) } &= \inprod{\bw - \tilde{\bw} }{ \nabla_{\bw} \Phi( \mathbf{1},\tilde{\bw}) } + \inprod{\tilde{\bw} - \hat{\bw}}{ \nabla_{\bw} \Phi(\mathbf{1},\tilde{\bw}) }\\
    &\leq \Phi( \mathbf{1}, \bw ) - \Phi( \mathbf{1},\tilde{\bw}) - \frac{\mu}{2}\norm{\tilde{\bw} -  {\bw} }^2+ \Phi( \mathbf{1},\tilde{\bw}) - \Phi( \mathbf{1},\hat{\bw}) + \frac{L}{2}\norm{\tilde{\bw} - \hat{\bw}}^2\\
    &\leq \Phi( \mathbf{1}, \bw ) - \Phi( \mathbf{1},\hat{\bw})   - \frac{\mu}{2}\norm{\tilde{\bw} -  {\bw} }^2+  \frac{L}{2}\norm{\tilde{\bw} - \hat{\bw}}^2
\end{align*}
Putting above bound back yields:
\begin{align*}
  T_1 =   \inprod{\bw - \hat{\bw}}{L(\hat{\bw} - \tilde{\bw})+\nabla_{\bw} \Phi( \mathbf{1},\tilde{\bw}) } \leq \Phi( \mathbf{1}, \bw ) - \Phi( \mathbf{1},\hat{\bw})  + L\norm{\tilde{\bw} -  {\bw} }^2 - \frac{ L}{4}   \norm{\tilde{\bw} - \hat{\bw}}^2.
\end{align*}
Now we switch to bounding $T_2$. Applying Cauchy-Schwartz yields:
\begin{align*}
    \inprod{\bw - \hat{\bw} }{\hat \nabla_{\bw} \Phi(\bp,\tilde{\bw} ) -\nabla_{\bw} \Phi( \mathbf{1} ,\tilde{\bw} )} &\leq \frac{L}{4}\norm{\bw - \tilde{\bw} }^2 + \frac{L}{4}\norm{\tilde{\bw} - \hat{\bw} }^2 + \frac{4}{L} \norm{\hat\nabla_{\bw} \Phi(\bp,\tilde{\bw} ) -\nabla_{\bw} \Phi( \mathbf{1} ,\tilde{\bw} )}^2\\ 
\end{align*}
To bound $ \norm{\hat\nabla_{\bw} \Phi(\bp,\tilde{\bw} ) -\nabla_{\bw} \Phi( \mathbf{1} ,\tilde{\bw} )}^2$, we follow the same steps in (\ref{eq:random lipschitz 1}):
 \begin{align}
            \norm{\hat \nabla_2 \Phi(\bp,\tilde{\bw}) - \nabla_2 \Phi(\mathbf{1},\tilde{\bw})   }^2 &= \norm{ \frac{1}{N}\sum_{i=1}^N   \mask_i  \odot \nabla f_i(\mask_i \odot \tilde{\bw})  -  \frac{1}{N}\sum_{i=1}^N   \mathbf{1} \odot \nabla f_i(\mathbf{1}\odot \tilde{\bw})   }^2  \nonumber\\
           & \leq \frac{1}{N}\sum_{i=1}^N   \norm{    \mask_i  \odot \nabla f_i(\mask_i \odot \tilde{\bw})  -   \mathbf{1} \odot \nabla f_i(\mathbf{1}\odot \tilde{\bw}) }^2 \nonumber\\
           & \leq \frac{1}{N}\sum_{i=1}^N  \pare{ 2G^2  \norm{  \mask_i    -    \mathbf{1}   }^2 
 + 2W^2L^2  \norm{\mask_i    -   \mathbf{1}}^2} \nonumber\\
 & = \frac{2G^2 + 2W^2L^2}{N } \sum_{i=1}^N \norm{  \mask_i    -    \mathbf{1}   }^2  .
        \end{align}
 
Putting pieces together yields:
\begin{align*}
    0 &\leq \Phi( \mathbf{1} , \bw ) - \Phi( \mathbf{1} ,\hat{\bw} )  + \frac{5L}{4}\norm{\tilde{\bw}  -  {\bw} }^2     + \frac{4 }{  L} \frac{2G^2 + 2W^2L^2}{N } \sum_{i=1}^N\norm{  \mask_i    -    \mathbf{1}   }^2 .
\end{align*}
Re-arranging terms and setting $\bw = \bw^*(\mathbf{1}) = \arg\min_{\bw\in\cW} \Phi(\mathbf{1},\bw)$ yields:
\begin{align*}
     \Phi( \mathbf{1},\hat{\bw}) -  \Phi( \mathbf{1}, \bw^*(\mathbf{1}) ) \leq   \frac{5L}{4}\norm{\tilde{\bw} -  {\bw}^* }^2     + \frac{4 }{ L}\frac{2G^2 + 2W^2L^2}{N } \sum_{i=1}^N \norm{  \mask_i    -    \mathbf{1}   }^2 .
\end{align*}
Taking expectation over randomness of $\mask_i$ yields
\begin{align*}
   \E [\Phi( \mathbf{1},\hat{\bw}) -  \Phi( \mathbf{1}, \bw^*(\mathbf{1}) )] \leq  \frac{5L}{4}\E\norm{\tilde{\bw} -  {\bw}^* }^2     + \frac{4 }{ L}\frac{2G^2 + 2W^2L^2}{N } \sum_{i=1}^N d(1- p_i).
\end{align*}
At last, due to the Lipschitzness  property of of $\bw^*(\cdot)$  as shown in Lemma~\ref{lem:random lipschitz}, it follows that:
        \begin{align*}
           \frac{5L}{4}\E\|\tilde{\bw}  - {\bw}^*( \mathbf{1}) \|^2 &\leq \frac{5L}{2}\E\|\tilde{\bw} - {\bw}^*(\bp)\|^2+\frac{5L}{2}\E\|{\bw}^*(\bp) - {\bw}^*( \mathbf{1} )\|^2 \\
            &\leq \frac{5L}{2}\E\|\tilde{\bw}  - {\bw}^*(\bp)\|^2+\frac{5L}{2}  \frac{2G^2 + 2W^2L^2}{\mu^2_\bp N } \sum_{i=1}^N d(1- p_i)  ,
        \end{align*}
  as desired.

    \end{proof}
    
    \end{lemma}

Next we are going to present technical lemmas for proving convergence of Algorithm~\ref{algorithm: Masked FedAvg} to $\bw^*(\bp)$.
For notational convenience, we drop the subscript and use $\bw^*$ to denote $\bw^*(\bp)$, and  we define $\tilde f_i(\bw) :=\E_{\mask_i \sim Ber(p_i)}[f_i(\mask_i\odot \bw)]$.
We define virtual local iterates $\tilde\bw_{r,k}^i$ be such that,for $j \in \mathrm{supp}(\mask_i^r) $, $\tilde\bw_{r,k}^i [j] = \bw_{r,k}^i[j] $;  for $j \notin \mathrm{supp}(\mask_i^r) $, we set $\tilde\bw_{r,k}^i [j] = \bw_{r}[j] $. An important property is that, $ \mask_i^r \odot \nabla \tilde f_i(\mask_i^r \odot \tilde\bw_{r,k}^i;\xi) =  \mask_i^r \odot \nabla \tilde f_i( \mask_i^r \odot \bw_{r,k}^i;\xi) $. Hence, the local updates can be equivalently viewed as conduct on the gradients queried on $\tilde\bw_{r,k}^i$, i.e.,
\begin{align*}
    \bw_{r+1} = \cP_{\cW}\pare{ \bw_r - \eta \sum_{k=0}^{K-1} \frac{1}{N}\sum_{i=1}^N  \tilde \bg_{r,k}^i}
\end{align*}
where $\tilde \bg_{r,k}^i =\mask_i\odot \nabla f_i(\mask_i\odot \tilde \bw_{r,k}^i;\xi_{r,k}^i)$.

\begin{lemma}\label{lem:one iteration}
For Algorithm~\ref{algorithm: Masked FedAvg}, under the condition of Theorem~\ref{thm:randomly masked fedavg}, the following statement holds true:
 \begin{align*}
       \E\norm{ \bw_{r+1} - \bw^* }^2   
          & \leq  (1-\tilde\mu\eta) \E \norm{\bw_r  - \bw^*}^2 - \frac{1}{2}\eta  K  \frac{1}{N}\sum_{i=1}^N    \pare{   \tilde f_i(  \bw_r ) -   \tilde f_i(  \bw^*)  }   \\
         &\quad +  2\eta \tilde L  \frac{1}{N}\sum_{i=1}^N   \sum_{k=0}^{K-1} \E\norm{ \tilde  \bw_{r,k}^i -  \bw_r }^2 + \frac{K \eta^2 \delta^2 }{N}.
    \end{align*}
    \begin{proof}
        According to updating rule we have:
    \begin{align*}
        \E\norm{ \bw_{r+1} - \bw^* }^2 &= \E\norm{ \bw_r - \eta \sum_{k=0}^{K-1} \frac{1}{N}\sum_{i=1}^N  \tilde \bg_{r,k}^i - \bw^* }^2\\
       & =\E \norm{\bw_t  - \bw^*}^2 - \E\inprod{  \eta\sum_{k=0}^{K-1} \frac{1}{N}\sum_{i=1}^N   \tilde\bg_{r,k}^i}{\bw_t  - \bw^*} + \E\norm{\eta \sum_{k=0}^{K-1} \frac{1}{N}\sum_{i=1}^N   \tilde\bg_{r,k}^i}^2\\
       & =\E \norm{\bw_r  - \bw^*}^2  + \E\norm{\eta \sum_{k=0}^{K-1} \frac{1}{N}\sum_{i=1}^N   \tilde \bg_{r,k}^i}^2\\
       & \quad - \E\inprod{  \eta \sum_{k=0}^{K-1} \frac{1}{N}\sum_{i=1}^N \mask^r_i\odot \nabla   f_i(\mask^r_i\odot \tilde\bw_t^i)}{ \bw_r  - \bw^*}\\
       & = \E \norm{\bw_r  - \bw^*}^2  + \E\norm{\eta \sum_{k=0}^{K-1} \frac{1}{N}\sum_{i=1}^N   \tilde \bg_{r,k}^i}^2\\
       & \quad - \E\inprod{  \eta \sum_{k=0}^{K-1} \frac{1}{N}\sum_{i=1}^N  \mask^r_i\odot \nabla f_i(\mask^r_i\odot \tilde \bw_{r,k}^i)}{\bw_r  - \bw^*}\\
         & =\E \norm{\bw_r  - \bw^*}^2 -   \inprod{  \eta \sum_{k=0}^{K-1} \frac{1}{N}\sum_{i=1}^N   \nabla  \tilde f_i(  \tilde \bw_{r,k}^i)}{ \bw_r  - \tilde\bw_{r,k}^i }   \\
         & \quad  -    \inprod{ \eta \sum_{k=0}^{K-1} \frac{1}{N}\sum_{i=1}^N   \nabla \tilde f_i(  \tilde\bw_{r,k}^i)}{ \tilde \bw_{r,t}^i  - \bw^* )} + \E\norm{\eta \sum_{k=0}^{K-1} \frac{1}{N}\sum_{i=1}^N   \tilde \bg_{r,k}^i}^2.
    \end{align*}
    where at last step we use the fact $ \E_{\mask^r_i} [\mask^r_i\odot \nabla f_i(\mask^r_i\odot \tilde \bw_{r,k}^i)] =   \nabla \tilde f_i( \tilde \bw_{r,k}^i )$. 
    Since $\tilde f_i$ is $ L_i=p_i L$ smooth and $ \mu_{i}$ strongly convex, and by definition $\tilde L = \max_{i\in[N]} L_i$, $\tilde \mu = \min_{i\in[N]} \mu_i$,  we have
     \begin{align*}
        \E\norm{ \bw_{r+1} - \bw^* }^2  
         & \leq  (1-\tilde\mu\eta K) \E \norm{\bw_r  - \bw^*}^2 -  \eta K \frac{1}{N}\sum_{i=1}^N    \pare{   \tilde f_i(  \bw_r ) -   \tilde f_i(  \bw^*)  } + \frac{K \eta^2 \delta^2 }{N}  \\
         &\quad + \eta \tilde L \frac{1}{N}\sum_{i=1}^N   \sum_{k=0}^{K-1} \E\norm{   \tilde \bw_{r,k}^i -  \bw_r }^2 + \eta^2\E\norm{  \sum_{k=0}^{K-1}  \frac{1}{N}\sum_{i=1}^N  \mask_i^r\odot  \nabla   f_i( \mask_i^r\odot\tilde \bw_{r,k}^i )  }^2 \\
         & \leq  (1-\tilde\mu\eta K) \E \norm{\bw_r  - \bw^*}^2 -  \eta K \frac{1}{N}\sum_{i=1}^N    \pare{   \tilde f_i(  \bw_r ) -   \tilde f_i(  \bw^*)  }  \\
         &\quad + \eta  \tilde L \frac{1}{N}\sum_{i=1}^N   \sum_{k=0}^{K-1} \E\norm{   \tilde \bw_{r,k}^i -  \bw_r }^2\\
         &\quad +2 \eta^2 K\sum_{k=0}^{K-1}  \frac{1}{N}\sum_{i=1}^N\E\norm{   \mask_i^r\odot  \nabla   f_i(\mask_i^r\odot \tilde \bw_{r,k}^i ) - \mask_i^r\odot\nabla\tilde f_i( \mask_i^r\odot \bw_{r}  )  }^2\\
         &\quad +2 \eta^2 K^2 \E\norm{   \frac{1}{N}\sum_{i=1}^N  \mask_i^r\odot   \nabla  f_i(  \mask_i^r\odot \bw_{r}  )  }^2  + \frac{K \eta^2 \delta^2 }{N}\\
          & \leq  (1-\tilde\mu\eta K)  \E \norm{\bw_r  - \bw^*}^2 -  (\eta  K - 4\eta^2 K^2   L)\frac{1}{N}\sum_{i=1}^N    \pare{   \tilde f_i(  \bw_r ) -   \tilde f_i(  \bw^*)  }   \\
         &\quad + (\eta \tilde L+2 \eta^2 \tilde L^2 K) \frac{1}{N}\sum_{i=1}^N   \sum_{k=0}^{K-1} \E\norm{ \tilde  \bw_{r,k}^i -  \bw_r }^2 + \frac{K \eta^2 \delta^2 }{N},
    \end{align*}

    where the last step is due to~\cite[Lemma 4] {demidovich2023mast} that
    \begin{align*}
      \E\norm{   \frac{1}{N}\sum_{i=1}^N   \mask_i^r\odot  \nabla  f_i( \mask_i^r\odot \bw_{r}  )  }^2
        & \leq 2   L  \frac{1}{N}\sum_{i=1}^N  \E\pare{  \tilde f_i( \bw_r  )  -     \tilde f_i( \bw^*  )  }.
    \end{align*}
    Since $\eta \leq \frac{1}{4L}$, we can conclude the proof.
        \end{proof}

\end{lemma}

\begin{lemma}\label{lem:model deviation}
For Algorithm~\ref{algorithm: Masked FedAvg}, under the condition of Theorem~\ref{thm:randomly masked fedavg}, the following statement holds true:
\begin{align*}
   \frac{1}{N}\sum_{i=1}^N  \E\norm{ \tilde \bw_{r,k}^i -  \bw_r  }^2  &  \leq  5K\pare{ 8 \eta^2 K  L\E\pare{      F_{\bp}(  \bw_{r } ) -  F_{\bp}(  \bw^* )  }   + 4\eta^2 K  \sigma_*^2  +  \eta^2 K  {\delta^2} } 
\end{align*}

\begin{proof}
    According to local updating rule we have:
    \begin{align*}
         \E\norm{ \tilde \bw_{r,k}^i -  \bw_r  }^2  
       &=  (1+\frac{1}{K-1})\E\norm{ \tilde \bw_{r,k-1}^i -  \bw_r  }^2 + K \E\norm{  \eta    \tilde \bg_{r,k-1}^i  }^2\\
      &  \leq (1+\frac{1}{K-1})\E\norm{ \tilde \bw_{r,k-1}^i -  \bw_r  }^2 + K \E\norm{  \eta  \mask_i^r\odot  \nabla   f_i(\mask_i^r \odot\tilde \bw_{r,k-1}^i)  }^2  +  \eta^2 K  {\delta^2} \\ 
       &  \leq (1+\frac{1}{K-1})\E\norm{ \tilde \bw_{r,k-1}^i -  \bw_r  }^2 + 2K \E\norm{  \eta \mask_i^r \odot   \nabla   f_i(\mask_i^r \odot\tilde \bw_{r } )  }^2 \\
       &\quad + 2\eta^2  \tilde L^2 K \norm{  \tilde \bw_{r,k-1}^i  -  \tilde \bw_{r} }^2 +  \eta^2 K  {\delta^2} \\ 
       &  \leq (1+\frac{2}{K-1})\E\norm{ \tilde \bw_{r,k-1}^i -  \bw_r  }^2 + 8 \eta^2 K  L \E\pare{      \tilde f_i( \bw_{r } ) -  \tilde f_i(  \bw^* )  } \\
       &\quad + 4\eta^2 K\E\norm{\mask_i^r \odot\nabla   f_i(\mask_i^r \odot\bw^*) }^2 +  \eta^2 K  {\delta^2} \\ 
       &  \leq  \sum_{j=1}^{k} (1+\frac{2}{K-1})^{k-j}\pare{ 8 \eta^2 K  L  \E\pare{      \tilde f_i(  \bw_{r } ) -  \tilde f_i(  \bw^* )  }   + 4\eta^2 K\E\norm{\mask_i^r \odot \nabla   f_i(\mask_i^r \odot\bw^*) }^2  +  \eta^2 K  {\delta^2} }\\ 
       &  \leq  5K\pare{ 8 \eta^2 K  L \E\pare{      \tilde f_i(  \bw_{r } ) -  \tilde f_i(  \bw^* )  }   + 4\eta^2 K\E\norm{\mask_i^r \odot \nabla  f_i(\mask_i^r \odot\bw^*) }^2  +  \eta^2 K  {\delta^2} }\\ 
    \end{align*}
    where the fourth step is due to $2\eta^2 \tilde L^2 K \leq \frac{1}{K-1}$ and~\cite[Lemma 4] {demidovich2023mast} that
    \begin{align*}
      \E\norm{    \mask_i^r\odot  \nabla  f_i( \mask_i^r\odot \bw_{r}  )  }^2 &\leq2\E\norm{    \mask_i^r\odot  \nabla  f_i( \mask_i^r\odot \bw_{r}  )  -  \mask_i^r\odot  \nabla  f_i( \mask_i^r\odot \bw^*  )  }^2 + 2\E\norm{    \mask_i^r\odot  \nabla  f_i( \mask_i^r\odot \bw^*  )  }^2 \\
      &\leq 4   L   \E\pare{  \tilde f_i( \bw_r  )  -     \tilde f_i( \bw^*  )  }+ 2\E\norm{    \mask_i^r\odot  \nabla  f_i( \mask_i^r\odot \bw^*  )  }^2.
    \end{align*}

    Summing $i=1$ to $N$ yields:
    \begin{align*}
         \frac{1}{N}\sum_{i=1}^N  \E\norm{ \tilde \bw_{r,k}^i -  \bw_r  }^2  &  \leq  5K\pare{ 8 \eta^2 K   L \E\pare{      F_{\bp}(  \bw_{r } ) -  F_{\bp}(  \bw^* )  }   + 4\eta^2 K  \sigma_*^2  +  \eta^2 K  {\delta^2} },
    \end{align*}
    which conclude the proof.
    \end{proof}

\end{lemma}

\subsubsection{Proof of Theorem~\ref{thm:randomly masked fedavg} } 

\begin{proof}
    Evoking Lemma~\ref{lem:one iteration} yields:
 \begin{align*}
       \E\norm{ \bw_{r+1} - \bw^* }^2   
          & \leq    (1-\tilde\mu\eta)\E \norm{\bw_r  - \bw^*}^2 - \frac{1}{2}\eta  K  \frac{1}{N}\sum_{i=1}^N    \pare{   \tilde f_i(  \bw_r ) -   \tilde f_i(  \bw^*)  }   \\
         &\quad + 2 \eta \tilde L \frac{1}{N}\sum_{i=1}^N   \sum_{k=0}^{K-1} \E\norm{ \tilde  \bw_{r,k}^i -  \bw_r }^2 + \frac{K \eta^2 \delta^2 }{N}.
    \end{align*}
    We plug in Lemma~\ref{lem:model deviation} and get 
        \begin{align*}
       \E\norm{ \bw_{r+1} - \bw^* }^2   
          & \leq  (1-\tilde
          \mu \eta) \E \norm{\bw_r  - \bw^*}^2 - \frac{1}{2}\eta  K       \pare{   F(  \bw_r ) -   F(  \bw^*)  }   \\
         &\quad + 2 \eta \tilde L \cdot 5K^2\pare{ 8 \eta^2 K L\E\pare{      F(  \bw_{r } ) -  F(  \bw^* )  }   + 4\eta^2 K  \sigma_*^2  +  \eta^2 K  {\delta^2} } + \frac{K \eta^2 \delta^2 }{N}\\
         & =  (1-\tilde
          \mu \eta)\E \norm{\bw_r  - \bw^*}^2 - \pare{\frac{1}{2}\eta  K -(\eta   L+2 \eta^2  L^2 K) 40 \eta^2 K^3    }  \pare{   F_{\bp}(  \bw_r ) -   F_{\bp}(  \bw^*)  }   \\
         &\quad + 2 \eta \tilde L\cdot 5K^2\pare{  4\eta^2 K  \sigma_*^2  +  \eta^2 K  {\delta^2} } + \frac{K \eta^2 \delta^2 }{N}.
    \end{align*} 
    Since we choose $\eta \leq \frac{1}{16  L K}$, we know $\pare{\frac{1}{2}\eta  K -(\eta   L+2 \eta^2  L^2 K) 40 \eta^2 K^3    } \geq 0$, so we can drop this term and get:
     \begin{align*}
       \E\norm{ \bw_{r+1} - \bw^* }^2   
          & \leq  (1-\tilde
          \mu \eta) \E \norm{\bw_r  - \bw^*}^2     + 2\eta \tilde L \cdot 5K^2\pare{  4\eta^2 K  \sigma_*^2  +  \eta^2 K  {\delta^2} } + \frac{K \eta^2 \delta^2 }{N}.
    \end{align*} 
    Unrolling the recursion yields:
    \begin{align}
       \E\norm{ \bw_{r+1} - \bw^* }^2   
          & \leq  (1-\tilde  \mu \eta K)^r \E \norm{\bw_0  - \bw^*}^2     + 2 \tilde \kappa \cdot 5K\pare{  4\eta^2 K  \sigma_*^2  +  \eta^2 K  {\delta^2} } + \frac{ \eta \delta^2 }{\tilde \mu N}. \label{eq:point convergence}
    \end{align} 
    Plugging in $\eta = \frac{\log (KR)^2}{\tilde \mu KR}$ will conclude the proof:
  \begin{align*}
       \E\norm{ \bw_{R} - \bw^* }^2   
          & \leq  O\pare{ \frac{\E \norm{\bw_0  - \bw^*}^2 }{K^2 R^2} }  +      \tilde O\pare{    \frac{ \tilde \kappa \frac{1}{N}\sum_{i=1}^N\norm{\nabla \tilde f_i(\bw^*) }^2+    \tilde \kappa  \delta^2 }{\tilde \mu^2   R^2} }   + \tilde O \pare{\frac{  \delta^2 }{\tilde \mu^2 NKR}}.
    \end{align*} 
 
\end{proof}

\subsection{Proof of Nonconvex Setting} \label{app:supp:proof:thm2}
In this subsection, we are going to prove Theorem~\ref{thm:random mask nonconvex}. Since constrained non-convex stochastic optimization suffers from residual noise error unless a large mini-batch is used~\citep{ghadimi2016mini}, here we assume an unconstrained setting, i.e., $\cW = \R^d$. We present the following technical lemma first.
\begin{lemma}\label{lem:model deviation random noncvx}
For Algorithm~\ref{algorithm: Masked FedAvg}, under the condition of Theorem~\ref{thm:randomly masked fedavg}, the following statement holds true:
\begin{align*}
   \frac{1}{N}\sum_{i=1}^N  \E\norm{ \tilde \bw_{r,k}^i -  \bw_r  }^2  &  \leq  5K\pare{ 4\eta^2 K \E\norm{     \nabla    F_{\bp}( \bw_{r } )  }^2 + 4\eta^2 K \zeta_{\bp}^2  +  \eta^2 K  {\delta^2} }.
\end{align*}
\end{lemma}
\begin{proof}
    According to local updating rule we have:
    \begin{align*}
        & \E\norm{ \tilde \bw_{r,k}^i -  \bw_r  }^2 \\ 
       &=  (1+\frac{1}{K-1})\E\norm{ \tilde \bw_{r,k-1}^i -  \bw_r  }^2 + K \E\norm{  \eta    \tilde \bg_{r,k-1}^i  }^2\\
      &  \leq (1+\frac{1}{K-1})\E\norm{ \tilde \bw_{r,k-1}^i -  \bw_r  }^2 + K \E\norm{  \eta  \mask_i^r\odot  \nabla   f_i(\mask_i^r \odot\tilde \bw_{r,k-1}^i)  }^2  +  \eta^2 K  {\delta^2} \\ 
       &  \leq (1+\frac{1}{K-1})\E\norm{ \tilde \bw_{r,k-1}^i -  \bw_r  }^2 + 2K \E\norm{  \eta   \mask_i^r \odot \nabla   f_i(\mask_i^r \odot  \bw_{r } )  }^2 + 2\eta^2  \tilde L^2 K \norm{  \tilde \bw_{r,k-1}^i  -   \bw_{r} }^2 +  \eta^2 K  {\delta^2} \\ 
       &  \leq (1+\frac{2}{K-1})\E\norm{ \tilde \bw_{r,k-1}^i -  \bw_r  }^2 +   4\eta^2 K \E\norm{     \nabla  F_{\bp}( \bw_{r } )  }^2\\
       &\quad + 4\eta^2 K\E\norm{ \mask_i^r \odot \nabla   f_i(\mask_i^r \odot  \bw_{r } ) -    \nabla  F_{\bp}(   \bw_{r } )  }^2 +  \eta^2 K  {\delta^2} \\ 
       &  \leq  \sum_{j=1}^{k} (1+\frac{2}{K-1})^{k-j}\pare{4\eta^2 K \E\norm{     \nabla  F_{\bp}(   \bw_{r } )  }^2 + 4\eta^2 K\E\norm{ \mask_i^r \odot \nabla   f_i(\mask_i^r \odot \bw_{r } ) -    \nabla  F_{\bp}(   \bw_{r } )  }^2 +  \eta^2 K  {\delta^2} }\\ 
       &  \leq  5K\pare{ 4\eta^2 K \E\norm{     \nabla  F_{\bp}(   \bw_{r } )  }^2 + 4\eta^2 K\E\norm{ \mask_i^r \odot \nabla   f_i(\mask_i^r \odot  \bw_{r } ) -    \nabla  F_{\bp}(   \bw_{r } )  }^2 +  \eta^2 K  {\delta^2} }, 
    \end{align*}
    where the fourth step is due to $2\eta^2 \tilde L^2 K \leq \frac{1}{K-1}$.

    Summing for $i=1$ to $N$ yields:
    \begin{align*}
         \frac{1}{N}\sum_{i=1}^N  \E\norm{ \tilde \bw_{r,k}^i -  \bw_r  }^2  &  \leq  5K\pare{ 4\eta^2 K \E\norm{     \nabla  F_{\bp}(  \bw_{r } )  }^2 + 4\eta^2 K \zeta_{\bp}^2  +  \eta^2 K  {\delta^2} },
    \end{align*}
    which concludes the proof.
 
    \end{proof}
\subsubsection{Proof of Theorem~\ref{thm:random mask nonconvex}}
\begin{proof}
    
From $L_\bp$-smoothness of $F_{\bp}$ (Proposition~\ref{prop:property of random masked obj}), we have
\begin{align*}
\E [F_{\bp}(\bw_{r+1})] &\leq \E [F_{\bp}(\bw_r) ]+ \E\inprod{\nabla F_{\bp}(\bw_r)}{ \bw_{r+1} - \bw_r } + \frac{L_{\bp}}{2}\E\norm{\bw_{r+1} - \bw_r}^2\\
  &  \leq \E [F_{\bp}(\bw_r) ] - \E\inprod{ \nabla F_{\bp}(\bw_r)}{  \eta  \sum_{k=0}^{K-1} \frac{1}{N}\sum_{i=1}^N \mask^r_i\odot \nabla   f_i(\mask^r_i\odot \tilde\bw_t^i) }\\
  &\quad + \frac{L}{2}\E\norm{\eta \sum_{k=0}^{K-1} \frac{1}{N}\sum_{i=1}^N \mask^r_i\odot \nabla   f_i(\mask^r_i\odot \tilde\bw_t^i) }^2 + \frac{\eta^2 L_{\bp} K\delta^2}{2N}\\
    &= \E [F_{\bp}(\bw_r) ] - \E \eta K \inprod{ \nabla F_{\bp}(\bw_r)}{  \frac{1}{K} \sum_{k=0}^{K-1} \frac{1}{N}\sum_{i=1}^N  \nabla  \tilde f_i(  \tilde\bw_t^i) }\\
  &\quad + \frac{L}{2}\E\norm{\eta \sum_{k=0}^{K-1} \frac{1}{N}\sum_{i=1}^N \mask^r_i\odot \nabla   f_i(\mask^r_i\odot \tilde\bw_t^i) }^2 + \frac{\eta^2 L_{\bp} K\delta^2}{2N}\\
\end{align*}

Applying the identity $\inprod{\ba}{\bb} = \frac{1}{2}\norm{\ba}^2 + \frac{1}{2}\norm{\bb}^2 - \frac{1}{2}\norm{\ba-\bb}^2$ yields:
\begin{align*}
\E [F_{\bp}(\bw_{r+1})] 
  &  \leq \E [F_{\bp}(\bw_r) ] -  \frac{1}{2}\eta K\E\norm{\nabla F_{\bp}(\bw_r)}^2 - \frac{1}{2}\eta  K \E\norm{\frac{1}{K}  \sum_{k=0}^{K-1} \frac{1}{N}\sum_{i=1}^N    \nabla  \tilde f_i(  \tilde\bw_t^i) }^2\\
  &\quad + \frac{1}{2}\eta K \norm{\nabla F_{\bp} (\bw_r) -    \frac{1}{K}\sum_{k=0}^{K-1} \frac{1}{N}\sum_{i=1}^N   \nabla \tilde  f_i( \tilde\bw_t^i)  }^2\\
  &\quad + \frac{L_{\bp}}{2}\eta^2 K^2 \E\norm{\frac{1}{K} \sum_{k=0}^{K-1} \frac{1}{N}\sum_{i=1}^N \mask^r_i\odot \nabla   f_i(\mask^r_i\odot \tilde\bw_t^i) }^2 + \frac{\eta^2 L_{\bp} K\delta^2}{2N}\\
  &  = \E [F_{\bp} (\bw_r) ] -  \frac{1}{2}\eta K\E\norm{\nabla F_{\bp}(\bw_r)}^2 \\
  & \quad - \pare{\frac{1}{2}\eta  K -  \frac{L_{\bp}}{2}\eta^2 K^2}\E\norm{\frac{1}{K}  \sum_{k=0}^{K-1} \frac{1}{N}\sum_{i=1}^N \mask^r_i\odot \nabla   f_i(\mask^r_i\odot \tilde\bw_t^i) }^2\\
  &\quad + \frac{1}{2}\eta K \norm{\nabla F_{\bp}(\bw_r) -    \frac{1}{K}\sum_{k=0}^{K-1} \frac{1}{N}\sum_{i=1}^N  \nabla \tilde  f_i(  \tilde\bw_t^i)  }^2+  \frac{\eta^2 L_{\bp} K V^2}{2N} +  \frac{\eta^2 L_{\bp} K\delta^2}{2N}\\
\end{align*}
where $V := \sup_{i\in[N]}\E_\mask\norm{ \nabla\tilde f_i(\bw) - \mask \odot\nabla f_i(\mask\odot\bw)}^2$. 
Since we choose $\eta \leq \frac{1}{KL}$, we know $\frac{1}{2}\eta  K -  \frac{L_{\bp}}{2}\eta^2 K^2 \leq 0$.
\begin{align*}
\E [F_{\bp}(\bw_{r+1})] 
  &  \leq   \E [F_{\bp}(\bw_r) ] -  \frac{1}{2}\eta K\E\norm{\nabla F_{\bp}(\bw_r)}^2  + \frac{1}{2}\eta K \norm{\nabla F_{\bp}(\bw_r) -    \frac{1}{K}\sum_{k=0}^{K-1} \frac{1}{N}\sum_{i=1}^N   \nabla  \tilde f_i(  \tilde\bw_t^i)  }^2\\
  &\quad +  \frac{\eta^2 L_{\bp} K(\delta^2+V^2)}{2N}\\
  &  \leq   \E [F_{\bp}(\bw_r) ] -  \frac{1}{2}\eta K\E\norm{\nabla F_{\bp}(\bw_r)}^2  + \frac{1}{2}\eta K \frac{1}{K}\sum_{k=0}^{K-1} \frac{1}{N}\sum_{i=1}^N \tilde L^2 \norm{  \bw_r -  \tilde\bw_t^i  }^2 \\
  &\quad +  \frac{\eta^2 L_{\bp} K(\delta^2+V^2)}{2N}.
\end{align*}
Plugging in Lemma~\ref{lem:model deviation random noncvx} above yields:
\begin{align*}
 \E [F_{\bp}(\bw_{r+1})]
  &  \leq   \E [F_{\bp}(\bw_r) ] -  \frac{1}{2}\eta K\E\norm{\nabla F_{\bp}(\bw_r)}^2  \\
  &\quad + \frac{1}{2}\eta K \tilde L^2       5K\pare{ 4\eta^2 K \E\norm{     \nabla  F_{\bp}( \tilde \bw_{r } )  }^2 + 4\eta^2 K \zeta_{\bp}^2  +  \eta^2 K  {\delta^2} } +  \frac{\eta^2 L_{\bp} K(V^2+\delta^2)}{2N}\\
    & =   \E [F_{\bp}(\bw_r) ] - \pare{ \frac{1}{2}\eta K - 10\eta^3 K^3 L^2       
 }\E\norm{\nabla F_{\bp}(\bw_r)}^2 \\
 &\quad +     20\eta^3 K^3 L_{\bp}^2 \zeta_{\bp}^2  +  5\eta^3 K^3 L_{\bp}^2 \delta^2   +  \frac{\eta^2 L_{\bp} K(V^2+\delta^2)}{2N}\\
& \leq \E [F_{\bp}(\bw_r) ] -   \frac{1}{4}\eta K \E\norm{\nabla F_{\bp}(\bw_r)}^2  +     20\eta^3 K^3  L_{\bp}^2 \zeta_{\bp}^2  +  5\eta^3 K^3 L_{\bp}^2 \delta^2   +  \frac{\eta^2 L_{\bp} K(V^2+\delta^2) }{2N}\\
\end{align*}

Re-arranging terms yields:
\begin{align*}
 \E\norm{\nabla F_{\bp}(\bw_r)}^2 
& \leq 4\frac{\E [F_{\bp}(\bw_r) ] -  \E [F_{\bp}(\bw_{r+1})]   }{\eta K }  +     80 \eta^2 K^2  L_{\bp}^2 \zeta_{\bp}^2  +  20\eta^2 K^2 L_{\bp}^2 \delta^2   +  \frac{4\eta  L_{\bp}  (V^2+\delta^2)}{2N}\\
\end{align*}
Summing over $r=1$ to $R$ yields:
\begin{align*}
 \frac{1}{R}\sum_{r=1}^R\E\norm{\nabla F_{\bp}(\bw_r)}^2 
& \leq 4\frac{\E [F_{\bp}(\bw_0) ]   }{\eta R K }  +     80 \eta^2 K^2  L_{\bp}^2 \zeta_{\bp}  +  20\eta^2 K^2 L_{\bp}^2 \delta^2   +  \frac{4\eta  L_{\bp} (V^2+ \delta^2)}{2N}.
\end{align*}
Finally plugging in $\eta = \frac{1}{L\sqrt{RK}}$ will give the desired rate:
\begin{align*}
 \frac{1}{R}\sum_{r=1}^R\E\norm{\nabla F_{\bp}(\bw_r)}^2 
& \leq O\pare{\frac{L\E [F_{\bp}(\bw_0) ]   }{\sqrt{R K} }  +     \frac{K\zeta_{\bp}^2}{R}    +  \frac{K\delta^2}{R}          +  \frac{ \delta^2}{ N\sqrt{RK}}} .
\end{align*}

\end{proof}

\section{Proof of Convergence of Rolling}\label{app:proof_fedrolex}
In this section, we are going to present convergence proof of Algorithm~\ref{algorithm: Rolling Masked FedAvg}. 
At the start of each epoch, the server shuffles these sub-models and assigns them sequentially to clients. This introduces complexity in analysis due to the interaction between both the model drift caused by partial training on sub-models and the impact of permutation-based assignments on convergence.

\subsection{Proof of Convex Setting} 
In this section, we will present proof of Algorithm~\ref{algorithm: Rolling Masked FedAvg} in convex setting (Theorem~\ref{thm:rolling mask}). We present useful lemmas first.
\begin{proposition}\label{prop:property of masked obj}
Define function $ F_{\mask}(\bw) = \frac{1}{N}\sum_{i=1}^N \frac{1}{d} \sum_{j=1}^d f_i(\mask_i^j\odot \bw) $. If each $f_i$ is $L$ smooth and $\mu$ strongly convex, then $ F_{\mask}$ is also $L$ smooth and $\mu$ strongly convex.
\begin{proof}
     We first examine the smoothness and convexity parameter of $ F_{\mask}$. For smoothness:
   \begin{align*}
     F_{\mask} (\bw+\Delta \bw) &= \frac{1}{N}\sum_{i=1}^N \frac{1}{d} \sum_{j=1}^d f_i(\mask^j_i\odot (\bw+\Delta\bw) )   \\
   & \leq \frac{1}{N}\sum_{i=1}^N \frac{1}{d} \sum_{j=1}^d \pare{ f_i(\mask_i^j\odot \bw  ) + \inprod{\mask^j_i\odot \Delta \bw}{ \nabla f_i(\mask^j_i\odot\bw ) } + \frac{L}{2}\norm{\mask^j_i\odot\Delta\bw}^2 } \\ 
   & =  F_{\mask}(\bw ) + \inprod{ \Delta \bw}{\frac{1}{N}\sum_{i=1}^N \frac{1}{d} \sum_{j=1}^d \mask^j_i\odot\nabla f_i(\mask^j_i\odot\bw ) }    \\
   &\quad + \frac{L}{2} \frac{1}{N}\sum_{i=1}^N \frac{1}{d} \sum_{j=1}^d \inprod{(\mM_i^j)^2 \Delta \bw}{\Delta \bw} \\
   & \leq  F_{\mask}(\bw ) + \inprod{ \Delta \bw}{\nabla  F_{\mask}(\bw ) }    + \frac{L}{2}\pare{   \frac{1}{N}\sum_{i=1}^N \mu_{\max}(\underbrace{\frac{1}{d} \sum_{j=1}^d \mM_i^2)}_{=\mI}   }\norm{\Delta \bw}^2\\
   & =  F_{\mask}(\bw ) + \inprod{ \Delta \bw}{\nabla  F_{\mask}(\bw ) }    + \frac{L}{2} \norm{\Delta \bw}^2.
\end{align*}

For convexity:

 \begin{align*}
     F_{\mask} (\bw+\Delta \bw) &= \frac{1}{N}\sum_{i=1}^N \frac{1}{d} \sum_{j=1}^d f_i(\mask^j_i\odot (\bw+\Delta\bw) )   \\
   & \geq \frac{1}{N}\sum_{i=1}^N \frac{1}{d} \sum_{j=1}^d \pare{ f_i(\mask_i^j\odot \bw  ) + \inprod{\mask^j_i\odot \Delta \bw}{ \nabla f_i(\mask^j_i\odot\bw ) } + \frac{\mu}{2}\norm{\mask^j_i\odot\Delta\bw}^2 } \\ 
   & =  F_{\mask}(\bw ) + \inprod{ \Delta \bw}{\frac{1}{N}\sum_{i=1}^N \frac{1}{d} \sum_{j=1}^d \mask^j_i\odot\nabla f_i(\mask^j_i\odot\bw ) }    \\
   &\quad + \frac{\mu}{2} \frac{1}{N}\sum_{i=1}^N \frac{1}{d} \sum_{j=1}^d \inprod{(\mM_i^j)^2 \Delta \bw}{\Delta \bw} \\
   & \geq  F_{\mask}(\bw ) + \inprod{ \Delta \bw}{\nabla  F_{\mask}(\bw ) }    + \frac{\mu}{2}\pare{   \frac{1}{N}\sum_{i=1}^N \mu_{\min}(\underbrace{\frac{1}{d} \sum_{j=1}^d \mM_i^2)}_{=\mI}   }\norm{\Delta \bw}^2\\
   & =  F_{\mask}(\bw ) + \inprod{ \Delta \bw}{\nabla  F_{\mask}(\bw ) }    + \frac{\mu}{2} \norm{\Delta \bw}^2.
\end{align*}

So $ F_{\mask}$ is also $L$ smooth and $\mu$ strongly convex.
\end{proof}
\end{proposition}

\begin{lemma}\label{lem:lipschitz} Given a mask $\mask= \sbr{\sbr{\mask_i^1,\ldots,\mask_i^R}}_{i=1}^N \in \cbr{0,1}^{dNR}$, we define $\bw^*(\bp):= \arg\min_{\bw \in \cW} \cbr{\Phi(\mask,\bw) := \frac{1}{N}\sum_{i=1}^N  \frac{1}{R}\sum_{j=1}^R f_i(\mask_i^j \odot \bw)} $. We further define $\bar\mask = \sbr{\sbr{\mathbf{1},\ldots,\mathbf{1}}}_{i=1}^N $ and $\bw^*(\bar{\mask}) := \arg\min_{\bw} F(\bw)$. If each $f_i$ is $\mu$ strongly convex, and $\sup_{\bw\in\cW}\norm{\nabla f_i(\bw)} \leq G$, 
 then the following statement holds:
    \begin{align*}
     \|\bw^*(\bar{\mask}) - \bw^*(\mask)\|\leq  \sqrt{\frac{2G^2 + 2W^2L^2}{\mu^2 NR}} \| \mask  -  \bar{\mask} \|   . 
    \end{align*}
   \begin{proof}
   We define $\Phi(\mask,\bw):= \frac{1}{N}\frac{1}{R}\sum_{i=1}^N\sum_{j=1}^R f_i(\mask_i^j \odot \bw) $.
         First, according to optimality conditions we have:
        \begin{align*}
            \langle \bw - \bw^*(\mask), \nabla_{2} \Phi(\mask, \bw^*(\mask)) \rangle \geq 0,\\
            \langle \bw - \bw^*(\bar{\mask}), \nabla_{2} \Phi(\bar{\mask}, \bw^*(\bar{\mask})) \rangle \geq 0
        \end{align*}
        Substituting $\bw$ with $\bw^*(\bar{\mask})$ and $\bw^*(\mask)$ in the above first and second inequalities respectively yields:
        \begin{align*}
             \langle\bw^*(\bar{\mask}) - \bw^*(\mask), \nabla_{2} \Phi(\mask, \bw^*(\mask)) \rangle \geq 0,\\
            \langle \bw^*(\mask)- \bw^*(\bar{\mask}), \nabla_{2} \Phi(\bar{\mask}, \bw^*(\bar{\mask})) \rangle \geq 0.
        \end{align*}
        Adding up the above two inequalities yields:
        \begin{align}
            \langle \bw^*(\bar{\mask}) - \bw^*(\mask), \nabla_{2} \Phi(\mask, \bw^*(\mask))-\nabla_{2} \Phi(\bar{\mask}, \bw^*(\bar{\mask}))   \rangle \geq 0. \label{eq: lipschitz 1} 
        \end{align}
        Since $F(\mask,\cdot)$ is $\mu$ strongly convex, as shown in Proposition~\ref{prop:property of masked obj}, we have:
        \begin{align}
            \langle \bw^*(\bar{\mask}) - \bw^*(\mask), \nabla_2 \Phi(\mask, \bw^*(\bar{\mask})) - \nabla_2 \Phi(\mask, \bw^*(\mask)) \geq \mu \|\bw^*(\bar{\mask}) - \bw^*(\mask)\|^2. \label{eq: lipschitz 2} 
        \end{align}
        Adding up (\ref{eq: lipschitz 1}) and (\ref{eq: lipschitz 2}) yields:
        \begin{align*}
            \langle \bw^*(\bar{\mask}) - \bw^*(\mask), \nabla_2 \Phi(\mask, \bw^*(\bar{\mask})) - \nabla_2 \Phi(\bar{\mask}, \bw^*(\bar{\mask})) \geq \mu \|\bw^*(\bar{\mask}) - \bw^*(\mask)\|^2
        \end{align*}
        Now we examine the smoothness of $\nabla_2 \Phi(\mask,\bw)$ in terms of the first variable.

        \begin{align*}
            \norm{ \nabla_2 \Phi(\mask,\bw) - \nabla_2 \Phi(\bar{\mask},\bw)   }^2 &= \norm{ \frac{1}{N}\sum_{i=1}^N \frac{1}{R}\sum_{j=1}^R \mask_i^j \odot \nabla f_i(\mask_i^j\odot \bw) -  \frac{1}{N}\sum_{i=1}^N \frac{1}{R}\sum_{j=1}^R \mathbf{1} \odot \nabla f_i(\bar{\mask}_i^j\odot \bw)  }^2\\
           & \leq \frac{1}{N}\sum_{i=1}^N \frac{1}{R}\sum_{j=1}^R \norm{  \mask_i^j \odot \nabla f_i(\mask_i^j\odot \bw) -    \bar{\mask}_i^j \odot \nabla f_i(\mathbf{1}\odot \bw)  }^2\\
           & \leq \frac{1}{N}\sum_{i=1}^N \frac{1}{R}\sum_{j=1}^R \pare{ 2G^2\norm{  \mask_i^j   -  \mathbf{1}   }^2 
 + 2W^2L^2\norm{\mask_i^j   -    \mathbf{1}}^2}\\
 & = \frac{2G^2 + 2W^2L^2}{NR}\norm{ \mask  -    \bar{\mask} }^2.
        \end{align*}
        Finally, using $ \sqrt{\frac{2G^2 + 2W^2L^2}{NR}}$ smoothness of $\nabla_2 \Phi(\cdot,\bw)$ will conclude the proof:
        \begin{align*}
           \sqrt{\frac{2G^2 + 2W^2L^2}{NR}}\| \bw^*(\bar{\mask}) - \bw^*(\mask)\| \|\mask   -  \bar{\mask} \| &\geq \mu \|\bw^*(\bar{\mask}) - \bw^*(\mask)\|^2\\
            \Longleftrightarrow   \sqrt{\frac{2G^2 + 2W^2L^2}{\mu^2 NR}} \| \mask  -  \bar{\mask} \| &\geq  \|\bw^*(\bar{\mask}) - \bw^*(\mask)\|. 
        \end{align*}
    \end{proof}
\end{lemma}

\begin{lemma} [Optimality Gap]\label{lem:opt gap}
Let $\Phi(\mask,\bw) := \frac{1}{N}\sum_{i=1}^N \frac{1}{R}\sum_{i=1}^j f_i(\mask_i^j \odot \bw)$ . Let $\hat{\bv} = \cP_{\cW}(\tilde{\bw} - \frac{1}{L} \nabla_{\bw} \Phi(\mask,\tilde{\bw}))$. If we assume each $f_i$ is $L$-smooth, $\mu$-strongly convex and with gradient bounded by $G$, then the following statement holds true:
    \begin{align*}
      F(  \hat{\bw}) -  F(  \bw^* ) \leq   2L\|\tilde{\bw}  - {\bw}^*(\mask)\|^2+ \pare{  \frac{ 2L}{\mu  }    + \frac{4 }{ L}  }{\frac{2G^2 + 2W^2L^2}{NR}} \sum_{j=1}^R\| \mask_i^j  - \mathbf{1} \|^2,
    \end{align*}
    where    ${\bw}^* = \arg\min_{\bw \in \cW} \Phi( \mathbf{1},  {\bw} )$.

    \begin{proof}
    From Lemma~\ref{lem:lipschitz} we know $\nabla_{\bw} \Phi(\cdot,\bw)$ is $ \sqrt{\frac{2G^2 + 2W^2L^2}{NR}} $ Lipschitz and we know $\bw^*(\balpha)$ is $\kappa_{\Phi}:= \frac{\sqrt{N}G}{\mu}$ Lipschitz. According to property of projection, we have:
    \begin{align*}
        0 &\leq \inprod{\bw - \hat{\bw}}{L(\hat{\bw} - \tilde{\bw})+\nabla_{\bw} \Phi(\mask,\tilde{\bw}) }\\
        &= \underbrace{\inprod{\bw - \hat{\bw}}{L(\hat{\bw} - \tilde{\bw})+\nabla_{\bw} \Phi( \mathbf{1},\tilde{\bw}) } }_{T_1}+ \underbrace{\inprod{\bw - \hat{\bw}}{ \nabla_{\bw} \Phi(\mask,\tilde{\bw}) -\nabla_{\bw} \Phi( \mathbf{1},\tilde{\bw})}}_{T_2}.
    \end{align*}

For $T_1$, we notice:
\begin{align*}
   \inprod{\bw - \hat{\bw}}{L(\hat{\bw} - \tilde{\bw})+\nabla_{\bw} \Phi( \mathbf{1},\tilde{\bw}) }  & =L \inprod{\bw - \tilde{\bw}} {\hat{\bw} - \tilde{\bw}} + L \inprod{ \tilde{\bw}-\hat{\bw}}{\hat{\bw} - \tilde{\bw}} +  \inprod{\bw - \hat{\bw}}{ \nabla_{\bw} \Phi( \mathbf{1},\tilde{\bw}) }\\
   & =L \inprod{\bw - \tilde{\bw}} {\hat{\bw} - \tilde{\bw}} - L \norm{ \tilde{\bw}-\hat{\bw}}^2 +   \inprod{\bw - \hat{\bw}_i}{ \nabla_{\bw} \Phi( \mathbf{1},\tilde{\bw}) } \\
   & \leq L ( \norm{\bw - \tilde{\bw}}^2 + \frac{1}{4}\norm{\hat{\bw} - \tilde{\bw}}^2 )- L \norm{ \tilde{\bw}-\hat{\bw}}^2 +  \underbrace{\inprod{\bw - \hat{\bw}}{ \nabla_{\bw} \Phi(\mathbf{1},\tilde{\bw}) }}_{\spadesuit } 
\end{align*}
where at last step we used Young's inequality.
To bound $\spadesuit$,  we apply the $L$ smoothness and $\mu$ strongly convexity of $\Phi(\mathbf{1},\cdot)$:
\begin{align*}
    \inprod{\bw - \hat{\bw} }{ \nabla_{\bw} \Phi(\mathbf{1},\tilde{\bw} ) } &= \inprod{\bw - \tilde{\bw} }{ \nabla_{\bw} \Phi( \mathbf{1},\tilde{\bw}) } + \inprod{\tilde{\bw} - \hat{\bw}}{ \nabla_{\bw} \Phi(\mathbf{1},\tilde{\bw}) }\\
    &\leq \Phi( \mathbf{1}, \bw ) - \Phi( \mathbf{1},\tilde{\bw}) - \frac{\mu}{2}\norm{\tilde{\bw} -  {\bw} }^2+ \Phi( \mathbf{1},\tilde{\bw}) - \Phi( \mathbf{1},\hat{\bw}) + \frac{L}{2}\norm{\tilde{\bw} - \hat{\bw}}^2\\
    &\leq \Phi( \mathbf{1}, \bw ) - \Phi( \mathbf{1},\hat{\bw})   - \frac{\mu}{2}\norm{\tilde{\bw} -  {\bw} }^2+  \frac{L}{2}\norm{\tilde{\bw} - \hat{\bw}}^2
\end{align*}
Putting above bound back yields:
\begin{align*}
     \inprod{\bw - \hat{\bw}}{L(\hat{\bw} - \tilde{\bw})+\nabla_{\bw} \Phi( \mathbf{1},\tilde{\bw}) } \leq \Phi( \mathbf{1}, \bw ) - \Phi( \mathbf{1},\hat{\bw})  + \frac{1}{2\eta}\norm{\tilde{\bw} -  {\bw} }^2 -\pare{\frac{3L}{4}- \frac{L}{2}}  \norm{\tilde{\bw} - \hat{\bw}}^2
\end{align*}
Now we switch to bounding $T_2$. Applying Cauchy-Schwartz yields:
\begin{align*}
    \inprod{\bw - \hat{\bw} }{ \nabla_{\bw} \Phi(\mask,\tilde{\bw} ) -\nabla_{\bw} \Phi( \mathbf{1} ,\tilde{\bw} )} &\leq \frac{L}{4}\norm{\bw - \tilde{\bw} }^2 + \frac{L}{4}\norm{\tilde{\bw} - \hat{\bw} }^2 + \frac{4}{L} \norm{\nabla_{\bw} \Phi(\mask,\tilde{\bw} ) -\nabla_{\bw} \Phi( \mathbf{1} ,\tilde{\bw} )}^2\\
    &\leq \frac{L}{4}\norm{\bw - \tilde{\bw} }^2 + \frac{L}{4}\norm{\tilde{\bw} - \hat{\bw} }^2 + \frac{4 }{L} {\frac{2G^2 + 2W^2L^2}{NR}}  \norm{  \mask  -  \mathbf{1}  }^2
\end{align*}
where at last step we apply $\sqrt{\frac{2G^2 + 2W^2L^2}{NR}}$ smoothness of $\nabla_2 \Phi(\cdot,\bw)$.
Putting pieces together yields:
\begin{align*}
    0 &\leq \Phi( \mathbf{1} , \bw ) - \Phi( \mathbf{1} ,\hat{\bw} )  + \frac{L}{2}\norm{\tilde{\bw}  -  {\bw} }^2    +\frac{L}{2}\norm{\bw - \tilde{\bw} }^2   + \frac{4 }{  L} {\frac{2G^2 + 2W^2L^2}{NR}}\norm{ \mask   -  \mathbf{1}_i }^2
\end{align*}
Re-arranging terms and setting $\bw = \bw^*(\mathbf{1}) = \arg\min_{\bw\in\cW} \Phi(\mathbf{1},\bw)$ yields:
\begin{align*}
     \Phi( \mathbf{1},\hat{\bw}) -  \Phi( \mathbf{1}, \bw^*(\mathbf{1}) ) \leq   L\norm{\tilde{\bw} -  {\bw}^* }^2     + \frac{4 }{ L}{\frac{2G^2 + 2W^2L^2}{NR}} \norm{  \mask  -  \mathbf{1} }^2.
\end{align*}

At last, due to the $\kappa_{\Phi}$-Lipschitzness  property of of $\bw^*(\cdot)$  as shown in Lemma~\ref{lem:lipschitz}, it follows that:
        \begin{align*}
            L\|\tilde{\bw}  - {\bw}^*(\mathbf{1})\|^2 &\leq L\|\tilde{\bw} - {\bw}^*(\mask)\|^2+L\|{\bw}^*(\mask) - {\bw}^*(\mathbf{1})\|^2 \\
            &\leq 2L\|\tilde{\bw}  - {\bw}^*(\mask)\|^2+2{\frac{2G^2 + 2W^2L^2}{\mu NR}} L\| \mask -  \mathbf{1} \|^2,
        \end{align*}
        as desired.

    \end{proof}
    
    \end{lemma}

\begin{lemma}[Recursion between each round]\label{lm:round recursion}
For Algorithm~\ref{algorithm: Rolling Masked FedAvg}, under the assumptions of Theorem~\ref{thm:rolling mask}, the following statement holds:
    \begin{align*}
          \bw_{e,r+1}  
            = \bw_{e,r} - \eta \frac{1}{N}\sum_{i=1}^N  K \mask_i^{\sigma_e(r)}\odot \nabla f_i(\mask_i^{\sigma_e(r)}\odot \bw_{e,r} )  - \eta\frac{1}{N}\sum_{i=1}^N\sum_{k=0}^{K-1}\br^i_{e,r,k}  - \bxi_{e,r}.
    \end{align*}
     where $\bxi_{e,r} = \eta \frac{1}{N}\sum_{i=1}^N  \sum_{k=0}^{K-1}\pare{\mask_i^{\sigma_e(r)}\odot \nabla f_i(\mask_i^{\sigma_e(r)}\odot \bw_{e,r,k} ) - \mask_i^{\sigma_e(r)}\odot \nabla f_i(\mask_i^{\sigma_e(r)}\odot \bw_{e,r,k};\xi_{e,r,k}^i)}$, and $\br^i_{e,r,k} = \mask_i^{\sigma_e(r)}\odot\mH^i_{e,r,k}\pare{\mask_i^{\sigma_e(r)}\odot \bw^i_{e,r,k}    -   \mask_i^{\sigma_e(r)}\odot \bw_{e,r} } $, $\mH^i_{e,r,k} \in \R^{d\times d}$ are some real matrix such that $\mH^i_{e,r,k} \preceq L \mathbf{I}$.
    \begin{proof}
        According to updating rule we have
        \begin{align*}
            \bw_{e,r+1} = \bw_{e,r} - \eta \frac{1}{N}\sum_{i=1}^N  \sum_{k=0}^{K-1}\mask_i^{\sigma_e(r)}\odot \nabla f_i(\mask_i^{\sigma_e(r)}\odot \bw_{e,r,k};\xi_{e,r,k}^i)\\
            = \bw_{e,r} - \eta \frac{1}{N}\sum_{i=1}^N  \sum_{k=0}^{K-1}\mask_i^{\sigma_e(r)}\odot \nabla f_i(\mask_i^{\sigma_e(r)}\odot \bw_{e,r,k} ) - \bxi_{e,r}\\
        \end{align*}

        Notice the following fact:
        \begin{align*}
         \mask_i^{\sigma_e(r)}\odot  \nabla f_i(\mask_i^{\sigma_e(r)}\odot \bw^i_{e,r,k} ) &= \mask_i^{\sigma_e(r)}\odot\nabla f_i(\mask_i^{\sigma_e(r)}\odot \bw_{e,r} )\\
         &\quad + \mask_i^{\sigma_e(r)}\odot\nabla f_i(\mask_i^{\sigma_e(r)}\odot \bw^i_{e,r,k} ) - \mask_i^{\sigma_e(r)}\odot\nabla f_i(\mask_i^{\sigma_e(r)}\odot \bw_{e,r} )\\
          & =\mask_i^{\sigma_e(r)}\odot\nabla f_i(\mask_i^{\sigma_e(r)}\odot \bw_{e,r} ) \\
          &\quad +  \underbrace{\mask_i^{\sigma_e(r)}\odot\mH^i_{e,r,k}\pare{\mask_i^{\sigma_e(r)}\odot \bw^i_{e,r,k}    -   \mask_i^{\sigma_e(r)}\odot \bw_{e,r} }}_{\br^i_{e,r,k}}.
        \end{align*}
        Hence we can conclude that:
        \begin{align*}
             \bw_{e,r+1}  
            = \bw_{e,r} - \eta \frac{1}{N}\sum_{i=1}^N  K \mask_i^{\sigma_e(r)}\odot \nabla f_i(\mask_i^{\sigma_e(r)}\odot \bw_{e,r} )  - \eta\frac{1}{N}\sum_{i=1}^N\sum_{k=0}^{K-1}\br^i_{e,r,k}  - \bxi_{e,r}.
        \end{align*}
    \end{proof}
\end{lemma}

 \begin{lemma}[Recursion between each epoch]\label{lem:epoch recursion}
 For Algorithm~\ref{algorithm: Rolling Masked FedAvg}, under the assumptions of Theorem~\ref{thm:rolling mask}, the following statement holds:
\begin{align*}
          \bw_{e,R} - \bw_{e}   
           &=-     \eta R K   \nabla  F_{\mask}(  \bw_{e } )  - \sum_{r=0}^{R-1} \mA_r \pare{\eta\frac{1}{N}\sum_{i=1}^N\sum_{k=0}^{K-1}\br^i_{e,r,k}  + \bxi_{e,r}} \\
            &\quad   + \eta^2 K^2 \sum_{r=0}^{R-2} \mA_{r+1}\frac{1}{N}\sum_{i=1}^N   \mM_i^{\sigma_e(r+1)} \mH^i_{e,r+1} \mM_i^{\sigma_e(r+1)}   \sum_{j=0}^r \frac{1}{N}\sum_{i=1}^N     \nabla f_i^{\sigma_e(j)}(  \bw_{e } ),
    \end{align*}
    where  $\mA_r:=  \prod_{r'=R-1}^{r+1} \pare{\mI - \eta K  \frac{1}{N}\sum_{i=1}^N   \mM_i^{\sigma_e(r')} \mH^i_{e,r'} \mM_i^{\sigma_e(r')} }$ .
    \begin{proof}
    Define $f_i^{\sigma_e(r)}(  \bw  ) : =   f_i(\mask_i^{\sigma_e(r)}\odot \bw ) $. According to Lemma~\ref{lm:round recursion} we have
    \begin{align*}
       \bw_{e,r+1} - \bw_{e}  
            &= \bw_{e,r} - \bw_{e}  - \eta \frac{1}{N}\sum_{i=1}^N  K    \nabla f_i^{\sigma_e(r)}(  \bw_{e,r} )  -\eta \frac{1}{N}\sum_{i=1}^N\sum_{k=0}^{K-1}\br^i_{e,r,k}  - \bxi_{e,r}\nonumber\\
            &= \bw_{e,r} - \bw_{e}  - \eta K \frac{1}{N}\sum_{i=1}^N     \nabla f_i^{\sigma_e(r)}(  \bw_{e } ) \nonumber\\
            &\quad -\eta K  \frac{1}{N}\sum_{i=1}^N    \pare{  \nabla f_i^{\sigma_e(r)}(  \bw_{e,r} ) -    \nabla f_i^{\sigma_e(r)}(  \bw_{e } ) }  - \eta \frac{1}{N}\sum_{i=1}^N\sum_{k=0}^{K-1}\br^i_{e,r,k}  - \bxi_{e,r}\nonumber.
    \end{align*}
    Applying mean-value theorem on the $\nabla f_i^{\sigma_e(r)}(\cdot)$ yields:
    \begin{align}
            \bw_{e,r+1} - \bw_{e}        &= \bw_{e,r} - \bw_{e}  - \eta K \frac{1}{N}\sum_{i=1}^N     \nabla f_i^{\sigma_e(r)}(  \bw_{e } ) \nonumber \\
            &\quad -\eta K  \frac{1}{N}\sum_{i=1}^N   \mask_i^{\sigma_e(r)}\odot\mH^i_{e,r} \mask_i^{\sigma_e(r)}\odot \pare{    \bw_{e,r}   -     \bw_{e }   }  - \eta\frac{1}{N}\sum_{i=1}^N\sum_{k=0}^{K-1}\br^i_{e,r,k}  - \bxi_{e,r}\nonumber\\
            &=\pare{\mI - \eta K  \frac{1}{N}\sum_{i=1}^N   \mM_i^{\sigma_e(r)} \mH^i_{e,r} \mM_i^{\sigma_e(r)} } \pare{\bw_{e,r} - \bw_{e}}  - \eta K \frac{1}{N}\sum_{i=1}^N     \nabla f_i^{\sigma_e(r)}(  \bw_{e } ) \nonumber\\
            &\quad -\eta \frac{1}{N}\sum_{i=1}^N\sum_{k=0}^{K-1}\br^i_{e,r,k}  - \bxi_{e,r} .\label{eq:epoch recursion 0}
    \end{align}
   Unrolling the recursion from $r=R-1$ to $0$ yields:
    \begin{align*}
        \bw_{e,R} - \bw_{e}  =\sum_{r=0}^{R-1}  \mA_r  \pare{  - \eta K \frac{1}{N}\sum_{i=1}^N     \nabla f_i^{\sigma_e(r)}(  \bw_{e } )  -\eta \frac{1}{N}\sum_{i=1}^N\sum_{k=0}^{K-1}\br^i_{e,r,k}  - \bxi_{e,r}}.
    \end{align*} 
    According to summation by part$\sum_{r=0}^{R-1} \mA_r \bb_r = \mA_{R-1} \sum_{r=0}^{R-1}\bb_r - \sum_{i=0}^{R-2} ( \mA_{i+1} - \mA_i ) \sum_{j=0}^i \bb_j  $ we have
    \begin{align*}
          \bw_{e,R} - \bw_{e}   
            &=-\sum_{r=0}^{R-1}    \pare{   \eta K \frac{1}{N}\sum_{i=1}^N     \nabla f_i^{\sigma_e(r)}(  \bw_{e } )   } + \eta K \sum_{i=0}^{R-2} ( \mA_{i+1} - \mA_i ) \sum_{j=0}^i \frac{1}{N}\sum_{i=1}^N     \nabla f_i^{\sigma_e(j)}(  \bw_{e } )   \\
            &\quad - \sum_{r=0}^{R-1} \mA_r \pare{\eta\frac{1}{N}\sum_{i=1}^N\sum_{k=0}^{K-1}\br^i_{e,r,k}  + \bxi_{e,r}}\\
             &=-     \eta R K   \nabla  F_{\mask}(  \bw_{e } )  + \eta K \sum_{r=0}^{R-2} ( \mA_{r+1} - \mA_r ) \sum_{j=0}^r \frac{1}{N}\sum_{i=1}^N     \nabla f_i^{\sigma_e(j)}(  \bw_{e } )   \\
            &\quad - \sum_{r=0}^{R-1} \mA_r \pare{\eta\frac{1}{N}\sum_{i=1}^N\sum_{k=0}^{K-1}\br^i_{e,r,k}  + \bxi_{e,r}}\\
            &=-     \eta R K   \nabla  F_{\mask}(  \bw_{e } ) \\
            & \quad  + \eta^2 K^2 \sum_{r=0}^{R-2} \mA_{r+1}  \frac{1}{N}\sum_{i=1}^N   \mM_i^{\sigma_e(r+1)} \mH^i_{e,r+1} \mM_i^{\sigma_e(r+1)}   \sum_{j=0}^r \frac{1}{N}\sum_{i=1}^N     \nabla f_i^{\sigma_e(j)}(  \bw_{e } )   \\
            &\quad - \sum_{r=0}^{R-1} \mA_r \pare{\eta\frac{1}{N}\sum_{i=1}^N\sum_{k=0}^{K-1}\br^i_{e,r,k}  + \bxi_{e,r}}.
    \end{align*}

    \end{proof}
\end{lemma}

\begin{lemma}[Local model deviation]\label{lem:local model deviation}
For Algorithm~\ref{algorithm: Rolling Masked FedAvg}, under the assumptions of Theorem~\ref{thm:rolling mask}, the following statement holds with probability at least $1-\nu$:
\begin{align*}
&   \sum_{r=0}^{R-1}\sum_{k=0}^{K-1}  \frac{1}{N}\sum_{i=1}^N \E \norm{ \bw_{e,r} - \bw_{e,r,k+1}^i}^2 \\
           & \leq    48\eta^2 R K^3   \zeta + \pare{ 12 \pare{ 6R^3\eta^4 K^5 L^2+  18R^5  \eta^6 K^7    L^4  } 
 + 48\eta^2 R K^3}\E\norm{\nabla  F_{\mask}(\bw_e)}^2 \\
         & \quad +   12    \pare{48R^2\eta^4 K^5 L^2+432 R^4 \eta^6 K^7        L^4 } G  { \log(2RK/\nu)}  \\ 
            &  \quad +   216     R^3 \eta^4 K^5 L^2 \frac{\delta^2}{N}+ 6\eta^2 R K^3 \delta^2 .
          \end{align*} 

and for any $i \in [N]$ we have
  \begin{align*}
   \sum_{r=0}^{R-1}\sum_{k=0}^{K-1}   \E \norm{ \bw_{e,r} - \bw_{e,r,k+1}^i}^2 
           & \leq    48\eta^2 R K^3   \zeta_i + \pare{ 12 \pare{ 6R^3\eta^4 K^5 L^2+  18R^5  \eta^6 K^7    L^4  } 
 + 48\eta^2 R K^3}\E\norm{\nabla  F_{\mask}(\bw_e)}^2 \\
         &  +   12    \pare{48R^2\eta^4 K^5 L^2+432 R^4 \eta^6 K^7        L^4 } G  { \log(2RK/\nu)}  \\ 
            &  +   216     R^3 \eta^4 K^5 L^2 \frac{\delta^2}{N}+ 6\eta^2 R K^3 \delta^2 .
          \end{align*} 
    \begin{proof}
    According to updating rule we have:
        \begin{align*}
         \E \norm{ \bw_{e,r} - \bw_{e,r,k+1}^i}^2  &= \pare{1+\frac{1}{K-1}}  \E \norm{\bw_{e,r} - \bw_{e,r,k }^i }^2  + \eta^2 K     \E \norm{\mask_i^{\sigma_e(r)} \odot \nabla f_i(\mask_i^{\sigma_e(r)} \odot \bw_{e,r,k}^i;\xi_{e,r,k}^i  )}^2\\
          & = \pare{1+\frac{1}{K-1}}  \E \norm{\bw_{e,r} - \bw_{e,r,k }^i }^2  + \eta^2 K     \E \norm{\mask_i^{\sigma_e(r)} \odot \nabla f_i(\mask_i^{\sigma_e(r)} \odot \bw_{e,r,k}^i;\xi_{e,r,k}^i  )}^2 + \eta^2 K \delta^2\\
         &  \leq \pare{1+\frac{1}{K-1}}  \E \norm{\bw_{e,r} - \bw_{e,r,k }^i }^2  + 2\eta^2 K     \E \norm{\mask_i^{\sigma_e(r)} \odot \nabla f_i(\mask_i^{\sigma_e(r)} \odot \bw_{e,r }    )}^2 + \eta^2 K \delta^2\\
         &\quad + 2\eta^2 K  L^2   \E \norm{  \bw_{e,r,k}^i  -  \bw_{e,r }   }^2\\ 
         &  \leq \pare{1+\frac{1}{K-1}}  \E \norm{\bw_{e,r} - \bw_{e,r,k }^i }^2    + 4\eta^2 K     \E \norm{\mask_i^{\sigma_e(r)} \odot \nabla f_i(\mask_i^{\sigma_e(r)} \odot \bw_{e }    )}^2  \\
         & \quad +4\eta^2 K     \E \norm{\mask_i^{\sigma_e(r)} \odot \nabla f_i(\mask_i^{\sigma_e(r)} \odot \bw_{e,r }    ) - \mask_i^{\sigma_e(r)} \odot \nabla f_i(\mask_i^{\sigma_e(r)} \odot \bw_{e  }    )}^2  + \eta^2 K \delta^2\\
         &\quad + 2\eta^2 K  L^2   \E \norm{  \bw_{e,r,k}^i  -  \bw_{e,r }   }^2\\ 
         &  \leq \pare{1+\frac{1}{K-1}}  \E \norm{\bw_{e,r} - \bw_{e,r,k }^i }^2    + 4\eta^2 K     \E \norm{\mask_i^{\sigma_e(r)} \odot \nabla f_i(\mask_i^{\sigma_e(r)} \odot \bw_{e }    )}^2  \\
         & \quad +4\eta^2 K L^2    \E \norm{  \bw_{e,r }    -   \bw_{e  }   }^2  + \eta^2 K \delta^2 + 2\eta^2 K  L^2   \E \norm{  \bw_{e,r,k}^i  -  \bw_{e,r }   }^2.
          \end{align*}
          Due to $2\eta^2 K L^2 \leq \frac{1}{K-1}$, we have
          \begin{align}
         \E \norm{ \bw_{e,r} - \bw_{e,r,k+1}^i}^2    
         &  \leq \pare{1+\frac{2}{K-1}}  \E \norm{\bw_{e,r} - \bw_{e,r,k }^i }^2    + 4\eta^2 K     \E \norm{  \nabla f_i^{\sigma_e(r)}(  \bw_{e }    )}^2  \\
         & \quad +4\eta^2 K L^2    \E \norm{  \bw_{e,r }    -   \bw_{e  }   }^2  + \eta^2 K \delta^2.  \label{eq:local model deviation 0}
          \end{align}
        Due to (\ref{eq:epoch recursion 0}) we have
         \begin{align*}
       \bw_{e,r } - \bw_{e}  
            &=\pare{\mI - \eta K  \frac{1}{N}\sum_{i=1}^N   \mM_i^{\sigma_e(r-1)} \mH^i_{e,r-1} \mM_i^{\sigma_e(r-1)} } \pare{\bw_{e,r-1} - \bw_{e}}\\
            &\quad - \eta K \frac{1}{N}\sum_{i=1}^N     \nabla f_i^{\sigma_e(r-1)}(  \bw_{e } )  - \eta\frac{1}{N}\sum_{i=1}^N\sum_{k=0}^{K-1}\br^i_{e,r-1,k}  - \bxi_{e,r-1}, r = 1,\ldots, R. 
    \end{align*}
    Unrolling the recursion yields:
         \begin{align*}
        \bw_{e,r} - \bw_{e}  
            &=\sum_{p=0}^{r-1}   \underbrace{\prod_{r'=r-1}^{p+1} \pare{\mI - \eta K  \frac{1}{N}\sum_{i=1}^N   \mM_i^{\sigma_e(r')} \mH^i_{e,r'} \mM_i^{\sigma_e(r')} }}_{\mA_p}   \pare{  - \eta K \frac{1}{N}\sum_{i=1}^N     \nabla f_i^{\sigma_e(p)}(  \bw_{e } )  -\eta \frac{1}{N}\sum_{i=1}^N\sum_{k=0}^{K-1}\br^i_{e,p,k}  - \bxi_{e,p}}.  
    \end{align*}
 According to summation by part$\sum_{p=0}^{r-1} \mA_p \bb_p = \mA_{r-1} \sum_{j=0}^{r-1}\bb_j - \sum_{i=0}^{r-2} ( \mA_{i+1} - \mA_i ) \sum_{j=0}^i \bb_j  $ we have
 \begin{align*}
      &  \bw_{e,r} - \bw_{e} \\ 
            &=\sum_{p=0}^{r-1}   \prod_{r'=r-1}^{p+1} \pare{\mI - \eta K  \frac{1}{N}\sum_{i=1}^N   \mM_i^{\sigma_e(r')} \mH^i_{e,r'} \mM_i^{\sigma_e(r')} }   \pare{  - \eta K \frac{1}{N}\sum_{i=1}^N \nabla f_i^{\sigma_e(p)}(  \bw_{e } )  }  \\
             &\quad +\sum_{p=0}^{r-1}   \prod_{r'=r-1}^{p+1} \pare{\mI - \eta K  \frac{1}{N}\sum_{i=1}^N   \mM_i^{\sigma_e(r')} \mH^i_{e,r'} \mM_i^{\sigma_e(r')} }   \pare{   - \eta\frac{1}{N}\sum_{i=1}^N\sum_{k=0}^{K-1}\br^i_{e,p,k}  - \bxi_{e,p}}\\
             &=\sum_{p=0}^{r-1}     \pare{  - \eta K \frac{1}{N}\sum_{i=1}^N \nabla f_i^{\sigma_e(p)}(  \bw_{e } )  } - \sum_{p=0}^{r-2} \pare{ \mA_{p+1} - \mA_{p} } \sum_{j=0}^p  \pare{  - \eta K \frac{1}{N}\sum_{i=1}^N \nabla f_i^{\sigma_e(j)}(  \bw_{e } )  }  \\
             &\quad +\sum_{p=0}^{r-1}   \prod_{r'=r-1}^{p+1} \pare{\mI - \eta K  \frac{1}{N}\sum_{i=1}^N   \mM_i^{\sigma_e(r')} \mH^i_{e,r'} \mM_i^{\sigma_e(r')} }   \pare{   - \eta\frac{1}{N}\sum_{i=1}^N\sum_{k=0}^{K-1}\br^i_{e,p,k}  - \bxi_{e,p}}\\ 
              &=\sum_{p=0}^{r-1}     \pare{  - \eta K \frac{1}{N}\sum_{i=1}^N \nabla f_i^{\sigma_e(p)}(  \bw_{e } )  }\\
              &  + \sum_{p=0}^{r-2} \prod_{r'=r-1}^{p+2} \pare{\mI - \eta K  \frac{1}{N}\sum_{i=1}^N   \mM_i^{\sigma_e(r')} \mH^i_{e,r'} \mM_i^{\sigma_e(r')} } \eta^2 K^2  \frac{1}{N}\sum_{i=1}^N   \mM_i^{\sigma_e(p+1)} \mH^i_{e,p+1} \mM_i^{\sigma_e(p+1)}\sum_{j=0}^p  \pare{     \frac{1}{N}\sum_{i=1}^N \nabla f_i^{\sigma_e(j)}(  \bw_{e } )  }  \\
             &  +\sum_{p=0}^{r-1}   \prod_{r'=r-1}^{p+1} \pare{\mI - \eta K  \frac{1}{N}\sum_{i=1}^N   \mM_i^{\sigma_e(r')} \mH^i_{e,r'} \mM_i^{\sigma_e(r')} }   \pare{   - \eta\frac{1}{N}\sum_{i=1}^N\sum_{k=0}^{K-1}\br^i_{e,p,k}  - \bxi_{e,p}}.
    \end{align*}

    Taking expected norm on both side yields:
     \begin{align*}
      & \E \norm{\bw_{e,r} - \bw_{e}}^2   \\
            &=     3\eta^2 K^2\E \norm{\sum_{p=0}^{r-1} \frac{1}{N}\sum_{i=1}^N \nabla f_i^{\sigma_e(p)}(  \bw_{e } )}^2  + 3\eta^4 K^4 r \sum_{p=0}^{r-2} \prod_{r'=r-1}^{p+2} (1+\eta K L)^{2r'}  L^2 \E\norm{  \sum_{j=0}^p \frac{1}{N}\sum_{i=1}^N     \nabla f_i^{\sigma_e(j)}(  \bw_{e } )}^2   \\
            &\quad + 3 r \sum_{p=0}^{r-1}  \prod_{r'=r-1}^{p+1} (1+\eta K L)^{2r'} \E \norm{\eta\frac{1}{N}\sum_{i=1}^N\sum_{k=0}^{K-1}\br^i_{e,p,k}  + \bxi_{e,p}}^2\\
            &\leq 3\eta^2 K^2\E \norm{\sum_{p=0}^{r-1} \frac{1}{N}\sum_{i=1}^N \nabla f_i^{\sigma_e(p)}(  \bw_{e } )}^2  + 3\eta^4 K^4 r \sum_{p=0}^{r-2}  9  L^2 \E\norm{  \sum_{j=0}^p \frac{1}{N}\sum_{i=1}^N     \nabla f_i^{\sigma_e(j)}(  \bw_{e } )}^2   \\
            &\quad + 18\eta^2 r \sum_{p=0}^{r-1}  \frac{1}{N}\sum_{i=1}^N KL^2\sum_{k=0}^{K-1} \E \norm{\bw^i_{e,p,k} - \bw_{e,p}   }^2+ 18 r \sum_{p=0}^{r-1}    \E \norm{ \bxi_{e,p}}^2.
    \end{align*}
    According to Hoeffding-Serfling inequality~\cite[Theorem 2]{pmlr-v51-schneider16} we know
    \begin{align*}
      \norm{\sum_{p=0}^{r-1} \frac{1}{N}\sum_{i=1}^N \nabla f_i^{\sigma_e(p)}(  \bw_{e } )}^2 & \leq  2r^2 \norm{\nabla  F_{\mask}(\bw_e) - \frac{1}{r} \sum_{p=0}^{r-1} \frac{1}{N}\sum_{i=1}^N \nabla f_i^{\sigma_e(p)}(  \bw_{e } )}^2 + 2r^2 \norm{\nabla  F_{\mask}(\bw_e)}^2\\
      &\leq 2r^2  G  \frac{8(1-\frac{r-1}{N})\log(2/\nu)}{r} + 2r^2 \norm{\nabla  F_{\mask}(\bw_e)}^2. 
    \end{align*}
    Similarly we know 
    \begin{align*}
       \norm{  \sum_{j=0}^p \frac{1}{N}\sum_{i=1}^N     \nabla f_i^{\sigma_e(j)}(  \bw_{e } )}^2
      &\leq 2(p+1)^2  G  \frac{8 \log(2RK/\nu)}{p+1} + 2(p+1)^2 \norm{\nabla  F_{\mask}(\bw_e)}^2. 
    \end{align*}
    For $ \E \norm{ \bxi_{e,p}}^2$, we have
    \begin{align*}
        \E \norm{ \bxi_{e,p}}^2 &= \E \norm{ \eta \frac{1}{N}\sum_{i=1}^N  \sum_{k=0}^{K-1}\pare{\mask_i^{\sigma_e(p)}\odot \nabla f_i(\mask_i^{\sigma_e(p)}\odot \bw_{e,p,k} ) - \mask_i^{\sigma_e(p)}\odot \nabla f_i(\mask_i^{\sigma_e(p)}\odot \bw_{e,p,k};\xi_{e,p,k}^i)} }^2 \\
        &\leq \eta^2 K^2 \frac{\delta^2}{N}.
    \end{align*}
    Putting piece together yields:
     \begin{align}
       \E \norm{\bw_{e,r} - \bw_{e}}^2    
            &\leq 3\eta^2 K^2\pare{16r G  { \log(2/\nu)}  + 2r^2 \norm{\nabla  F_{\mask}(\bw_e)}^2  } \nonumber\\
            &\quad + 3\eta^4 K^4 r \sum_{p=0}^{r-2}  9  L^2 \pare{16(p+1)   G  {  \log(2RK/\nu)}  + 2(p+1)^2 \norm{\nabla  F_{\mask}(\bw_e)}^2} \nonumber \\
            &\quad + 18 \eta^2 r \sum_{p=0}^{r-1}  KL^2\sum_{k=0}^{K-1} \frac{1}{N}\sum_{i=1}^N  \E \norm{\bw^i_{e,p,k} - \bw_{e,p}   }^2+ 18 r^2 \eta^2 K^2 \frac{\delta^2}{N}\nonumber\\
             &\leq 3\eta^2 K^2\pare{16r G  { \log(2RK/\nu)}  + 2r^2 \norm{\nabla  F_{\mask}(\bw_e)}^2  }\nonumber \\
            &\quad + 27\eta^4 K^4 r      L^2 \pare{16 r^2   G  {  \log(2RK/\nu)}  +  \frac{ 2r^3}{3} \norm{\nabla  F_{\mask}(\bw_e)}^2}  \nonumber\\
            &\quad + 18 \eta^2 r \sum_{p=0}^{r-1}  KL^2\sum_{k=0}^{K-1} \frac{1}{N}\sum_{i=1}^N  \E \norm{\bw^i_{e,p,k} - \bw_{e,p}   }^2+ 18 r^2 \eta^2 K^2 \frac{\delta^2}{N} \nonumber\\
            & =  \pare{48r\eta^2 K^2+ 432 r^3 \eta^4 K^4        L^2 } G  { \log(2RK/\nu)}  + \pare{ 6r^2\eta^2 K^2+  18 r^4  \eta^4 K^4    L^2  }\norm{\nabla  F_{\mask}(\bw_e)}^2  \nonumber\\ 
            &\quad + 18\eta^2 r \sum_{p=0}^{r-1}  KL^2\sum_{k=0}^{K-1} \frac{1}{N}\sum_{i=1}^N  \E \norm{\bw^i_{e,p,k} - \bw_{e,p}   }^2+ 18 r^2 \eta^2 K^2 \frac{\delta^2}{N}. \label{eq:w_er - w_e}
    \end{align}
    Plugging above bound back to (\ref{eq:local model deviation 0}) yields:
       \begin{align*}
      &   \E \norm{ \bw_{e,r} - \bw_{e,r,k+1}^i}^2    \\
         &  \leq \pare{1+\frac{2}{K-1}}  \E \norm{\bw_{e,r} - \bw_{e,r,k }^i }^2    + 4\eta^2 K     \E \norm{  \nabla f_i^{\sigma_e(r)}(  \bw_{e }    )}^2+ \eta^2 K \delta^2   \\
         & \quad +4\eta^2 K L^2  \pare{ \pare{48r\eta^2 K^2+432 r^3 \eta^4 K^4        L^2 } G  { \log(2RK/\nu)}  + \pare{ 6r^2\eta^2 K^2+  18 r^4  \eta^4 K^4    L^2  }\norm{\nabla  F_{\mask}(\bw_e)}^2  }\\
            & \quad +4\eta^2 K L^2  \pare{  18\eta^2 r \sum_{p=0}^{r-1}  KL^2\sum_{k=0}^{K-1} \frac{1}{N}\sum_{i=1}^N  \E \norm{\bw^i_{e,p,k} - \bw_{e,p}   }^2+ 18 r^2 \eta^2 K^2 \frac{\delta^2}{N}}.
          \end{align*}
          Unrolling the recursion from $k$ to $0$ yields
    \begin{align*}
        &  \E \norm{ \bw_{e,r} - \bw_{e,r,k+1}^i}^2    
          \leq   \sum_{k'=0}^k \pare{1+\frac{2}{K-1}}^{k'} \pare{4\eta^2 K     \E \norm{  \nabla f_i^{\sigma_e(r)}(  \bw_{e }    )}^2+ \eta^2 K \delta^2}   \\
         &   + \sum_{k'=0}^k \pare{1+\frac{2}{K-1}}^{k'}  
 4\eta^2 K L^2  \pare{ \pare{48r\eta^2 K^2+432 r^3 \eta^4 K^4        L^2 } G  { \log(2/\nu)}  + \pare{ 6r^2\eta^2 K^2+  18 r^4  \eta^4 K^4    L^2  }\norm{\nabla  F_{\mask}(\bw_e)}^2  }\\
            &  +\sum_{k'=0}^k \pare{1+\frac{2}{K-1}}^{k'}   4\eta^2 K L^2  \pare{  18\eta^2 r \sum_{p=0}^{r-1}  KL^2\sum_{k=0}^{K-1} \frac{1}{N}\sum_{i=1}^N  \E \norm{\bw^i_{e,p,k} - \bw_{e,p}   }^2+ 18 r^2 \eta^2 K^2 \frac{\delta^2}{N}}\\
            &\leq    12\eta^2 K^2     \E \norm{  \nabla f_i^{\sigma_e(r)}(  \bw_{e }    )}^2+ 3\eta^2 K^2 \delta^2   \\
         &   +    12\eta^2 K^2  
  L^2  \pare{ \pare{48R\eta^2 K^2+432 R^3 \eta^4 K^4        L^2 } G  { \log(2/\nu)}  + \pare{ 6R^2\eta^2 K^2+ 18R^4  \eta^4 K^4    L^2  }\norm{\nabla  F_{\mask}(\bw_e)}^2  }\\
            &   +   12\eta^2 K^2     L^2  \pare{  18\eta^2 R  \sum_{p=0}^{R-1}  KL^2\sum_{k=0}^{K-1} \frac{1}{N}\sum_{i=1}^N  \E \norm{\bw^i_{e,p,k} - \bw_{e,p}   }^2+ 18 R^2 \eta^2 K^2 \frac{\delta^2}{N}}.
          \end{align*}

          We further sum above inequality over $k=0$ to $K-1$ and get
  \begin{align*}
      \sum_{r=0}^{R-1}\sum_{k=0}^{K-1} & \frac{1}{N}\sum_{i=1}^N \E \norm{ \bw_{e,r} - \bw_{e,r,k+1}^i}^2 
            \leq    12\eta^2 K^3   \sum_{r=0}^{R-1} \frac{1}{N}\sum_{i=1}^N  \E \norm{  \nabla f_i^{\sigma_e(r)}(  \bw_{e }    )}^2+ 3\eta^2 R K^3 \delta^2   \\
         &  +    12\eta^2 R K^3  
  L^2  \pare{ \pare{48R\eta^2 K^2+72 R^3 \eta^4 K^4        L^2 } G  { \log(2/\nu)}  + \pare{ 6R^2\eta^2 K^2+  6R^4  \eta^4 K^4    L^2  }\norm{\nabla  F_{\mask}(\bw_e)}^2  }\\
            &  +   12\eta^2 R K^3     L^2  \pare{  18\eta^2 R KL^2  \sum_{p=0}^{R-1} \sum_{k=0}^{K-1} \frac{1}{N}\sum_{i=1}^N  \E \norm{\bw^i_{e,p,k} - \bw_{e,p}   }^2+ 18 R^2 \eta^2 K^2 \frac{\delta^2}{N}}.
          \end{align*}
          Re-arranging the terms yields:
\begin{align*}
  &\underbrace{\pare{1 -  216\eta^4 R^2 K^4     L^4 }}_{\geq \frac{1}{2}}    \sum_{r=0}^{R-1}\sum_{k=0}^{K-1}  \frac{1}{N}\sum_{i=1}^N \E \norm{ \bw_{e,r} - \bw_{e,r,k+1}^i}^2  \leq    12\eta^2 K^3   \sum_{r=0}^{R-1} \frac{1}{N}\sum_{i=1}^N  \E \norm{  \nabla f_i^{\sigma_e(r)}(  \bw_{e }    )}^2+ 3\eta^2 R K^3 \delta^2   \\
         &  +    12\eta^2 R K^3   L^2  \pare{ \pare{48R\eta^2 K^2+432 R^3 \eta^4 K^4        L^2 } G  { \log(2RK/\nu)}  + \pare{ 6R^2\eta^2 K^2+  18 R^4  \eta^4 K^4    L^2  }\norm{\nabla  F_{\mask}(\bw_e)}^2  }\\
            &  +   12\eta^2 R K^3     L^2  \pare{   18 R^2 \eta^2 K^2 \frac{\delta^2}{N}}.
\end{align*} 
          Hence we conclude that
              \begin{align*}
   \sum_{r=0}^{R-1}\sum_{k=0}^{K-1}  \frac{1}{N}\sum_{i=1}^N \E \norm{ \bw_{e,r} - \bw_{e,r,k+1}^i}^2 
           & \leq    24\eta^2 K^3   \sum_{r=0}^{R-1} \frac{1}{N}\sum_{i=1}^N  \E \norm{  \nabla f_i^{\sigma_e(r)}(  \bw_{e }    )}^2+ 6\eta^2 R K^3 \delta^2   \\
         &  +   12    \pare{48R^2\eta^4 K^5 L^2+432 R^4 \eta^6 K^7        L^4 } G  { \log(2RK/\nu)}  \\
         & +  12 \pare{ 6R^3\eta^4 K^5 L^2+  18R^5  \eta^6 K^7    L^4  }\norm{\nabla  F_{\mask}(\bw_e)}^2 \\
            &  +   216     R^3 \eta^4 K^5 L^2 \frac{\delta^2}{N}.
          \end{align*} 
Finally, recall the definition of gradient dissimilarity and get
\begin{align*}
     \sum_{r=0}^{R-1} \frac{1}{N}\sum_{i=1}^N  \E \norm{  \nabla f_i^{\sigma_e(r)}(  \bw_{e }    )}^2 \leq  2\sum_{r=0}^{R-1} \frac{1}{N}\sum_{i=1}^N  \E \norm{  \nabla f_i^{\sigma_e(r)}(  \bw_{e }    ) - \nabla F(  \bw_{e }    ) }^2 +   2R \E \norm{  \nabla F(  \bw_{e }    )}^2,
\end{align*}
which concludes the proof.
The proof of second statement follows the same reasoning.
    \end{proof}
\end{lemma}

\subsubsection{Proof of Theorem~\ref{thm:rolling mask}} 

In the convergence analysis of Theorem~\ref{thm:rolling mask}, we account for the fact that the full model is partitioned into $R$ sub-models. At the start of each epoch, the server shuffles these sub-models and assigns them sequentially to clients. This introduces complexity in analysis due to the interaction between both the model drift caused by partial training on sub-models and the impact of permutation-based assignments on convergence. Our technical contribution is to  tackle these challenges jointly to establish the convergence, which could be interesting by its own.
Here we briefly illustrate our proof strategy. Consider an alternative objective induced by a mask configuration $\mask$:
\begin{align}
    F_{\mask}(\bw) = \frac{1}{N}\sum\nolimits_{i=1}^N  \frac{1}{R}\sum\nolimits_{j=1}^R  f_i(\mask_i^j\odot\bw) \label{eq: obj rolling}
\end{align}
and define $\bw^*(\mask):= \arg\min_{\bw\in\cW}F_\mask(\bw)$ as the optimal model given a masking configuration $\mask= \sbr{\sbr{\mask_i^1,\ldots,\mask_i^R}}_{i=1}^N \in \cbr{0,1}^{dNR}$. Apparently, when $\mask = \mathbf{1}$, $F_{\mask}(\bw)$ becomes original objective $F(\bw)$.
Our proof relies on a key Lipschitz property of $\bw^*(\mask)$. That is, if $f_i(\bw)$ is strongly convex and with bounded gradient, then
\begin{align*}
    \norm{\bw^*(\mask) -\bw^*(\mathbf{1})} \leq c\cdot \norm{\mask - \mathbf{1}},
\end{align*}
for some constant $c$. As a result, we can decompose the objective value into (1) convergence error to ${\bw}^*(\mask)$ and (2) residual error due to masking:
 \begin{align*}
      F(  \hat{\bw}) -  F(  \bw^* ) & \leq   2L\|\tilde{\bw}  - {\bw}^*(\mask)\|^2 + \pare{  \frac{ 2L}{\mu  }    + \frac{4 }{ L}  }{\frac{2G^2 + 2W^2L^2}{NR}}\| \mask  - \mathbf{1} \|^2 .
    \end{align*}
Then the heart of the proof is to prove the convergence of Algorithm~\ref{algorithm: Rolling Masked FedAvg}  to ${\bw}^*(\mask)$. Algorithm~\ref{algorithm: Rolling Masked FedAvg} performs a variant of Local SGD where each client shuffles its local component function $f_i(\mask_i^1\odot\bw),\ldots,f_i(\mask_i^R\odot \bw)$, and conduct local updates on one of them for $K$ steps during one communication round. 

Having outlined the key ideas, we turn to  rigorously establish the convergence rate.
\begin{proof}
Due to updating rule we know
\begin{align*}
   \E\norm{\bw_{e+1} - \bw^*}^2 & =  \E\norm{\cP_{\cW}(\bw_{e,R}) - \bw^*}^2\\
   & \leq  \E\norm{ \bw_{e,R} -\bw_{e} + \bw_{e} - \bw^*}^2\\
   & \leq \E\norm{ \bw_e - \bw^*}^2 + 2\inprod{ \bw_e - \bw^* }{-\eta K   \nabla  F_{\mask}(  \bw_{e } ) }\\
   &\quad + 2\inprod{ \bw_e - \bw^* }{ \eta^2 K^2 \sum_{r=0}^{R-2} \mA_{r+2}  \frac{1}{N}\sum_{i=1}^N   \mM_i^{\sigma_e(r)} \mH^i_{e,r} \mM_i^{\sigma_e(r)}   \sum_{j=0}^r \frac{1}{N}\sum_{i=1}^N     \nabla f_i^{\sigma_e(j)}(  \bw_{e } )}\\
   & \quad + 2\inprod{\bw_e - \bw^*}{ \sum_{r=0}^{R-1} \mA_r \pare{
            \eta \frac{1}{N}\sum_{i=1}^N\sum_{k=0}^{K-1}\br^i_{e,r,k}  + \bxi_{e,r}} } + \E \norm{\bw_e - \bw_{e,R}}^2,
\end{align*}
where we plug in Lemma~\ref{lem:epoch recursion}.

Applying Cauchy-Schwartz inequality and taking the expectation over randomness of $\xi$ yields:
\begin{align*}
   \E\norm{\bw_{e+1} - \bw^*}^2  
   & \leq \E\norm{ \bw_e - \bw^*}^2 + 2\inprod{ \bw_e - \bw^* }{-\eta R K   \nabla  F_{\mask}(  \bw_{e } ) }\\
   &\quad + 2\norm{ \bw_e - \bw^* }\norm{ \eta^2 K^2 \sum_{r=0}^{R-2} \mA_{r+2}  \frac{1}{N}\sum_{i=1}^N   \mM_i^{\sigma_e(r)} \mH^i_{e,r} \mM_i^{\sigma_e(r)}   \sum_{j=0}^r \frac{1}{N}\sum_{i=1}^N     \nabla f_i^{\sigma_e(j)}(  \bw_{e } )}\\
   & \quad + 2\norm{\bw_e - \bw^*}\norm{ \sum_{r=0}^{R-1} \mA_r \pare{
            \eta \frac{1}{N}\sum_{i=1}^N\sum_{k=0}^{K-1}\br^i_{e,r,k}  + \bxi_{e,r}} } + \E \norm{\bw_e - \bw_{e,R}}^2\\
            & \leq (1-\mu \eta RK) \E\norm{ \bw_e - \bw^*}^2  - \underbrace{\eta RK \pare{    F_{\mask}(  \bw_{e } ) -  F_{\mask}(  \tilde\bw^* ) }}_{\geq 0}\\
   &\quad + \underbrace{2\eta^2 K^2 \norm{ \bw_e - \bw^* }\pare{\sum_{r=0}^{R-2}(1+\eta K L)^RL\norm{   \sum_{j=0}^r \frac{1}{N}\sum_{i=1}^N     \nabla f_i^{\sigma_e(j)}(  \bw_{e } )}}}_{T_1}\\
   & \quad + \underbrace{2\norm{\bw_e - \bw^*}\pare{\sum_{r=0}^{R-1} (1+\eta K L)^R \norm{   \eta \frac{1}{N}\sum_{i=1}^N\sum_{k=0}^{K-1}\br^i_{e,r,k}  } }}_{T_2} + \E \norm{\bw_e - \bw_{e,R}}^2.
\end{align*}

Notice that we choose $\eta$ such that $\eta K L \leq \frac{1}{R}$, so we know $(1+\eta K L)^R \leq e \leq 3$.

To bound $T_1$, we again use Hoeffding-Serfling inequality:
    \begin{align*}
    \sum_{r=0}^{R-2}3L  \norm{\sum_{j=0}^{r } \frac{1}{N}\sum_{i=1}^N \nabla f_i^{\sigma_e(j)}(  \bw_{e } )}  & \leq   \sum_{r=0}^{R-2}3L \pare{ \norm{\nabla  F_{\mask}(\bw_e) - \frac{1}{r} \sum_{j=0}^{r} \frac{1}{N}\sum_{i=1}^N \nabla f_i^{\sigma_e(j)}(  \bw_{e } )} +  r  \norm{\nabla  F_{\mask}(\bw_e)} }\\
      &\leq   \sum_{r=0}^{R-2}3L \pare{  \sqrt{G  \frac{8 \log(2RK/\nu)}{r}} + r \norm{\nabla  F_{\mask}(\bw_e)}}\\
    & \leq  3L \pare{ \frac{2R^{3/2}}{3} \sqrt{G   {8 \log(2RK/\nu)} } + R^2 \norm{\nabla  F_{\mask}(\bw_e)}}.
    \end{align*}
    Hence we can bound $T_1$ as:
    \begin{align*}
        T_1 &\leq  2\eta^2 K^2\norm{\bw_e - \tilde\bw^* }\pare{3L \pare{ \frac{2R^{3/2}}{3} \sqrt{G   {8 \log(2RK/\nu)} } + R^2 \norm{\nabla  F_{\mask}(\bw_e)}}}\\
       & =   4L \eta^2 K^2\norm{\bw_e - \tilde\bw^* }   R^{3/2}   \sqrt{8G    \log(2RK/\nu)  } + 3L\eta^2 R^2 K^2 \norm{\nabla  F_{\mask}(\bw_e)}^2  \\
       & \leq  4L \eta^2 K^2\pare{\frac{1}{4\sqrt{\eta L K}}R^{1/2} \norm{\bw_e -\tilde\bw^*} \cdot 4 \sqrt{\eta L K} R \sqrt{G   {8 \log(2RK/\nu)} }} + 6L\eta^2 R^2 K^2 \norm{\bw_e - \tilde\bw^* }\norm{\nabla  F_{\mask}(\bw_e)}   \\
       & \leq   4 L \eta^2 K^2\pare{\frac{1}{16 \eta L K}R  \norm{\bw_e - \bw^* }^2 +  4\eta L R^2  K G    \log(2RK/\nu)}    + 6L\eta^2 R^2 K^2 \norm{\bw_e - \tilde\bw^* }\norm{\nabla  F_{\mask}(\bw_e)}   \\
       & =  \frac{1}{4   }\eta R K   \norm{\bw_e - \tilde\bw^* }^2 +  16 \eta^3 K^3 L^2 R^2   G    \log(2RK/\nu)     +3L \eta^2 R^2 K^2 \norm{\nabla  F_{\mask}(\bw_e)}^2+3L \eta^2 R^2 K^2 \norm{\bw_e - \tilde\bw^*}^2.
    \end{align*}
    To bound $T_2$ we notice
    \begin{align*}
     T_2 &\leq   6\eta\pare{\frac{1}{4}\sqrt{RK}\norm{ \bw_e - \tilde\bw^*}\cdot  4\frac{1}{\sqrt{RK}}\pare{\sum_{r=0}^{R-1} \sum_{k=0}^{K-1}L\frac{1}{N}\sum_{i=1}^N\norm{   \bw^i_{e,r,k}  - \bw_{e,r} } }}\\
     &\leq  \frac{6}{32}\eta   {RK}\norm{\bw_e - \tilde\bw^* }^2 +  \frac{48\eta}{ {RK}}\pare{\sum_{r=0}^{R-1} \sum_{k=0}^{K-1}L\frac{1}{N}\sum_{i=1}^N\norm{   \bw^i_{e,r,k}  - \bw_{e,r} } }^2  \\
     &\leq  \frac{3}{8}\eta   {RK}\norm{ \bw_e - \tilde\bw^* }^2 +   {48\eta }  \sum_{r=0}^{R-1} \sum_{k=0}^{K-1}L\frac{1}{N}\sum_{i=1}^N\norm{   \bw^i_{e,r,k}  - \bw_{e,r} }^2 .
    \end{align*}
    We plug in Lemma~\ref{lem:local model deviation} and get
    \begin{align*}
      &  48\eta L  \sum_{r=0}^{R-1} \sum_{k=0}^{K-1} \frac{1}{N}\sum_{i=1}^N\norm{   \bw^i_{e,r,k}  - \bw_{e,r} }^2 \\
      &\leq         48^2\eta^3 L R K^3   \zeta +  48\eta L \pare{ 12 \pare{ 6R^3\eta^4 K^5 L^2+  18R^5  \eta^6 K^7    L^4  } 
 + 48\eta^2 R K^3}\E\norm{\nabla  F_{\mask}(\bw_e)}^2 \\
         &  \quad +   576 \eta L   \pare{48R^2\eta^4 K^5 L^2+432 R^4 \eta^6 K^7        L^4 } G  { \log(2RK/\nu)}  \\ 
            &  \quad +  10368    R^3 \eta^5 K^5 L^3 \frac{\delta^2}{N}+   288  L  \eta^3 R K^3 \delta^2 .
    \end{align*} 
    Putting pieces together yields:
\begin{align*}
   \E\norm{\bw_{e+1} - \bw^*}^2   
            & \leq \pare{1- \frac{3}{8}\mu \eta RK - 3L\eta^2 R^2 K^2} \E\norm{ \bw_e - \bw^*}^2  - \eta RK \pare{    F_{\mask}(  \bw_{e } ) -  F_{\mask}(  \tilde\bw^* ) } \\
   &\quad  +  \pare{16 \eta^3 K^3 L^2 R^2 +576 \eta L   \pare{48R^2\eta^4 K^5 L^2+432 R^4 \eta^6 K^7        L^4 }  } G    \log(2RK/\nu)  \\
   & \quad   + \E \norm{\bw_e - \bw_{e,R}}^2+ 48^2\eta^3 L R K^3   \zeta \\
   &\quad +   \pare{   3456 R^3\eta^5 K^5 L^3+  10368 R^5  \eta^7 K^7    L^5  
 + 2304  L \eta^3 R K^3+3L \eta^2 R^2 K^2}\E\norm{\nabla  F_{\mask}(\bw_e)}^2 \\ 
            & \quad +  10368    R^3 \eta^5 K^5 L^3 \frac{\delta^2}{N}+   288  L  \eta^3 R K^3 \delta^2.
\end{align*} 
Due to our choice of $\eta = \frac{2\log(T^2)}{\mu RKT}$ and $T \geq 128 \kappa L \log T$ we know $\eta RK \leq  \frac{2\log(T^2)}{\mu RK 512 \kappa \log T^2}   RK = \frac{1}{256 L^2}   $
\begin{align*}
     3\eta^2 R^2 K^2 L \leq \frac{1}{256L } \eta R K \\
    2304  L    \eta^3 R K^3 \leq   \frac{2304}{256^2L}\eta RK\\
    3456 R^3 \eta^5 K^5 L^3 \leq \frac{3456}{256^4L} \eta RK\\
    10368 R^5 \eta^7 K^7 L^5 \leq \frac{10368}{256^6L} \eta RK.
\end{align*}
Hence we have:
\begin{align*}
   \E\norm{\bw_{e+1} - \bw^*}^2   
            & \leq \pare{1- \frac{3}{8}\mu \eta RK - 3L\eta^2 R^2 K^2} \E\norm{ \bw_e - \bw^*}^2  - \eta RK \pare{    F_{\mask}(  \bw_{e } ) -  F_{\mask}(  \tilde\bw^* ) } \\
   &\quad  +  \pare{16 \eta^3 K^3 L^2 R^2 +576 \eta L   \pare{48R^2\eta^4 K^5 L^2+432 R^4 \eta^6 K^7        L^4 }  } G    \log(2RK/\nu)  \\
   & \quad   + \E \norm{\bw_e - \bw_{e,R}}^2+ 48^2\eta^3 L R K^3   \zeta \\
   &\quad +   \frac{\eta RK}{16L}\E\norm{\nabla  F_{\mask}(\bw_e)}^2 \\ 
            & \quad +  10368    R^3 \eta^5 K^5 L^3 \frac{\delta^2}{N}+   288  L  \eta^3 R K^3 \delta^2\\
            & \leq \pare{1- \frac{3}{8}\mu \eta RK - 3L\eta^2 R^2 K^2} \E\norm{ \bw_e - \bw^*}^2  - \frac{7}{8}\eta RK \pare{    F_{\mask}(  \bw_{e } ) -  F_{\mask}(  \tilde\bw^* ) } \\
   &\quad  +  \pare{16 \eta^3 K^3 L^2 R^2 +576 \eta L   \pare{48R^2\eta^4 K^5 L^2+432 R^4 \eta^6 K^7        L^4 }  } G    \log(2RK/\nu)  \\
   & \quad   + \E \norm{\bw_e - \bw_{e,R}}^2+ 48^2\eta^3 L R K^3   \zeta  +  10368    R^3 \eta^5 K^5 L^3 \frac{\delta^2}{N}+   288  L  \eta^3 R K^3 \delta^2,  
\end{align*}
where at last step we use the $L$ smoothness of $ F_{\mask}$ such that $\norm{\nabla  F_{\mask}(\bw_e)}^2 \leq 2 L  \pare{    F_{\mask}(  \bw_{e } ) -  F_{\mask}(  \tilde\bw^* ) }$.

It remains to bound $ \norm{\bw_{e+1} - \bw_e}^2$. We again evoke Lemma~\ref{lem:epoch recursion}:
\begin{align*}
    \E\norm{\bw_{e,R} - \bw_e}^2 &\leq   3\norm{  \eta R K   \nabla  F_{\mask}(  \bw_{e } ) }^2\\
            &  + 3 \eta^4 K^4 R \sum_{r=0}^{R-2} \prod_{r'=R-1}^{r+2} (1+\eta K L)^{2r'}   L^2 \E\norm{  \sum_{j=0}^r \frac{1}{N}\sum_{i=1}^N     \nabla f_i^{\sigma_e(j)}(  \bw_{e } )   }^2\\
            &\quad +3  \E\norm{ \sum_{r=0}^{R-1} \mA_r \pare{\eta\frac{1}{N}\sum_{i=1}^N\sum_{k=0}^{K-1}\br^i_{e,r,k}  + \bxi_{e,r}}}^2\\
            &\leq    3\E\norm{  \eta R K   \nabla  F_{\mask}(  \bw_{e } ) }^2  +   {27L^2} \eta^4 K^4 R \sum_{r=0}^{R-2}     \E\norm{  \sum_{j=0}^r \frac{1}{N}\sum_{i=1}^N     \nabla f_i^{\sigma_e(j)}(  \bw_{e } )   }^2 \\
            &\quad + {27 }  R\sum_{r=0}^{R-1} \E\norm{   \eta\frac{1}{N}\sum_{i=1}^N\sum_{k=0}^{K-1}\br^i_{e,r,k}  }^2 + {27L}  R\sum_{r=0}^{R-1}\E \norm{   \bxi_{e,r} }^2\\ 
             &\leq  3\E\norm{  \eta R K   \nabla  F_{\mask}(  \bw_{e } ) }^2  +   {27L^2}  \eta^4 K^4 R \underbrace{\sum_{r=0}^{R-2}    \E \norm{  \sum_{j=0}^r \frac{1}{N}\sum_{i=1}^N     \nabla f_i^{\sigma_e(j)}(  \bw_{e } )   }^2 }_{T_1}\\
            &\quad +  27 \eta^2 RKL^2 \underbrace{\sum_{r=0}^{R-1}   \frac{1}{N}\sum_{i=1}^N\sum_{k=0}^{K-1}\E\norm{ \bw^i_{e,r,k} - \bw^i_{e,r}  }^2}_{T_2} + {27 }    \eta^2 R^2 K^2 \frac{\delta^2}{N}. 
\end{align*}
For $T_1$, we know that 
\begin{align*}
    T_1 &\leq 2\sum_{r=0}^{R-2}   r^2  \E \norm{      \nabla  F_{\mask}(  \bw_{e } )   }^2 + 2\sum_{r=0}^{R-2}  r^2  \E \norm{  \frac{1}{r}\sum_{j=0}^r \frac{1}{N}\sum_{i=1}^N     \nabla f_i^{\sigma_e(j)}(  \bw_{e } ) - \nabla  F_{\mask}(  \bw_{e } )   }^2\\
   & \leq 2 \frac{R^3}{3}  \E \norm{      \nabla  F_{\mask}(  \bw_{e } )   }^2 + 2\sum_{r=0}^{R-2}  r^2  8G\frac{\log(2RK/\nu)}{r}\\
   & \leq 2   R^3  \E \norm{      \nabla  F_{\mask}(  \bw_{e } )   }^2 + 16  R^2   G  \log(2RK/\nu).
\end{align*}
with probability at least $1-\nu$.

For $T_2$, we again evoke Lemma~\ref{lem:local model deviation}:
\begin{align*}
  T_2 \leq &48\eta^2 R K^3   \zeta + \pare{ 12 \pare{ 6R^3\eta^4 K^5 L^2+  18R^5  \eta^6 K^7    L^4  } 
 + 48\eta^2 R K^3}\E\norm{\nabla  F_{\mask}(\bw_e)}^2 \\
         &  +   12    \pare{48R^2\eta^4 K^5 L^2+432 R^4 \eta^6 K^7        L^4 } G  { \log(2RK/\nu)}  \\ 
            &  +   216     R^3 \eta^4 K^5 L^2 \frac{\delta^2}{N}+ 6\eta^2 R K^3 \delta^2 .
\end{align*}
Putting pieces together yields:
\begin{align*}
  &   \E\norm{\bw_{e+1} - \bw_e}^2  \\
  &\leq  3\E\norm{  \eta R K   \nabla  F_{\mask}(  \bw_{e } ) }^2  +   {27L^2}  \eta^4 K^4 R \pare{ 2   R^3  \E \norm{      \nabla  F_{\mask}(  \bw_{e } )   }^2 + 16  R^2   G  \log(2RK/\nu) } \\
            &\quad +  27 \eta^2 RKL^2 \pare{48\eta^2 R K^3   \zeta + \pare{ 12 \pare{ 6R^3\eta^4 K^5 L^2+  18R^5  \eta^6 K^7    L^4  } 
 + 48\eta^2 R K^3}\E\norm{\nabla  F_{\mask}(\bw_e)}^2} \\
         & \quad +  27 \eta^2 RKL^2 \pare{ 12    \pare{48R^2\eta^4 K^5 L^2+432 R^4 \eta^6 K^7        L^4 } G  { \log(2RK/\nu)} } \\ 
            & \quad +  27 \eta^2 RKL^2 \pare{ 216     R^3 \eta^4 K^5 L^2 \frac{\delta^2}{N}+ 6\eta^2 R K^3 \delta^2 }  + {27 }    \eta^2 R^2 K^2 \frac{\delta^2}{N}\\ 
            & = \pare{3\eta^2 R^2 K^2 + 54L^2  \eta^4 R^4 K^4 +        1944 R^4\eta^6 K^6 L^4+  5832R^6  \eta^8 K^8    L^6    
 + 1296\eta^4 R^2 K^4L^2 } \E\norm{    \nabla  F_{\mask}(  \bw_{e } ) }^2\\
 & \quad +  \pare{27 \eta^2 RKL^2  12    \pare{48R^2\eta^4 K^5 L^2+432 R^4 \eta^6 K^7        L^4 } + 432 L^2 \eta^4 R^3 K^4 } G  { \log(2RK/\nu)}\\ 
       & \quad +  27 \eta^2 RKL^2 \pare{ 216     R^3 \eta^4 K^5 L^2 \frac{\delta^2}{N}+ 6\eta^2 R K^3 \delta^2 }  + {27 }    \eta^2 R^2 K^2 \frac{\delta^2}{N}  + 1296 \eta^4 R^2 K^4 L^2         \zeta.
\end{align*}
Due to our choice of $\eta = \frac{4\log(T^2)}{\mu RKT}$ and $T \geq 512 \kappa^2 \log T$ we know $\eta RK \leq  \frac{2\log(T^2)}{\mu RK 512 \kappa^2 \log T^2}   RK = \frac{1}{256 \kappa L }   $
\begin{align*}
     3\eta^2 R^2 K^2  \leq \frac{3}{256L } \eta R K \leq \frac{1}{16L}\eta RK \\
   54L^2  \eta^4 R^4 K^4 \leq   \frac{54}{256^3L}\eta RK\leq \frac{1}{16L}\eta RK\\
  1944 R^4\eta^6 K^6 L^4\leq \frac{1944}{256^5L} \eta RK\leq \frac{1}{16L}\eta RK\\
   5832R^6  \eta^8 K^8    L^6 \leq \frac{5832}{256^7L} \eta RK\leq \frac{1}{16L}\eta RK\\
     1296\eta^4 R^2 K^4L^2 \leq \frac{1296}{256^3 L} \eta RK \leq \frac{1}{16L}\eta RK.
\end{align*}
Hence we know 
\begin{align}
     \E\norm{\bw_{e+1} - \bw_e}^2   
            & \leq \frac{5}{16L}\eta RK \E\norm{    \nabla  F_{\mask}(  \bw_{e } ) }^2\nonumber\\
 & \quad +  \pare{324 \eta^2 RKL^2    \pare{48R^2\eta^4 K^5 L^2+432 R^4 \eta^6 K^7        L^4 } + 432 L^2 \eta^4 R^3 K^4 } G  { \log(2RK/\nu)}\nonumber \\ 
       & \quad +  27 \eta^2 RKL^2 \pare{ 216     R^3 \eta^4 K^5 L^2 \frac{\delta^2}{N}+ 6\eta^2 R K^3 \delta^2 }  + {27 }    \eta^2 R^2 K^2 \frac{\delta^2}{N}\nonumber\\ 
       & \quad + 1296 \eta^4 R^2 K^4 L^2         \zeta^2. \label{eq: iterate difference bound }
\end{align}
Putting pieces together yields:
\begin{align*}
   \E\norm{\bw_{e+1} - \bw^*}^2   
            & \leq \pare{1- \frac{3}{8}\mu \eta RK + 3L\eta^2 R^2 K^2} \E\norm{ \bw_e - \bw^*}^2  - \frac{1}{4}\eta RK \underbrace{\pare{    F_{\mask}(  \bw_{e } ) -  F_{\mask}(  \tilde\bw^* ) } }_{\geq 0}\\
   &\quad  +  \pare{16 \eta^3 K^3 L^2 R^2 +576 \eta L   \pare{48R^2\eta^4 K^5 L^2+432 R^4 \eta^6 K^7        L^4 }  } G    \log(2RK/\nu)  \\
   & \quad    + 48^2\eta^3 L R K^3   \zeta  +  10368    R^3 \eta^5 K^5 L^3 \frac{\delta^2}{N}+   288  L  \eta^3 R K^3 \delta^2 + {27 }    \eta^2 R^2 K^2 \frac{\delta^2}{N}\\
   & \quad +  \pare{324 \eta^2 RKL^2    \pare{48R^2\eta^4 K^5 L^2+432 R^4 \eta^6 K^7        L^4 } + 432 L^2 \eta^4 R^3 K^4 } G  { \log(2RK/\nu)}\\ 
       & \quad +  27 \eta^2 RKL^2 \pare{ 216     R^3 \eta^4 K^5 L^2 \frac{\delta^2}{N}+ 6\eta^2 R K^3 \delta^2 }  \\ 
       & \quad + 1296 \eta^4 R^2 K^4 L^2         \zeta^2.
\end{align*}
 Due to our choice of $\eta$, we know $ 3L\eta^2 R^2 K^2 \leq 3L \eta RK \frac{\mu}{256 L^2}\leq \frac{1}{8  } \mu \eta RK  $. We unroll the recursion from $T$ to $0$ to get:
\begin{align*}
   \E\norm{\bw_{e+1} - \bw^*}^2   
            & \leq \pare{1- \frac{1}{4}\mu \eta RK }^T \E\norm{ \bw_0 - \bw^*}^2  \\
   &\quad  +  4\pare{16 \eta^2 K^2 \kappa L  R  +576  \kappa  \pare{48R \eta^4 K^4 L^2+432 R^3 \eta^6 K^6        L^4 }  } G    \log(2RK/\nu)  \\
   & \quad    +4 \cdot 48^2\eta^2 \kappa   K^2   \zeta  +  4 \cdot 10368    R^2 \eta^4 K^4 \kappa L^2 \frac{\delta^2}{N}+  4 \cdot 288  \kappa \eta^2  K^2 \delta^2 +4 \cdot {27 }    \eta  R  K  \frac{\delta^2}{\mu N}\\
   & \quad +  4\pare{324 \eta  \kappa L     \pare{48R^2\eta^4 K^5 L^2+432 R^4 \eta^6 K^7        L^4 } + 432 \kappa L  \eta^3 R^2 K^3 } G  { \log(2RK/\nu)}\\ 
       & \quad + 4 \cdot 27 \eta   \kappa L \pare{ 216     R^3 \eta^4 K^5 L^2 \frac{\delta^2}{N}+ 6\eta^2 R K^3 \delta^2 }  \\ 
       & \quad + 4 \cdot 1296 \eta^3 R  K^3 \kappa L         \zeta.
\end{align*}

Plugging $\eta = \frac{4\log T^2}{\mu TRK}$ yields:
   \begin{align*}
              \E\norm{\bw_{T} - \bw^*}^2 
            &\leq O\pare{ \frac{ \E\norm{ \bw_0 - \bw^*}^2 }{T^2} } + O\pare{\frac{\kappa^2        \log(2RK/\nu)  }{\mu T^2 R  }  } +  O\pare{ \frac{\kappa   \zeta \log^2(T )}{\mu^2 T^2   R^2}    }    +  O\pare{      \frac{  \delta^2 \log(T)}{\mu^2 T N}}. 
\end{align*}
\end{proof}

\subsection{Proof of Convergence in Nonconvex Setting (Theorem~\ref{thm:rolling mask nonconvex})}

In this section we will present the proof of convergence of Algorithm~\ref{algorithm: Rolling Masked FedAvg} in nonconvex setting. Since constrained non-convex stochastic optimization suffers from residual noise error unless a large mini-batch is used~\citep{ghadimi2016mini}, here we assume an unconstrained setting, i.e., $\cW = \R^d$.
\begin{proof}
    
 From the smoothness of $ F_{\mask}$ and Lemma~\ref{lem:epoch recursion}, we have

    \begin{align*}
         F_{\mask}(\bw_{e+1}) &\leq  F_{\mask}(\bw_e) + \inprod{\nabla  F_{\mask}(\bw_e)}{\bw_{e+1} - \bw_e} + \frac{L}{2}\norm{\bw_{e+1} - \bw_e}^2\\
        &\leq  F_{\mask}(\bw_e) - \eta RK \norm{ \nabla  F_{\mask}(\bw_e)}^2 + \inprod{\nabla  F_{\mask}(\bw_e)}{\bg_e} + \frac{L}{2}\norm{\bw_{e+1} - \bw_e}^2\\
    \end{align*}
   where 
      \begin{align*}
          \bg_e &=  \eta^2 K^2 \sum_{r=0}^{R-2} \mA_{r+1}   \frac{1}{N}\sum_{i=1}^N   \mM_i^{\sigma_e(r+1)} \mH^i_{e,r+1} \mM_i^{\sigma_e(r+1)}   \sum_{j=0}^r \frac{1}{N}\sum_{i=1}^N     \nabla f_i^{\sigma_e(j)}(  \bw_{e } )\\
           &\quad - \sum_{r=0}^{R-1} \mA_r \pare{\eta\frac{1}{N}\sum_{i=1}^N\sum_{k=0}^{K-1}\br^i_{e,r,k}  + \bxi_{e,r}}.
      \end{align*} 
    Taking expectation over randomness of samples yields:
    \begin{align*}
       \E [ F_{\mask}(\bw_{e+1}) ] 
        &\leq \E [ F_{\mask}(\bw_e)] - \eta RK \E \norm{ \nabla  F_{\mask}(\bw_e)}^2 + \E \inprod{\nabla  F_{\mask}(\bw_e)}{\bh_e} + \frac{L}{2}\E \norm{\bw_{e+1} - \bw_e}^2\\
        &\leq \E [ F_{\mask}(\bw_e)] - \eta RK \E \norm{ \nabla  F_{\mask}(\bw_e)}^2\\
        &+ \underbrace{\E \inprod{\nabla  F_{\mask}(\bw_e)}{\eta^2 K^2 \bh_e}}_{T_1}  + \underbrace{\E \inprod{\nabla  F_{\mask}(\bw_e)}{- \sum_{r=0}^{R-1} \mA_r \pare{\eta\frac{1}{N}\sum_{i=1}^N\sum_{k=0}^{K-1}\br^i_{e,r,k}  } } }_{T_2} + \frac{L}{2}\E \norm{\bw_{e+1} - \bw_e}^2\\
    \end{align*}
      where 
      \begin{align*}
          \bh_e &=    \sum_{r=0}^{R-2} \prod_{r'=R-1}^{r+2} \pare{\mI - \eta K  \frac{1}{N}\sum_{i=1}^N   \mM_i^{\sigma_e(r')} \mH^i_{e,r'} \mM_i^{\sigma_e(r')} }   \frac{1}{N}\sum_{i=1}^N   \mM_i^{\sigma_e(r+1)} \mH^i_{e,r+1} \mM_i^{\sigma_e(r+1)}   \sum_{j=0}^r \frac{1}{N}\sum_{i=1}^N     \nabla f_i^{\sigma_e(j)}(  \bw_{e } ).
      \end{align*} 
    For $T_1$:
    \begin{align*}
        T_1  = \E \inprod{\nabla  F_{\mask}(\bw_e)}{\eta^2 K^2 \bh_e}  \leq \eta^2 K^2 \E \sbr{\frac{1}{\sqrt{\eta K}}\norm{\nabla  F_{\mask}(\bw_e)}\sqrt{\eta K}\norm{ \bh_e}}. 
    \end{align*}
    Then we bound $ \norm{ \bh_e} $ as
    \begin{align*}
       \norm{ \bh_e}  \leq  \sum_{r=0}^{R-2}(1+\eta K L)^RL\norm{   \sum_{j=0}^r \frac{1}{N}\sum_{i=1}^N     \nabla f_i^{\sigma_e(j)}(  \bw_{e } )}  .
    \end{align*}
  and applying Hoeffding-Serfling inequality~\cite[Theorem 2]{pmlr-v51-schneider16} gives:
  \begin{align*}
    \sum_{r=0}^{R-2} L  \norm{\sum_{j=0}^{r } \frac{1}{N}\sum_{i=1}^N \nabla f_i^{\sigma_e(j)}(  \bw_{e } )}  & \leq   \sum_{r=0}^{R-2} L \pare{ \norm{\nabla  F_{\mask}(\bw_e) - \frac{1}{r} \sum_{j=0}^{r} \frac{1}{N}\sum_{i=1}^N \nabla f_i^{\sigma_e(j)}(  \bw_{e } )}  +  r  \norm{\nabla  F_{\mask}(\bw_e)} }\\
      &\leq   \sum_{r=0}^{R-2} L \pare{  \sqrt{G  \frac{8 \log(2RK/\nu)}{r}} + r \norm{\nabla  F_{\mask}(\bw_e)}}\\
    & \leq   L \pare{ \frac{2R^{3/2}}{3} \sqrt{G   {8 \log(2RK/\nu)} } + R^2 \norm{\nabla  F_{\mask}(\bw_e)}}.
    \end{align*}
    Putting pieces together yields:
      \begin{align*}
        T_1    & \leq \eta^2 K^2 \E \sbr{ \norm{\nabla  F_{\mask}(\bw_e)} \norm{ \bh_e}}\\ 
        & \leq    \eta^2 K^2 \E \sbr{ \norm{\nabla  F_{\mask}(\bw_e)}  3   L \pare{ \frac{2R^{3/2}}{3} \sqrt{G   {8 \log(2RK/\nu)} } + R^2 \norm{\nabla  F_{\mask}(\bw_e)}}  }\\ 
        & \leq    2 L \E \sbr{\sqrt{\frac{\eta K R}{4L}} \norm{\nabla  F_{\mask}(\bw_e)}      \pare{ 4\sqrt{L}\eta^{3/2} K^{3/2} R    \sqrt{G   {8 \log(2RK/\nu)} }  }  } +    3   L\eta^2 R^2 K^2 \E \sbr{ \norm{\nabla  F_{\mask}(\bw_e)}^2  }\\  
         & \leq      L \E \sbr{ {\frac{\eta K R}{4L}} \norm{\nabla  F_{\mask}(\bw_e)}^2  +    \pare{128  {L}\eta^{3 } K^{3 } R^2    {G   {  \log(2RK/\nu)} }  }  } +    3   L\eta^2 R^2 K^2 \E \sbr{ \norm{\nabla  F_{\mask}(\bw_e)}^2  }\\  
         & =   \pare{{\frac{\eta K R}{4 }}+ 3   L\eta^2 R^2 K^2 }  \E \sbr{  \norm{\nabla  F_{\mask}(\bw_e)}^2 } +    128  {L}^2 \eta^{3 } K^{3 } R^2    {G   {  \log(2RK/\nu)} }     . 
    \end{align*}
   
    For $T_2$ we have:
    \begin{align*}
        T_2 &=\E \inprod{\nabla  F_{\mask}(\bw_e)}{- \sum_{r=0}^{R-1} \mA_r \pare{\eta\frac{1}{N}\sum_{i=1}^N\sum_{k=0}^{K-1}\br^i_{e,r,k}  } } \\
        & \leq \eta\E \sbr{ \sqrt{RK}\norm{\nabla  F_{\mask}(\bw_e)} \frac{1}{\sqrt{RK}}\norm{ \sum_{r=0}^{R-1} \mA_r \pare{\frac{1}{N}\sum_{i=1}^N\sum_{k=0}^{K-1}\br^i_{e,r,k}  } } }\\
        &\leq  \frac{1}{4}\eta   RK \E\norm{\nabla  F_{\mask}(\bw_e)}^2 + \eta\frac{1}{  RK }\E\norm{ \sum_{r=0}^{R-1} \mA_r \pare{\frac{1}{N}\sum_{i=1}^N\sum_{k=0}^{K-1}\br^i_{e,r,k}  } }^2 .
    \end{align*}
    For $\E\norm{ \sum_{r=0}^{R-1} \mA_r \pare{\frac{1}{N}\sum_{i=1}^N\sum_{k=0}^{K-1}\br^i_{e,r,k}  } }^2$ we have
    \begin{align*}
        \E\norm{ \sum_{r=0}^{R-1} \mA_r \pare{\frac{1}{N}\sum_{i=1}^N\sum_{k=0}^{K-1}\br^i_{e,r,k}  } }^2 &\leq R\sum_{r=0}^{R-1} (1+\eta K L)^{2R} \norm{   \frac{1}{N}\sum_{i=1}^N\sum_{k=0}^{K-1}\br^i_{e,r,k}  }^2 \\
        &\leq RK\sum_{r=0}^{R-1}    \sum_{k=0}^{K-1}(1+\eta K L)^{2R} \frac{1}{N}\sum_{i=1}^N\norm{   \bw^i_{e,r,k}  - \bw_{e,r} }^2.
    \end{align*} 
    We plug in Lemma~\ref{lem:local model deviation} to get
    \begin{align*}
        &    \E\norm{ \sum_{r=0}^{R-1} \mA_r \pare{\frac{1}{N}\sum_{i=1}^N\sum_{k=0}^{K-1}\br^i_{e,r,k}  } }^2  \\
        &\leq RK\sum_{r=0}^{R-1}    \sum_{k=0}^{K-1}(1+\eta K L)^{2R}   \pare{48\eta^2 R K^3   \zeta + \pare{ 12 \pare{ 6R^3\eta^4 K^5 L^2+  18R^5  \eta^6 K^7    L^4  } 
 + 48\eta^2 R K^3}\E\norm{\nabla  F_{\mask}(\bw_e)}^2} \\
         &\quad   +  RK\sum_{r=0}^{R-1}    \sum_{k=0}^{K-1}(1+\eta K L)^{2R}   12    \pare{48R^2\eta^4 K^5 L^2+432 R^4 \eta^6 K^7        L^4 } G  { \log(2RK/\nu)}  \\ 
            & \quad  +  RK\sum_{r=0}^{R-1}    \sum_{k=0}^{K-1}(1+\eta K L)^{2R}   \pare{216     R^3 \eta^4 K^5 L^2 \frac{\delta^2}{N}+ 6\eta^2 R K^3 \delta^2}  \\ 
            &\leq 9R^2K^2   \pare{48\eta^2 R K^3   \zeta + \pare{ 12 \pare{ 6R^3\eta^4 K^5 L^2+  18R^5  \eta^6 K^7    L^4  } 
 + 48\eta^2 R K^3}\E\norm{\nabla  F_{\mask}(\bw_e)}^2} \\
         & \quad +  9R^2K^2  12    \pare{48R^2\eta^4 K^5 L^2+432 R^4 \eta^6 K^7        L^4 } G  { \log(2RK/\nu)}    +  9R^2K^2   \pare{216     R^3 \eta^4 K^5 L^2 \frac{\delta^2}{N}+ 6\eta^2 R K^3 \delta^2}.
    \end{align*}
Putting pieces together yields:
   
    \begin{align*}
        T_2  &\leq  \frac{1}{4}\eta   RK \E\norm{\nabla  F_{\mask}(\bw_e)}^2 + \eta\frac{1}{  RK }\E\norm{ \sum_{r=0}^{R-1} \mA_r \pare{\frac{1}{N}\sum_{i=1}^N\sum_{k=0}^{K-1}\br^i_{e,r,k}  } }^2 \\
        &\leq  \frac{1}{4}\eta   RK \E\norm{\nabla  F_{\mask}(\bw_e)}^2   +  9\eta R K    \pare{48\eta^2 R K^3   \zeta + \pare{ 12 \pare{ 6R^3\eta^4 K^5 L^2+  18R^5  \eta^6 K^7    L^4  } 
 + 48\eta^2 R K^3}\E\norm{\nabla  F_{\mask}(\bw_e)}^2} \\
         & \quad +    9\eta R K  12    \pare{48R^2\eta^4 K^5 L^2+432 R^4 \eta^6 K^7        L^4 } G  { \log(2RK/\nu)}   +   9\eta R K    \pare{216     R^3 \eta^4 K^5 L^2 \frac{\delta^2}{N}+ 6\eta^2 R K^3 \delta^2} .
    \end{align*}
    Plugging $T_1$ and $T_2$ back yields:
     \begin{align*}
    &   \E [ F_{\mask}(\bw_{e+1}) ]  \\
        &\leq \E [ F_{\mask}(\bw_e)] - \pare{\frac{1}{2}\eta RK - 3   L\eta^2 R^2 K^2 - 9\eta RK\pare{ 12 \pare{ 6R^3\eta^4 K^5 L^2+  18R^5  \eta^6 K^7    L^4  } 
 + 48\eta^2 R K^3}} \E \norm{ \nabla  F_{\mask}(\bw_e)}^2 \\
        &\quad   +   128  {L}^2 \eta^{3 } K^{3 } R^2    {G   {  \log(2RK/\nu)} }  \\
        &\quad +  9\eta R K     \cdot 48\eta^2 R K^3   \zeta  +    9\eta R K  12    \pare{48R^2\eta^4 K^5 L^2+432 R^4 \eta^6 K^7        L^4 } G  { \log(2RK/\nu)}  \\ 
            &\quad   +   9\eta R K    \pare{216     R^3 \eta^4 K^5 L^2 \frac{\delta^2}{N}+ 6\eta^2 R K^3 \delta^2} + \frac{L}{2}\E \norm{\bw_{e+1} - \bw_e}^2.
    \end{align*}

From (\ref{eq: iterate difference bound }) we know 
    \begin{align*}
     \E\norm{\bw_{e+1} - \bw_e}^2   
            & \leq \frac{5}{16L}\eta RK \E\norm{    \nabla  F_{\mask}(  \bw_{e } ) }^2\\
 & \quad +  \pare{324 \eta^2 RKL^2    \pare{48R^2\eta^4 K^5 L^2+432 R^4 \eta^6 K^7        L^4 } + 432 L^2 \eta^4 R^3 K^4 } G  { \log(2RK/\nu)}\\ 
       & \quad +  27 \eta^2 RKL^2 \pare{ 216     R^3 \eta^4 K^5 L^2 \frac{\delta^2}{N}+ 6\eta^2 R K^3 \delta^2 }  + {27 }    \eta^2 R^2 K^2 \frac{\delta^2}{N}\\ 
       & \quad + 1296 \eta^4 R^2 K^4 L^2         \zeta^2. 
\end{align*}
Putting pieces together yields:
     \begin{align*}
     &  \E [ F_{\mask}(\bw_{e+1}) ]  \\
        &\leq \E [ F_{\mask}(\bw_e)] - \pare{\frac{3}{16}\eta RK - 3   L\eta^2 R^2 K^2  - 9 \pare{ 12 \pare{ 6R^4\eta^5 K^6 L^2+  18R^6  \eta^7 K^8    L^4  } 
 + 48\eta^3R^2 K^4}} \E \norm{ \nabla  F_{\mask}(\bw_e)}^2 \\
        &\quad   +   128  {L}^2 \eta^{3 } K^{3 } R^2    {G   {  \log(2RK/\nu)} }   \\
        &\quad +  9    \cdot 48\eta^3 R^2 K^4   \zeta  +   108    \pare{48R^3\eta^5 K^6 L^2+432 R^5 \eta^7 K^8        L^4 } G  { \log(2RK/\nu)}  \\ 
            &\quad   +   9    \pare{216     R^4 \eta^5 K^6 L^2 \frac{\delta^2}{N}+ 6\eta^3 R^2 K^4 \delta^2} \\ 
 & \quad +  \frac{L}{2}\pare{324 \eta^2 RKL^2    \pare{48R^2\eta^4 K^5 L^2+432 R^4 \eta^6 K^7        L^4 } + 432 L^2 \eta^4 R^3 K^4 } G  { \log(2RK/\nu)}\\ 
       & \quad +  \frac{L}{2} 27 \eta^2 RKL^2 \pare{ 216     R^3 \eta^4 K^5 L^2 \frac{\delta^2}{N}+ 6\eta^2 R K^3 \delta^2 }  + {27 }    \eta^2 R^2 K^2 \frac{\delta^2}{N}\\ 
       & \quad +   648 \eta^4 R^2 K^4 L^3         \zeta^2. 
    \end{align*}
 Since $\eta \leq  {\frac{1}{c\cdot L\sqrt{ RKT}}}$ for some sufficiently large $c$, we know
    \begin{align*}
       3\eta^2 R^2 K^2 L \leq \frac{1}{64} \eta RK,\\
       648\eta^5 R^4 K^6 L^2 \leq \frac{1}{64} \eta RK,\\
       162 \eta^7 R^6 K^8 L^4 \leq \frac{1}{64} \eta RK,\\
       432 \eta^3 R^2 K^4 \leq \frac{1}{64} \eta RK.
    \end{align*}
   Hence we have
         \begin{align*}
  \frac{1}{T}\sum_{e=1}^T  \E \norm{ \nabla  F_{\mask}(\bw_e)}^2   \
        &\leq O\pare{\frac{\E [ F_{\mask}(\bw_0)]}{\eta RK T} }    + O\pare{      {L}^2 \eta^{2 } K^{2 } R     {G   {  \log(2RK/\nu)} } +  \eta^2 R  K^3   \zeta^2 +    \eta  R  K  \frac{\delta^2}{N} } 
    \end{align*}
    Since we choose $\eta = \Theta\pare{\frac{1}{L\sqrt{ RKT}}}$ we have
      \begin{align*}
  \frac{1}{T}\sum_{e=1}^T  \E \norm{ \nabla  F_{\mask}(\bw_e)}^2   \
        &\leq O\pare{\frac{L\E [ F_{\mask}(\bw_0)]}{\sqrt{ RK T}}      +  \frac{KG  \log(2RK/\nu)}{ T}        +    \frac{K^2   \zeta^2}{  TL^2}    +      \frac{\delta^2}{N\sqrt{ RKT} L}   } .
    \end{align*}
   \end{proof}

\section{Stability of Masked Training Method}\label{app:stability_masked_training}
In this section, we will present the proof of Theorem~\ref{thm:stab random cvx}. We first introduce the following technical lemma.
\begin{lemma}\label{lem:model deviation cvx stability}
For Algorithm~\ref{algorithm: Masked FedAvg}, under the condition of Theorem~\ref{thm:randomly masked fedavg}, the following statement holds true:
\begin{align*}
    \E\norm{ \tilde \bw_{r,k}^i -  \bw_r  }^2  &  \leq  5K\pare{ 4\eta^2 K \E\norm{       \nabla  F_{\bp}(   \bw_{r } )  }^2    + 4\eta^2 K\zeta_{\max}^2 +  \eta^2 K  {\delta^2} } 
\end{align*}

\begin{proof}
According to local updating rule we have:
    \begin{align*}
&         \E\norm{ \tilde \bw_{r,k}^i -  \bw_r  }^2  \\
       &=  (1+\frac{1}{K-1})\E\norm{ \tilde \bw_{r,k-1}^i -  \bw_r  }^2 + K \E\norm{  \eta    \tilde \bg_{r,k-1}^i  }^2\\
      &  \leq (1+\frac{1}{K-1})\E\norm{ \tilde \bw_{r,k-1}^i -  \bw_r  }^2 + K \E\norm{  \eta  \mask_i^r\odot  \nabla   f_i(\mask_i^r \odot\tilde \bw_{r,k-1}^i)  }^2  +  \eta^2 K  {\delta^2} \\ 
       &  \leq (1+\frac{1}{K-1})\E\norm{ \tilde \bw_{r,k-1}^i -  \bw_r  }^2 + 2K \E\norm{  \eta    \mask_i^r \odot \nabla   f_i(\mask_i^r \odot  \bw_{r } )  }^2 + 2\eta^2 \tilde L^2 K \norm{  \tilde \bw_{r,k-1}^i  -    \bw_{r} }^2 +  \eta^2 K  {\delta^2} \\ 
       &  \leq (1+\frac{2}{K-1})\E\norm{ \tilde \bw_{r,k-1}^i -  \bw_r  }^2 + 4\eta^2 K \E\norm{       \nabla F_{\bp}(   \bw_{r } )  }^2    + 4\eta^2 K\zeta_{\max}^2 +  \eta^2 K  {\delta^2} \\ 
       &  \leq  \sum_{j=1}^{k} (1+\frac{2}{K-1})^{k-j}\pare{   4\eta^2 K \E\norm{       \nabla F_{\bp}(   \bw_{r } )  }^2    + 4\eta^2 K\zeta_{\max}^2 +  \eta^2 K  {\delta^2} }\\ 
       &  \leq  5K\pare{   4\eta^2 K \E\norm{       \nabla F_{\bp}(   \bw_{r } )  }^2    + 4\eta^2 K\zeta_{\max}^2 +  \eta^2 K  {\delta^2} }\\ 
    \end{align*}
    where the fourth step is due to $2\eta^2 \tilde L^2 K \leq \frac{1}{K-1}$,
    which conclude the proof. 
    \end{proof}

\end{lemma}

\subsection{Proof of Theorem~\ref{thm:stab random cvx}}
Recall the output of Algorithm~\ref{algorithm: Masked FedAvg}:
\begin{align*}
    &\E \norm{\hat\bw - \hat\bw'} \\
    &= \E\norm{\cP_{\cW}\pare{ \bw_R - (1/L) \frac{1}{N}\sum_{i=1}^N \mask_i \odot \nabla f_i(\mask_i\odot\bw_R ) } - \cP_{\cW}\pare{ \bw'_R - (1/L) \frac{1}{N}\sum_{i=1}^N \mask_i \odot \nabla f_i(\mask_i\odot\bw'_R ) }  }\\
    &\leq \E\norm{  \bw_R -   \bw'_R - (1/L) \frac{1}{N}\sum_{i=1}^N \mask_i \odot \nabla f_i(\mask_i\odot\bw_R )   + (1/L) \frac{1}{N}\sum_{i=1}^N \mask_i \odot \nabla f_i(\mask_i\odot\bw'_R )   }\\
    &\leq 2\E\norm{  \bw_R -   \bw'_R }.
\end{align*}

Hence it suffices to bound $\E\norm{  \bw_R -   \bw'_R }$.

For the client $i$ with perturbed data, by the updating rule we know with probability $1-\frac{1}{n}$:
\begin{align*}
  \E\norm{\bw^i_{r,k+1 }  - {\bw'}^i_{r,k+1 } }  & \leq   \E \norm{\bw^i_{r,k }  - {\bw'}^i_{r,k }  - \eta   \pare{\mask_i^r \odot \nabla f_i(\mask_i^r\odot\bw_{r,k }^i;\xi_{r,k }^i)  -  \mask_i^r \odot \nabla f_i(\mask_i^r\odot {\bw'}_{r,k }^i ;\xi_{r,k }^i) } }   \\ 
     & \leq  \E \norm{\bw_{r,k }  - \bw'_{r,k }      } .  
\end{align*}
With probability $\frac{1}{n}$, we have
\begin{align*}
  \E\norm{\bw^i_{r,k+1 }  - {\bw'}^i_{r,k+1 } }  & \leq   \E \norm{\bw^i_{r,k }  - {\bw'}^i_{r,k }  - \eta   \pare{\mask_i^r \odot \nabla f_i(\mask_i^r\odot\bw_{r,k }^i;\xi_{r,k }^i)  -  \mask_i^r \odot \nabla f_i(\mask_i^r\odot {\bw'}_{r,k }^i ;{\xi'}_{r,k }^i) } }   \\ 
  & \leq   \E \norm{\bw^i_{r,k }  - {\bw'}^i_{r,k }  - \eta   \pare{\mask_i^r \odot \nabla f_i(\mask_i^r\odot\bw_{r,k }^i;\xi_{r,k }^i)  -  \mask_i^r \odot \nabla f_i(\mask_i^r\odot {\bw'}_{r,k }^i ;{\xi}_{r,k }^i) } } \\
  &\quad + \norm{\eta   \pare{\mask_i^r \odot \nabla f_i(\mask_i^r\odot{\bw'}_{r,k }^i;\xi_{r,k }^i)  -  \mask_i^r \odot \nabla f_i(\mask_i^r\odot {\bw'}_{r,k }^i ;{\xi'}_{r,k }^i) }} \\ 
     & \leq   \E \norm{\bw^i_{r,k }  - {\bw'}^i_{r,k }     } + 
    \eta \E\norm{    \mask_i^r \odot \nabla f_i(\mask_i^r\odot{\bw'}_{r,k }^i;\xi_{r,k }^i) } + \eta\E\norm{ \mask_i^r \odot \nabla f_i(\mask_i^r\odot {\bw'}_{r,k }^i ;{\xi'}_{r,k }^i) } \\
    &\leq  \E \norm{\bw^i_{r,k }  - {\bw'}^i_{r,k }      } + 
    \eta \E\norm{      \nabla  F_{\bp}( {\bw}_{r }  ) }  + \eta\E\norm{  \nabla  F_{\bp}( {\bw'}_{r }  ) } \\
    &\quad + 2\eta \norm{ {\bw'}_{r } -  {\bw'}_{r,k }^i } + 2\eta \delta     + 2\eta \zeta_i   
\end{align*}
For $j \neq i$, we have $\E\norm{\bw^j_{r+1 }  - {\bw'}^j_{r+1 } }  \leq  \E\norm{\bw^j_{r,K-1 }  - {\bw'}^j_{r,K-1 } } \leq \E\norm{\bw^j_{r }  - {\bw'}^j_{r } }$.
Combining two cases we have
\begin{align*}
   \frac{1}{N}\sum_{j=1}^N\E\norm{\bw^j_{r,k+1 }  - {\bw'}^j_{r,k+1 } }  
      &\leq  \frac{1}{N}\sum_{j=1}^N \E \norm{\bw^j_{r,k }  - {\bw'}^j_{r,k }     }  + \frac{2\eta \delta}{Nn}+ \frac{2\eta \zeta_i}{Nn} + \frac{1}{Nn}2\eta \norm{ {\bw'}_{r } -  {\bw'}_{r,k }^i }  \\
      &\quad + \frac{1}{Nn}
    \eta \E\norm{      \nabla  F_{\bp}( {\bw}_{r }  ) }  + \frac{1}{Nn}\eta\E\norm{  \nabla F_{\bp} ( {\bw'}_{r }  ) }  .
\end{align*}
Performing telescoping sum from $k=K-1$ to $0$ yields:
\begin{align*}
  \E\norm{\bw_{r+1 }  - {\bw'}_{r+1 } }   
     & \leq  \E \norm{\bw_{r}  - {\bw'}_{r}     }  + \frac{2\eta K\delta}{Nn}+ \frac{2\eta K\zeta_{\max}}{Nn}+ \frac{2\eta}{Nn} \sum_{k=1}^K\E\norm{ {\bw'}_{r } -  {\bw'}_{r,k }^i }  \\
      &\quad + \frac{1}{Nn}
    \eta K\E\norm{      \nabla F_{\bp}( {\bw}_{r }  ) }  + \frac{1}{Nn}\eta K\E\norm{  \nabla F_{\bp}( {\bw'}_{r }  ) }  .
\end{align*}
Performing telescoping sum from $r=R-1$ to $0$, and using the fact $\bw_0 = \bw'_0$ yields:

\begin{align*}
  \E\norm{\bw_{R }  - {\bw'}_{R } }   
     & \leq    \frac{2\eta RK\delta}{Nn}+ \frac{2\eta RK\zeta_{\max}}{Nn}+ \frac{2\eta}{Nn} \sum_{r=1}^R\sum_{k=1}^K\E\norm{ {\bw'}_{r } -  {\bw'}_{r,k }^i }  \\
      &\quad + \frac{1}{Nn}
    \eta K\sum_{r=1}^R\E\norm{      \nabla F_{\bp}( {\bw}_{r }  ) }  + \frac{1}{Nn}\eta K\sum_{r=1}^R\E\norm{  \nabla F_{\bp} ( {\bw'}_{r }  ) }  \\
    & \leq    \frac{2\eta RK\delta}{Nn}+ \frac{2\eta RK\zeta_{\max}}{Nn}+ \frac{2\eta}{Nn} \sum_{r=1}^R\sum_{k=1}^K\E\norm{ {\bw'}_{r } -  {\bw'}_{r,k }^i }  \\
      &\quad + \frac{1}{Nn}
    \eta RK\sqrt{\frac{1}{R}\sum_{r=1}^R\E\norm{      \nabla F_{\bp}( {\bw}_{r }  ) }^2 } + \frac{1}{Nn}\eta RK\sqrt{\frac{1}{R}\sum_{r=1}^R\E\norm{  \nabla F_{\bp}( {\bw'}_{r }  ) }^2 }.
\end{align*}

To bound $\E\norm{ \bw_{r,k }^i   -  {\bw}_{r,k }    } $, evoking Lemma~\ref{lem:model deviation cvx stability} gives:
\begin{align*}
     \E\norm{ \bw_{r,k }^i   -  {\bw}_{r,k }    } &\leq     \sqrt{\E\norm{ \bw_{r,k }^i   -  {\bw}_{r,k }    }^2 }   \leq  \sqrt{ 5K\pare{ 4\eta^2 K \E\norm{       \nabla F_{\bp}(   \bw_{r } )  }^2    + 4\eta^2 K\zeta_{\max}^2 +  \eta^2 K  {\delta^2} }  }.
\end{align*}

Hence we have
\begin{align}
  \E\norm{\bw_{R }  - {\bw'}_{R } }    
    & \leq    \frac{2\eta RK\delta}{Nn}+ \frac{2\eta RK\zeta_i}{Nn}+ \frac{2\eta}{Nn} \sum_{r=1}^R\sum_{k=1}^K  \sqrt{ 5K\pare{ 4\eta^2 K \E\norm{       \nabla F_{\bp}(   \bw_{r } )  }^2    + 4\eta^2 K\zeta_{\max}^2 +  \eta^2 K  {\delta^2} }  } \nonumber  \\
      &\quad + \frac{1}{Nn}
    \eta RK\sqrt{\frac{1}{R}\sum_{r=1}^R\E\norm{      \nabla F_{\bp}( {\bw}_{r }  ) }^2 } + \frac{1}{Nn}\eta RK\sqrt{\frac{1}{R}\sum_{r=1}^R\E\norm{  \nabla F_{\bp}( {\bw'}_{r }  ) }^2 }.\label{eq:stab 0}
\end{align}

Now we will bound the gradient norm $\E\norm{ \nabla F_{\bp}(   \bw_{r } )  }^2 $.
Due to Eq.(\ref{eq:point convergence}) we have

 \begin{align*}
       \E\norm{ \bw_{r+1} - \bw^* }^2   
           & \leq   \E \norm{\bw_r  - \bw^*}^2 - \frac{1}{8}\eta  K     \pare{   F_{\bp}(  \bw_r ) -   F_{\bp}(  \bw^*)  }   \\
         &\quad + 2 \eta \tilde L\cdot 5K^2\pare{  4\eta^2 K   \sigma_{*}^2  +  \eta^2 K  {\delta^2} } + \frac{K \eta^2 \delta^2 }{N}.
    \end{align*}
    Due to the fact $\norm{\nabla F_\bp(\bw_r)  }^2 \leq 2L\pare{   F_{\bp}(  \bw_r ) -   F_{\bp}(  \bw^*)  }$ we have
 \begin{align*}
       \E\norm{ \bw_{r+1} - \bw^* }^2   
           & \leq   \E \norm{\bw_r  - \bw^*}^2 - \frac{1}{16L}\eta  K      \E\norm{\nabla F_{\bp}(\bw_r)  }^2  \\
         &\quad + 2 \eta \tilde L\cdot 5K^2\pare{  4\eta^2 K \sigma_{*}^2  +  \eta^2 K  {\delta^2} } + \frac{K \eta^2 \delta^2 }{N}.
    \end{align*}
Summing over $r=1$ to $R$  and dividing both sides by $R$ yields:
 \begin{align*}
      \frac{1}{16R\tilde L}\eta  K   \sum_{r=1}^R   \E\norm{\nabla  F_{\bp}(\bw_r)  }^2     
           & \leq \frac{\E \norm{\bw_1  - \bw^*}^2}{R}    + 2 \eta \tilde L\cdot 5K^2\pare{  4\eta^2 K  \sigma_{*}^2  +  \eta^2 K  {\delta^2} } + \frac{K \eta^2 \delta^2 }{N}\\
            \Longleftrightarrow \frac{1}{ R  }  \sum_{r=1}^R   \E\norm{\nabla  F_{\bp}(\bw_r)  }^2     
           & \leq \frac{16\tilde L\E \norm{\bw_1  - \bw^*}^2}{\eta KR}    +32   \tilde L^2\cdot 5K \pare{  4\eta^2 K  \sigma_*^2  +  \eta^2 K  {\delta^2} } + \frac{ 32L \eta  \delta^2 }{N}.
    \end{align*}
    
Plugging above bound back to (\ref{eq:stab 0}) yields:
\begin{align*}
 & \E\norm{\bw_{R }  - {\bw'}_{R } } \\   
     & \leq    \frac{2\eta RK\delta}{Nn}+ \frac{2\eta RK\zeta_{\max}}{Nn}+ \frac{2\eta}{Nn} \sum_{r=1}^R K\sqrt{ 5K\pare{ 4\eta^2 K \E\norm{       \nabla  F_{\bp}(   \bw_{r } )  }^2    + 4\eta^2 K\zeta_{\max}^2 +  \eta^2 K  {\delta^2} }  }  \\
      &\quad + 2\frac{1}{Nn}
    \eta RK \sqrt{ \frac{16\tilde L\E \norm{\bw_1  - \bw^*}^2}{\eta KR}    +32   \tilde L^2\cdot 5K \pare{  4\eta^2 K  \sigma_*^2  +  \eta^2 K  {\delta^2} } + \frac{ 32L \eta  \delta^2 }{N}  } \\
    & \leq   O\pare{  \frac{ \eta RK(\delta+\zeta_{\max})}{Nn} + \frac{ RK\eta}{Nn} \sqrt{    \eta^2 K^2 \pare{\frac{ \tilde L }{\eta KR}    +   \tilde L^2  K \pare{   \eta^2 K  \sigma_*^2  +  \eta^2 K  {\delta^2} } + \frac{  L \eta  \delta^2 }{N}}    +  \eta^2 K^2(\zeta_{\max}^2 +   {\delta^2}) }  }  \\
      &\quad +  \frac{\eta RK}{Nn}
     O\pare{ \sqrt{ \frac{ \tilde L }{\eta KR}    + \tilde L^2 K \pare{   \eta^2 K  \sigma_*^2  +  \eta^2 K  {\delta^2} } + \frac{ L \eta  \delta^2 }{N}  }}.
\end{align*}
Choosing $\eta = \frac{\sqrt{Nn}}{RK}$ and $\frac{R}{\sqrt{Nn}} \geq \tilde L $ will conclude the proof:
\begin{align*}
  \E\norm{\bw_{R }  - {\bw'}_{R } }    
    & \leq   O\pare{  \frac{   \delta+\zeta_{\max} }{\sqrt{Nn}} +  \frac{1}{\sqrt{Nn}}
      \pare{ \sqrt{ \frac{ \tilde L }{\sqrt{Nn}}    +     \tilde L^2  \frac{ {Nn}}{R^2 }    \sigma_*^2  +  \tilde L^2 \frac{ {Nn}}{R^2 }    {\delta^2}   + \frac{\sqrt{Nn}}{RK}\frac{ L   \delta^2 }{N}  }} }\\  
      & \leq   O\pare{  \frac{   \delta+\zeta_{\max} }{\sqrt{Nn}} +  \frac{1}{\sqrt{Nn}}
      \sqrt{ \frac{ \tilde L }{\sqrt{Nn}}    +          \sigma_*^2  +    {\delta^2}   +  \frac{     \delta^2 }{KN}  }  }.   
\end{align*}
Finally, applying the fact that $\zeta_{\max} \leq GD_{\max}$ (See~\cite{sun2024understanding}, Lemma 1) will conclude the proof.

\subsection{Proof of Theorem~\ref{thm:stab rolling cvx}}
Recall the output of Algorithm~\ref{algorithm: Rolling Masked FedAvg}:
\begin{align*}
    &\E \norm{\hat\bw - \hat\bw'} = \E\left\|\cP_{\cW}\pare{ \bw_{T } - \frac{1}{L} \frac{1}{N}\sum_{i=1}^N \frac{1}{R}\sum_{r=1}^R\mask_i^{\sigma_e(r)} \odot \nabla f_i(\mask_i^{\sigma_e(r)}\odot\bw_{T } ) } \right. \\
    &\qquad  \qquad \qquad \left.- \cP_{\cW}\pare{ \bw'_{T } - \frac{1}{L} \frac{1}{N}\sum_{i=1}^N \frac{1}{R}\sum_{r=1}^R\mask_i^{\sigma_e(r)} \odot \nabla f_i(\mask_i^{\sigma_e(r)}\odot\bw'_{T } ) } \right\|\\
    &\qquad  \qquad \qquad \leq 2\E\norm{  \bw_T -   \bw'_T }.
\end{align*}

Hence it suffices to bound $\E\norm{  \bw_T -   \bw'_T }$.

For the client $i$ with perturbed data, by the updating rule and convexity of $f_i$ we know with probability $1-\frac{1}{n}$:
\begin{align*}
  &\E\norm{\bw^i_{e,r,k+1 }  - {\bw'}^i_{e,r,k+1 } }  \\
  & \leq   \E \norm{\bw^i_{e,r,k }  - {\bw'}^i_{e,r,k }  - \eta   \pare{\mask_i^{\sigma_e(r)} \odot \nabla f_i(\mask_i^{\sigma_e(r)} \odot\bw_{e,r,k }^i;\xi_{e,r,k }^i)  -  \mask_i^{\sigma_e(r)} \odot \nabla f_i(\mask_i^{\sigma_e(r)}\odot {\bw'}_{e,r,k }^i ;\xi_{e,r,k }^i) } }   \\ 
     & \leq  \E \norm{\bw_{e,r,k }  - \bw'_{e,r,k }      } .  
\end{align*}
With probability $\frac{1}{n}$, we have
\begin{align*}
 & \E\norm{\bw^i_{e,r,k+1 }  - {\bw'}^i_{e,r,k+1 } }  \\
  & \leq   \E \norm{\bw^i_{e,r,k }  - {\bw'}^i_{e,r,k }  - \eta   \pare{\mask_i^{\sigma_e(r)} \odot \nabla f_i(\mask_i^{\sigma_e(r)}\odot\bw_{e,r,k }^i;\xi_{e,r,k }^i)  -  \mask_i^{\sigma_e(r)} \odot \nabla f_i(\mask_i^{\sigma_e(r)}\odot {\bw'}_{e,r,k }^i ;{\xi'}_{e,r,k }^i) } }   \\ 
    & \leq   \E \norm{\bw^i_{e,r,k }  - {\bw'}^i_{e,r,k }}   + \eta  \norm{ \mask_i^{\sigma_e(r)} \odot \nabla f_i(\mask_i^{\sigma_e(r)}\odot\bw_{e,r,k }^i;\xi_{e,r,k }^i)  }+ \eta \norm{  \mask_i^{\sigma_e(r)} \odot \nabla f_i(\mask_i^{\sigma_e(r)}\odot {\bw'}_{e,r,k }^i ;{\xi'}_{e,r,k }^i) }     \\ 
    & \leq   \E \norm{\bw^i_{e,r,k }  - {\bw'}^i_{e,r,k }} + 2\eta \delta   + \eta  \norm{  \mask_i^{\sigma_e(r)} \odot \nabla f_i(\mask_i^{\sigma_e(r)}\odot\bw_{e,r,k }^i ) }+ \eta \norm{  \mask_i^{\sigma_e(r)} \odot \nabla f'_i(\mask_i^{\sigma_e(r)}\odot {\bw'}_{e,r,k }^i ) }     \\ 
     & \leq   \E \norm{\bw^i_{e,r,k }  - {\bw'}^i_{e,r,k }} + 2\eta \delta   + \eta  \norm{   \nabla F_{\mask}( \bw_{e,r,k }^i ) }+ \eta \norm{ F_{\mask}'(  {\bw'}_{e,r,k }^i ) }    + \eta \zeta_i + \eta \zeta_i'  \\
    &\leq  \E \norm{\bw^i_{e,r,k }  - {\bw'}^i_{e,r,k }      } + 
    \eta \E\norm{      \nabla  F_{\mask}( {\bw}_{e }  ) }  + \eta\E\norm{  \nabla  F'_{\mask}( {\bw'}_{e}  ) }  + \eta L \norm{ {\bw}_{e} -  {\bw}_{e,r,k  }^i } + \eta L \norm{ {\bw'}_{e } -  {\bw'}_{e,r,k  }^i } \\
    &\quad + 2\eta \delta     + 2\eta G D_{\max}\\
    &\leq  \E \norm{\bw^i_{e,r,k }  - {\bw'}^i_{e,r,k }      } + 
    \eta \E\norm{      \nabla  F_{\mask}( {\bw}_{e }  ) }  + \eta\E\norm{  \nabla  F'_{\mask}( {\bw'}_{e}  ) }  + \eta L \norm{ {\bw}_{e,r} -  {\bw}_{e,r,k  }^i } + \eta L \norm{ {\bw'}_{e,r} -  {\bw'}_{e,r,k  }^i } \\
    &\quad + 2\eta \delta     + 2\eta G D_{\max}\\
    &\quad +   \eta L \norm{ {\bw}_{e} -  {\bw}_{e,r} } + \eta L \norm{ {\bw'}_{e } -  {\bw'}_{e,r} }.
\end{align*}

To bound $\E\norm{ {\bw}_{e} -  {\bw}_{e,r} } $ and $ \E\norm{ {\bw'}_{e } -  {\bw'}_{e,r} }$, we evoke \eqref{eq:w_er - w_e}:
  \begin{align}
       \E \norm{\bw_{e,r} - \bw_{e}}^2    & \leq  \pare{48r\eta^2 K^2+ 432 r^3 \eta^4 K^4        L^2 } G  { \log(2RK/\nu)}  + \pare{ 6r^2\eta^2 K^2+  18 r^4  \eta^4 K^4    L^2  }\norm{\nabla  F_{\mask}(\bw_e)}^2  \nonumber\\ 
            &\quad + 18\eta^2 r \sum_{p=0}^{r-1}  KL^2\sum_{k=0}^{K-1} \frac{1}{N}\sum_{i=1}^N  \E \norm{\bw^i_{e,p,k} - \bw_{e,p}   }^2+ 18 r^2 \eta^2 K^2 \frac{\delta^2}{N}. 
    \end{align}
    Hence we have
\begin{align*}
  \E\norm{\bw^i_{e,r,k+1 }  - {\bw'}^i_{e,r,k+1 } }    &\leq  \E \norm{\bw^i_{e,r,k }  - {\bw'}^i_{e,r,k }      } + 
   \pare{ \eta + \eta L  \sqrt{ 6r^2\eta^2 K^2+  18 r^4  \eta^4 K^4    L^2  } }\E\norm{      \nabla  F_{\mask}( {\bw}_{e }  ) }  \\
   & \quad +   \pare{ \eta + \eta L  \sqrt{ 6r^2\eta^2 K^2+  18 r^4  \eta^4 K^4    L^2  } }\E\norm{      \nabla  F'_{\mask}( {\bw'}_{e }  ) }  \\
   &\quad  + \eta L \sqrt{ 18\eta^2 r \sum_{p=0}^{r-1}  KL^2\sum_{k=0}^{K-1} \frac{1}{N}\sum_{i=1}^N  \E \norm{\bw^i_{e,p,k} - \bw_{e,p}   }^2 + 18 r^2 \eta^2 K^2 \frac{\delta^2}{N}} \\
   &\quad  + \eta L \sqrt{ 18\eta^2 r \sum_{p=0}^{r-1}  KL^2\sum_{k=0}^{K-1} \frac{1}{N}\sum_{i=1}^N  \E \norm{{\bw'}^i_{e,p,k} - \bw'_{e,p}   }^2 + 18 r^2 \eta^2 K^2 \frac{\delta^2}{N}} \\
   &\quad  + 4\eta L \sqrt{  \pare{48r\eta^2 K^2+ 432 r^3 \eta^4 K^4        L^2 } G  { \log(2RK/\nu)}    } \\
   &\quad + \eta L   \E\norm{ {\bw}_{e,r} -  {\bw}_{e,r,k  }^i } + \eta L  \E\norm{ {\bw'}_{e,r} -  {\bw'}_{e,r,k  }^i } + 2\eta \delta     + 2\eta G D_{\max} .
\end{align*}

For $j \neq i$, we have $\E\norm{\bw^j_{e,r,k+1 }  - {\bw'}^j_{e,r,k+1} }  \leq  \E\norm{\bw^j_{e,r,k }  - {\bw'}^j_{e,r,k} } $.
Combining two cases we have
\begin{align*}
    \frac{1}{N}\sum_{j=1}^N\E\norm{\bw^j_{e,r,k+1 }  - {\bw'}^j_{e,r,k+1 } }  &\leq  \frac{1}{N}\sum_{j=1}^N \E \norm{\bw^j_{e,r,k }  - {\bw'}^j_{e,r,k }     }  \nonumber\\
      & \   + \frac{1}{Nn}\pare{ \eta + \eta L  \sqrt{ 6r^2\eta^2 K^2+  18 r^4  \eta^4 K^4    L^2  } }\pare{\E\norm{      \nabla  F_{\mask}( {\bw}_{e }  ) }  + \E\norm{\nabla  F'_{\mask}( {\bw'}_{e }  )} }\\
       & \   + \frac{1}{Nn} \eta L \sqrt{ 18\eta^2 r \sum_{p=0}^{r-1}  KL^2\sum_{k=0}^{K-1} \frac{1}{N}\sum_{i=1}^N  \E \norm{\bw^i_{e,p,k} - \bw_{e,p}   }^2 + 18 r^2 \eta^2 K^2 \frac{\delta^2}{N}} \\
   & \ +  \frac{1}{Nn}\eta L \sqrt{ 18\eta^2 r \sum_{p=0}^{r-1}  KL^2\sum_{k=0}^{K-1} \frac{1}{N}\sum_{i=1}^N  \E \norm{{\bw'}^i_{e,p,k} - \bw'_{e,p}   }^2 + 18 r^2 \eta^2 K^2 \frac{\delta^2}{N}} \\
   & \   + 4 \frac{1}{Nn}\eta L \sqrt{  \pare{48r\eta^2 K^2+ 432 r^3 \eta^4 K^4        L^2 } G  { \log(2RK/\nu)}    } \\
   & \  +  \frac{1}{Nn} \eta L   \E\norm{ {\bw}_{e,r} -  {\bw}_{e,r,k  }^i } +  \frac{1}{Nn}\eta L  \E\norm{ {\bw'}_{e,r} -  {\bw'}_{e,r,k  }^i } + 2 \frac{1}{Nn}\eta \delta     + 2 \frac{1}{Nn}\eta G D_{\max} .
\end{align*}
Performing telescoping sum from $k=K-1$ to $0$ yields:
\begin{align*}
   \E\norm{\bw_{e,r+1}  - {\bw'}_{e,r+1 } }  
      &\leq   \E \norm{\bw_{e,r }  - {\bw'}_{e,r }     }  \nonumber\\
      &\  + \frac{1}{Nn}K \pare{ \eta + \eta L  \sqrt{ 6r^2\eta^2 K^2+  18 r^4  \eta^4 K^4    L^2  } }\pare{\E\norm{      \nabla  F_{\mask}( {\bw}_{e }  ) }  + \E\norm{\nabla  F'_{\mask}( {\bw'}_{e }  )} }\\
       &\   + \frac{1}{Nn} K\eta L \sqrt{ 18\eta^2 r \sum_{p=0}^{r-1}  KL^2\sum_{k=0}^{K-1} \frac{1}{N}\sum_{i=1}^N  \E \norm{\bw^i_{e,p,k} - \bw_{e,p}   }^2 + 18 r^2 \eta^2 K^2 \frac{\delta^2}{N}} \\
   &\   +  \frac{1}{Nn}K\eta L \sqrt{ 18\eta^2 r \sum_{p=0}^{r-1}  KL^2\sum_{k=0}^{K-1} \frac{1}{N}\sum_{i=1}^N  \E \norm{{\bw'}^i_{e,p,k} - \bw'_{e,p}   }^2 + 18 r^2 \eta^2 K^2 \frac{\delta^2}{N}} \\
   &\   + 4 \frac{1}{Nn}K\eta L \sqrt{  \pare{48r\eta^2 K^2+ 432 r^3 \eta^4 K^4        L^2 } G  { \log(2RK/\nu)}    } \\
   &\  +  \frac{1}{Nn}  \eta L \sum_{k=0}^{K-1}  \E\norm{ {\bw}_{e,r} -  {\bw}_{e,r,k  }^i } +  \frac{1}{Nn}\eta L  \sum_{k=0}^{K-1} \E\norm{ {\bw'}_{e,r} -  {\bw'}_{e,r,k  }^i } + 2 \frac{K}{Nn}\eta \delta     + 2 \frac{K}{Nn}\eta G D_{\max} .
\end{align*}
Performing telescoping sum from $r=R-1$ to $0$, and using the fact $\bw_0 = \bw'_0$ yields:

\begin{align*}
   \E\norm{\bw_{e+1}  - {\bw'}_{e+1} }  
      &\leq    \E\norm{\bw_{e}  - {\bw'}_{e} }  \\
      & \ +   \frac{1}{Nn}RK \pare{ \eta + \eta L  \sqrt{ 6R^2\eta^2 K^2+  18 R^4  \eta^4 K^4    L^2  } } \pare{\E\norm{      \nabla  F_{\mask}( {\bw}_{e }  ) }  + \E\norm{\nabla  F'_{\mask}( {\bw'}_{e }  )} }\\
       &\  + \frac{1}{Nn} RK\eta L  \sqrt{ 18\eta^2 RKL^2 \sum_{p=0}^{R-1}  \sum_{k=0}^{K-1} \frac{1}{N}\sum_{i=1}^N  \E \norm{\bw^i_{e,p,k} - \bw_{e,p}   }^2 + 18 R^2 \eta^2 K^2 \frac{\delta^2}{N}} \\
   &\  +  \frac{1}{Nn}RK\eta L  \sqrt{ 18\eta^2 RKL^2\sum_{p=0}^{R-1}  \sum_{k=0}^{K-1} \frac{1}{N}\sum_{i=1}^N  \E \norm{{\bw'}^i_{e,p,k} - \bw'_{e,p}   }^2 + 18 R^2 \eta^2 K^2 \frac{\delta^2}{N}} \\
   &\  + 4 \frac{1}{Nn}RK\eta L \sqrt{  \pare{48R\eta^2 K^2+ 432 R^3 \eta^4 K^4        L^2 } G  { \log(2RK/\nu)}    } \\
   &\ +  \frac{1}{Nn}  \eta L\sum_{r=0}^{R-1} \sum_{k=0}^{K-1}  \E\norm{ {\bw}_{e,r} -  {\bw}_{e,r,k  }^i } +  \frac{1}{Nn}\eta L  \sum_{r=0}^{R-1}\sum_{k=0}^{K-1} \E\norm{ {\bw'}_{e,r} -  {\bw'}_{e,r,k  }^i } \\
   &\ + 2 \frac{RK}{Nn}\eta \delta     + 2 \frac{RK}{Nn}\eta G D_{\max} .
\end{align*}

Note the fact that
\begin{align*}
    \sum_{r=0}^{R-1} \sum_{k=0}^{K-1}  \E\norm{ {\bw}_{e,r} -  {\bw}_{e,r,k  }^i } &= RK\cdot \frac{1}{RK}\sum_{r=0}^{R-1} \sum_{k=0}^{K-1}  \E\sqrt{\norm{ {\bw}_{e,r} -  {\bw}_{e,r,k  }^i }^2} \\
    &\leq RK\cdot \sqrt{ \frac{1}{RK}\sum_{r=0}^{R-1} \sum_{k=0}^{K-1}  \E\norm{ {\bw}_{e,r} -  {\bw}_{e,r,k  }^i }^2}, 
\end{align*}
where we apply the Jensen's inequality and concavity of square root function. Hence we have
\begin{align*}
  & \E\norm{\bw_{e+1}  - {\bw'}_{e+1} }  
      \leq    \E\norm{\bw_{e}  - {\bw'}_{e} } \\
      &\ +   \frac{1}{Nn}RK \pare{ \eta + \eta L  \sqrt{ 6R^2\eta^2 K^2+  18 R^4  \eta^4 K^4    L^2  } } \pare{\E\norm{      \nabla  F_{\mask}( {\bw}_{e }  ) }  + \E\norm{\nabla  F'_{\mask}( {\bw'}_{e }  )} }\\
       &\  + \frac{1}{Nn} RK\eta L  \pare{\sqrt{ 18\eta^2 RKL^2 \sum_{p=0}^{R-1}  \sum_{k=0}^{K-1} \frac{1}{N}\sum_{i=1}^N  \E \norm{\bw^i_{e,p,k} - \bw_{e,p}   }^2 }+\sqrt{ 18 R^2 \eta^2 K^2 \frac{\delta^2}{N}}} \\
   &\  +  \frac{1}{Nn}RK\eta L \pare{ \sqrt{ 18\eta^2 RKL^2\sum_{p=0}^{R-1}  \sum_{k=0}^{K-1} \frac{1}{N}\sum_{i=1}^N  \E \norm{{\bw'}^i_{e,p,k} - \bw'_{e,p}   }^2 }+ \sqrt{ 18 R^2 \eta^2 K^2 \frac{\delta^2}{N}}} \\
   &\  + 4 \frac{1}{Nn}RK\eta L \sqrt{  \pare{48R\eta^2 K^2+ 432 R^3 \eta^4 K^4        L^2 } G  { \log(2RK/\nu)}    } \\
   &\  +  \frac{1}{Nn} RK \eta L \sqrt{\frac{1}{RK}\sum_{r=0}^{R-1} \sum_{k=0}^{K-1}  \E\norm{ {\bw}_{e,r} -  {\bw}_{e,r,k  }^i }^2} +  \frac{1}{Nn}RK\eta L  \sqrt{\frac{1}{RK}\sum_{r=0}^{R-1}\sum_{k=0}^{K-1} \E\norm{ {\bw'}_{e,r} -  {\bw'}_{e,r,k  }^i }^2 } \\
   &\  + 2 \frac{RK}{Nn}\eta \delta     + 2 \frac{RK}{Nn}\eta G D_{\max} \\
   &=  \E\norm{\bw_{e}  - {\bw'}_{e} }  +   \frac{1}{Nn}RK \pare{ \eta + \eta L  \sqrt{ 6R^2\eta^2 K^2+  18 R^4  \eta^4 K^4    L^2  } } \pare{\E\norm{      \nabla  F_{\mask}( {\bw}_{e }  ) }  + \E\norm{\nabla  F'_{\mask}( {\bw'}_{e }  )} }\\
   &\  + \frac{RK \eta L }{Nn}\pare{ 1+   \sqrt{18\eta^2 RKL^2}  }\pare{\sqrt{\frac{1}{RK}\sum_{r=0}^{R-1} \sum_{k=0}^{K-1}  \E\norm{ {\bw}_{e,r} -  {\bw}_{e,r,k  }^i }^2} + \sqrt{\frac{1}{RK}\sum_{r=0}^{R-1}\sum_{k=0}^{K-1} \E\norm{ {\bw'}_{e,r} -  {\bw'}_{e,r,k  }^i }^2 } } \\ 
   &\  +   \frac{2 RK\eta L}{Nn}    \sqrt{ 18 R^2 \eta^2 K^2 \frac{\delta^2}{N}}+ 2 \frac{RK}{Nn}\eta (\delta     +   G D_{\max})  +  \frac{4RK\eta L}{Nn} \sqrt{  \pare{48R\eta^2 K^2+ 432 R^3 \eta^4 K^4        L^2 } G  { \log(2RK/\nu)}    } .
\end{align*}

Now we plug in Lemma~\ref{lem:model deviation}:

\begin{align*}
 &  \E\norm{\bw_{e+1}  - {\bw'}_{e+1} }  
      \leq     \E\norm{\bw_{e}  - {\bw'}_{e} }\\
      & +   \frac{ \eta RK}{Nn} \pare{ 1 +   L  \sqrt{  R^2\eta^2 K^2+   R^4  \eta^4 K^4    L^2  }  +    L  \pare{ 1+   \sqrt{ \eta^2 RKL^2}  }\sqrt{    R^3\eta^4 K^5 L^2+  R^5  \eta^6 K^7    L^4  
 +  \eta^2 R K^3 } }\\
 & \ \times \pare{\sqrt{\E\norm{      \nabla  F_{\mask}( {\bw}_{e }  ) }^2 } + \sqrt{\E\norm{\nabla  F'_{\mask}( {\bw'}_{e }  )}^2}}\\
   & + \frac{RK \eta L }{Nn} \\
   & \ \times \pare{ 1+   \sqrt{ \eta^2 RKL^2}  }\pare{  \sqrt{ \eta^2 R K^3   \zeta  +     \pare{ R^2\eta^4 K^5 L^2+  R^4 \eta^6 K^7        L^4 } G  { \log(2RK/\nu)}+    R^3 \eta^4 K^5 L^2 \frac{\delta^2}{N}+  \eta^2 R K^3 \delta^2 } } \\ 
   &+   \frac{  RK\eta L}{Nn}    \sqrt{   R^2 \eta^2 K^2 \frac{\delta^2}{N}}+  \frac{RK}{Nn}\eta (\delta     +   G D_{\max})  +  \frac{ RK\eta L}{Nn} \sqrt{  \pare{ R\eta^2 K^2+   R^3 \eta^4 K^4        L^2 } G  { \log(2RK/\nu)}    } .
\end{align*}

Performing telescoping sum yields:

\begin{align*}
&    \E\norm{\bw_{T}  - {\bw'}_{T} } \\ 
  &    \leq          \frac{ \eta TRK}{Nn} \pare{ 1 +   L  \sqrt{  R^2\eta^2 K^2+   R^4  \eta^4 K^4    L^2  }  +    L  \pare{ 1+   \sqrt{ \eta^2 RKL^2}  }\sqrt{    R^3\eta^4 K^5 L^2+  R^5  \eta^6 K^7    L^4  
 +  \eta^2 R K^3 } }\\
 & \times \frac{1}{T}\sum_{t=1}^T\pare{\sqrt{\E\norm{      \nabla  F_{\mask}( {\bw}_{e }  ) }^2 } + \sqrt{\E\norm{\nabla  F'_{\mask}( {\bw'}_{e }  )}^2}}\\
   & + \frac{TRK \eta L }{Nn}\pare{ 1+   \sqrt{ \eta^2 RKL^2}  }\pare{  \sqrt{ \eta^2 R K^3   \zeta  +     \pare{ R^2\eta^4 K^5 L^2+  R^4 \eta^6 K^7        L^4 } G  { \log(2RK/\nu)}+    R^3 \eta^4 K^5 L^2 \frac{\delta^2}{N}+  \eta^2 R K^3 \delta^2 } } \\ 
   &+   \frac{ T RK\eta L}{Nn}    \sqrt{   R^2 \eta^2 K^2 \frac{\delta^2}{N}}+  \frac{TRK}{Nn}\eta (\delta     +   G D_{\max})  +  \frac{ TRK\eta L}{Nn} \sqrt{  \pare{ R\eta^2 K^2+   R^3 \eta^4 K^4        L^2 } G  { \log(2RK/\nu)}    }    \\
\end{align*}

Now we will bound the gradient norm $\E\norm{ \nabla F_{\mask}(   \bw_{r } )  }^2 $.
Due to Eq.(\ref{eq:point convergence}) we have
 \begin{align*}
  \frac{1}{T}\sum_{e=1}^T  \E \norm{ \nabla  F_{\mask}(\bw_e)}^2   \
        &\leq O\pare{\frac{\E [ F_{\mask}(\bw_0)]}{\eta RK T} }    + O\pare{      {L}^2 \eta^{2 } K^{2 } R     {G   {  \log(2RK/\nu)} } +  \eta^2 R  K^3   \zeta^2 +    \eta  R  K  \frac{\delta^2}{N} } 
    \end{align*}
 
Choosing $\eta = \frac{\sqrt{Nn}}{TRK}$ and $T \geq \sqrt{\frac{n}{N}} $ will conclude the proof:
\begin{align*}
  \E\norm{\bw_{T }  - {\bw'}_{T } }    
    & \leq   O\pare{  \frac{   \delta+GD_{\max} }{\sqrt{Nn}} +  \frac{1}{\sqrt{Nn}}
      \pare{ \sqrt{ \frac{\E [ F_{\mask}(\bw_0)]}{\sqrt{Nn}} }    +     \frac{{L}^2         {G   {  \log(2RK/\nu)} }}{T^2 R }   + \frac{  K    \zeta^2 }{T^2 R  K }    +       \frac{\sqrt{Nn}\delta^2}{TN}     } }\\  
      & \leq   O\pare{  \frac{   \delta+GD_{\max} }{\sqrt{Nn}} +  \frac{1}{\sqrt{Nn}}
      \pare{ \sqrt{ \frac{\E [ F_{\mask}(\bw_0)]}{\sqrt{Nn}} }    +     \frac{{L}^2         {G   {  \log(2RK/\nu)} }}{T^2 R }   + \frac{   K    \zeta^2 }{T^2 R }    +       \frac{\sqrt{Nn}\delta^2}{TN}     }  }.   
\end{align*}

\end{document}